\documentclass[twoside]{article}

\usepackage[accepted]{aistats2025}

\usepackage[round]{natbib}

\usepackage{amsthm,amsmath}
\usepackage{latexsym,amsfonts}
\usepackage{color}
\usepackage{colortbl}
\usepackage{graphicx}
  \graphicspath{{./Figures/}}
  \DeclareGraphicsExtensions{.jpeg,.png,.jpg,.eps,.pdf}
\usepackage{tikz}
\usepackage{epstopdf}
\usepackage{hyperref}
\usepackage{xcolor}

\usepackage{enumitem}
\setlist{
  listparindent=\parindent,
  parsep=0pt,
}
\setlist[itemize]{leftmargin=*}

\usepackage{natbib}
\usepackage[utf8]{inputenc}
\usepackage{comment}
\usepackage{bbm}
\usepackage{cancel}
\usepackage{multirow}
\def\b{{\bf b}}

\def\k{{\bf k}}

\def\u{{\bf u}}
\def\v{{\bf v}}
\def\w{{\bf w}}
\def\x{{\bf x}}
\def\y{{\bf y}}
\def\z{{\bf z}}

\def\A{{\cal A}}
\def\B{{\cal B}}
\def\C{{\cal C}}

\def\E{{\cal E}}
\def\F{{\cal F}}

\def\H{{\cal H}}

\def\N{{\cal N}}
\def\P{{\cal P}}

\def\U{{\cal U}}
\def\V{{\cal V}}
\def\X{{\cal X}}
\def\Y{{\cal Y}}

\def\R{{\mathbb R}}

\def\l{\lambda}

\def\arg{{\rm arg}}

\def\nm{\Vert}

\renewcommand{\iff}{\mbox{$\; \; \Longleftrightarrow \; \;$}}
\renewcommand{\and}{\mbox{$\wedge$}}

\newcommand{\bc}{\begin{center}}
\newcommand{\ec}{\end{center}}
\newcommand{\be}{\begin{equation}}
\newcommand{\ee}{\end{equation}}
\newcommand{\bd}{\begin{displaymath}}
\newcommand{\ed}{\end{displaymath}}
\newcommand{\ba}{\begin{array}}
\newcommand{\ea}{\end{array}}
\newcommand{\ben}{\begin{enumerate}}
\newcommand{\een}{\end{enumerate}}
\newcommand{\bit}{\begin{itemize}}
\newcommand{\eit}{\end{itemize}}
\newcommand{\beq}{\begin{eqnarray}}
\newcommand{\eeq}{\end{eqnarray}}
\newcommand{\btab}{\begin{tabular}}
\newcommand{\etab}{\end{tabular}}
\newcommand{\bfig}{\begin{figure}}
\newcommand{\efig}{\end{figure}}
\newcommand{\btp}{\begin{tikzpicture}}
\newcommand{\etp}{\end{tikzpicture}}

\newcommand{\argmin}{\operatornamewithlimits{arg~min}}

\newcommand{\nmm}[1]{ \nm #1 \nm }
\newcommand{\nmL}[1]{ \nm #1 \nm_{\textrm{Lip}} }

\newcommand{\nmeusq}[1]{ \nm #1 \nm_2^2 }

\newcommand{\nmF}[1]{ \nm #1 \nm_F }

\newcommand{\IP}[2]{ \langle #1 , #2 \rangle }

\def\nmsl1{\nm_{{\rm SL1}}}

\def\NC{{\sf NC}}
\def\CP{{\sf C}}
\def\vecc{{\sf vec}}
\usepackage{accents}
\usepackage[normalem]{ulem}

\usepackage{minitoc}

\definecolor{verm}{rgb}{0.6,0.2,0.2}
\definecolor{purp}{rgb}{0.3,0.1,0.6}
\definecolor{purple}{rgb}{0.4,0.0,0.6}
\definecolor{bggreen}{rgb}{0.1,0.3,0.1}
\definecolor{dgreen}{rgb}{0.1,0.6,0.1}
\definecolor{black}{rgb}{0.0,0.0,0.0}
\definecolor{crim}{rgb}{0.3,0.1,0.1}
\definecolor{dred}{rgb}{0.5,0.1,0.1}
\definecolor{NavyBlue}{HTML}{1f4e79}
\definecolor{DeepTeal}{HTML}{006d77}
\definecolor{DarkSlateGray}{HTML}{2f4f4f}
\definecolor{SlateBlue}{HTML}{6a5acd}
\definecolor{MutedBurgundy}{HTML}{800020}

\newtheorem{corollary}{Corollary}{\bf}{\it}
\newtheorem{definition}{Definition}{\bf}{\it}
{\bf}{\rm}
\newtheorem{lemma}{Lemma}{\bf}{\it}	
\newtheorem{theorem}{Theorem}{\bf}{\it}
{\bf}{\it}
\newtheorem{proposition}{Proposition}{\bf}{\it}
{\bf}{\it}
{\bf}{\rm}
\newtheorem{assumption}{Assumption}{\bf}{\it}

\usepackage{enumitem,amssymb}
\newlist{todolist}{itemize}{2}
\setlist[todolist]{label=$\square$}
\usepackage{pifont}

\usepackage{url}

\usepackage{xr}
\begin{document}

\runningtitle{A Convex Relaxation Approach to Generalization Analysis}

\runningauthor{Tadipatri, Haeffele, Agterberg, Vidal}

\twocolumn[

\aistatstitle{A Convex Relaxation Approach to Generalization Analysis for Parallel Positively Homogeneous Networks}

\aistatsauthor{ Uday Kiran Reddy Tadipatri \And Benjamin D. Haeffele}

\aistatsaddress{ University of Pennsylvania \And 
University of Pennsylvania}

\aistatsauthor{ Joshua Agterberg \And  Ren\'e Vidal}

\aistatsaddress{ University of Illinois Urbana-Champaign
\And University of Pennsylvania} 
]

\begin{abstract}
We propose a general framework for deriving generalization bounds for parallel positively 
homogeneous neural networks--a class of neural networks whose input-output map decomposes as 
the sum of positively homogeneous maps. Examples of such networks include matrix 
factorization and sensing, single-layer multi-head attention mechanisms, tensor 
factorization, deep linear and ReLU networks, and more. Our general framework is based on 
linking the non-convex empirical risk minimization (ERM) problem to a closely related convex 
optimization problem over prediction functions, which provides a global, achievable lower-bound
to the ERM problem. We exploit this convex lower-bound to perform generalization 
analysis in the convex space while controlling the discrepancy between the convex model and 
its non-convex counterpart. We apply our general framework to a wide variety of models 
ranging from low-rank matrix sensing, to structured matrix sensing, two-layer linear 
networks, two-layer ReLU networks, and single-layer multi-head attention mechanisms, 
achieving generalization bounds with a sample complexity that scales almost linearly with the 
network width. 
\end{abstract}

% \newpage
\section{INTRODUCTION}
Despite significant recent advances in the analysis of deep neural networks (DNNs), key gaps persist in establishing guaranteed performance of such models--particularly regarding theoretical guarantees on unseen data. This lack of performance guarantees is especially concerning for high-stakes applications such as autonomous vehicles, healthcare, or other high-consequence decision-making systems. To ensure the safe and reliable deployment of deep learning models, it is essential that generalization guarantees be established under reasonable data-generating mechanisms.

\textbf{Related work.} 
There is a broad literature on generalization theory. Classical approaches can be categorized along two separate (but related) lines:(i) data-dependent versus data-independent bounds, and (ii) uniform versus non-uniform concentration. Informally, \emph{data-dependent} bounds take into account explicit data-generating assumptions, whereas data-independent bounds hold \emph{regardless} of the underlying data distribution. Similarly, \emph{uniform} concentration guarantees focus on obtaining concentration inequalities \emph{simultaneously} for all functions in some function class (known as the \emph{hypothesis space}). In contrast, non-uniform concentration inequalities focus on particular functions estimated from the data. Classical approaches marry these two separate types of analyses by introducing measures such as the VC-dimension \citep{vapnik-chervonenkis-tpa71} or the Rademacher Complexity \citep{bartlett-medelson-jmlr01}. However,~these classical measures are often difficult to compute and overly pessimistic, especially when applied to DNNs \citep{zhang-et-al-acm21}. Consequently, many classical approaches may fail in the modern, more complex DNN setting.

Modern generalization frameworks for DNNs acknowledge that data often comes from structured distributions (e.g., with an intrinsic dimensionality significantly below that of the ambient space) and that optimization algorithms like 
Stochastic Gradient Descent (SGD) explores only a small portion of the hypothesis space
\citep{neyshabur-et-al-nips17}. As a result, the effective hypothesis space is much smaller
than what classical bounds account for based on the expressivity of the model alone.
Consequently, modern bounds focus on data-dependent, non-uniform approaches. 
For instance, \textit{margin bounds} 
\citep{neyshabur-et-al-iclr18, golowich-et-al-colt18, barron-klusowski-arxiv19}
provide specific generalization error bounds for DNNs trained to minimize max-margin type loss functions
for classification tasks.
Another line of research \citep{dziugaite-roy-arxiv17, arora-et-alicml18, banerjee-et-al-arxiv20}
exploits the sensitivity of the non-convex landscapes around learned weights;
however, this approach requires the estimation of hard quantities like expected sharpness
and KL divergence and questions remain regarding the extent to which quantities such as sharpness explain network generalization \citep{wen2023sharpness,andriushchenko2023modern}.

Recent work has observed that optimization methods such as SGD, even without explicit regularization,
tend to yield solutions that generalize well, a notion known as  \textit{implicit bias} \citep{gunasekar2017implicit, gunasekar2018characterizing, gunasekar2018implicit, soudry-et-al-jmlr18, li-et-al-iclr21, haochen2021shape, vardi-acm23}. This stands in contrast to classical theory, which suggests that explicit regularization is necessary to avoid overfitting. For example, DNNs have 
been shown to converge toward maximum-margin solutions in classification tasks \citep{soudry-et-al-jmlr18}, 
while solutions in regression tasks often exhibit low-rank structures \citep{li-et-al-iclr21} that generalize well. 
Although these analyses provide valuable insights, they are generally limited to specific 
objectives and types of neural network architectures.

A key challenge in understanding the generalization properties of DNNs is their non-convex landscape. Indeed, convex landscapes are better understood, and numerous generalization bounds have already been derived \citep{shalev-et-al-colt09, lugosi-neu-colt22}. We argue that bridging the gap between non-convex and convex landscapes could provide a pathway to understanding generalization better. Our key contribution is to propose a new generalization analysis framework for DNNs based on linking their non-convex landscape to a convex one. Our framework builds upon \cite{haeffele2017global} and \cite{vidal-et-al-madl22}, who connected certain non-convex optimization problems to closely related convex ones. However, their work focuses on characterizing the optimization properties of such problems and does not consider generalization.
% \newpage

\textbf{Paper contributions.}
In this work, we use the idea of analyzing non-convex problems via a closely related convex problem to derive generalization bounds for a broad family of learning models, which take the form of sums of (\textit{slightly generalized}) positively homogeneous functions whose parameters are regularized by sums of positively homogeneous functions of the same degree. This allows for a reinterpretation of the (empirical and expected) non-convex optimization problems
as closely related to carefully constructed convex problems. We then apply concentration of measure techniques to the
convexified version under reasonable data distributions and show that this also implies the concentration of the non-convex problem of interest. More specifically, we extend the finite-dimensional framework of
\cite{haeffele2017global} and \cite{vidal-et-al-madl22} to its infinite-dimensional counterpart, which allows us to derive generalization guarantees from a novel 
viewpoint by exploiting the connection between our problem of interest and a closely related convex problem. We note that other prior work \citep{bach-jmlr17} has also considered similar relationships between convex and non-convex problems for establishing generalization results. However, the generalization guarantees in \cite{bach-jmlr17} largely rely on Rademacher complexities, which results in a sample complexity that grows quadratically with the network width. In contrast, we exploit the relationship between the convex and non-convex problems more directly, which allows us to derive bounds with an improved sample complexity.

To be more precise, our main results can be stated informally as follows. Let $N$ be the number of data points, $R$ be the number of positively homogeneous functions (or the width of the network) whose predictions are summed together to form the output, and ${\sf dim}(\mathcal{W})$ be the dimension of the parameters in one of the functions. When $N \gtrsim \tilde{\mathcal{O}}(R \times {\sf dim}(\mathcal{W}))$, we show that the generalization error can be bounded with high probability by two terms: the first term, dubbed the optimization error, which vanishes at a globally optimal solution, and the second term, dubbed the statistical error,  which depends on the ratio $\frac{R \times {\sf dim}(\mathcal{W})}{N}$, and hence vanishes only asymptotically.

Our results apply to a wide range of signal processing and DNN problems. The derived bounds 
achieve near state-of-the-art sample complexity for non-convex low-rank matrix sensing that match
the lower bound provided by \cite{candes-yaniv-arxiv10} for convex low-rank matrix sensing.
By applying  these general results to 
two-layer linear (and ReLU) neural networks with weight decay and multi-head attention models, a key component of 
transformer architecture \citep{vaswani-et-al-arxiv23}, we obtain novel 
generalization bounds with ``tight''\footnote{
Our notion of ``tight'' bounds 
corresponds to cases where the sample complexity scales linearly or nearly linearly (up to 
logarithmic factors) with the number of model parameters.}
sample complexities for both problems.
\newpage
\textbf{Outline.}
The remainder of this paper is organized as follows. In \textsection\ref{sec:formulation}, we formulate the learning problem and introduce our approach. In \textsection\ref{sec:convex_lower_bound}, we explore how learning problems can be bounded via convex surrogates. In \textsection\ref{sec:main_results}, we present the statistical bounds
through the master theorem that provides generalization error bounds. In \textsection\ref{sec:applications}, we apply the master theorem to
various problems in signal processing and DNNs and compare our derived sample complexities with those in the existing literature. The supplementary material contains detailed proofs of the mathematical statements, validations of our framework’s assumptions through simulations, and an additional survey of related works.

\textbf{Notation.} 
For two random variables $(Z,W)$, drawn from a joint distribution $q$, we define $\langle Z, W \rangle_q = \mathbb{E}[\IP{Z}{W}]$, where the expectation is with respect to the joint probability distribution $q$. For a generic function $f: \mathbb{R}^{d} \to \mathbb{R}$, we denote $\|f\|_{Lip}$ as its Lipschitz constant; i.e., the smallest number $L_f$ such that $|f(x) - f(y)| \leq L_f \| x- y\|$. A function $f$ is said to be integrable with respect to measure $q$, i.e., $f \in L^2(q)$, if $\left(\int_{x \in \X}\nmm{f(x)}^2dq(x)\right)^{1/2} < \infty$.
The inequality $f(x) \gtrsim g(x)$,
means that there exists a constant $c > 0$ such that $f(x) \geq cg(x)$. We define the ReLU function as $[x]_+ = \max(x,0).$
For the 
matrix $U$ the variable $\u_j$
corresponds to $j$th column
of $U$.

% !TeX root = uday_aistats25.tex

\section{PROBLEM FORMULATION}
\label{sec:formulation}
Given a realization of a pair $(X,Y) \in \X \times \Y$ from a distribution $\mu$ with $\X \subset \mathbb{R}^{n_X}$, $\Y \subset \mathbb{R}^{n_Y}$, we consider a (non)parametric regression problem of the form $Y = g(X,\epsilon)$, where $\epsilon$ is a source of additional noise (typically independent from $X$).  
We are interested in approximating $g$ by the sum of $r$ prediction functions, $\phi : \mathcal{W} \times \mathbb{R}^{n_X} \to \mathbb{R}^{n_Y}$, parameterized by $W \in \mathcal{W}$, i.e.,
%of the form
\be
\hat{Y} = \sum_{j=1}^{r}\phi(W_j)(X) = \Phi_r(\{W_j\})(X).
\ee
We will additionally refer to $\phi(W)(X)$ as
the \textit{factor map/sub-network} depending on the specific problem.

Our goal is to learn the parameters $\{W_j\}$\footnote{We occasionally notate $\{W_j\}_{i=1}^r$ as $\{W_j\}$ %$\NC$ 
for brevity of notation, but the dependence on $r$ is always implied.} that minimize the regularized population risk defined as
\be
\begin{split}
\NC_{\mu}(\{W_j\}) &:= \underbrace{
\mathbb{E}_{(X, Y)}\left[\ell(Y, \Phi_r(\{W_j\}_{j=1}^{r})(X))\right]}_{ =: \ell(g, \Phi_r(\{W_j\}_{j=1}^{r}))_{\mu}}\\
&\hspace{20pt} +  \l \Theta_r(\{W_j\}_{j=1}^{r}),
\end{split}
\ee
where $Y = g(X,\epsilon)$ is the target random variable, $\ell(\cdot, \cdot)$ is the \textit{loss function},
typically convex in the second argument, and $\Theta_r(\{W_j\}_{j=1}^{r})$ is an \textit{explicit regularization} function
which helps find structured parameters, 
such as minimum norm or sparse solutions. 
Specifically, the regularization term $\Theta_r(\{W_j\}_{j=1}^{r})$ is defined as
\be
\Theta_r(\{W_j\}_{j=1}^{r}) := \sum_{j=1}^{r}\theta(W_j),
\ee
where $\theta : \mathcal{W} \to \mathbb{R}^+$
is a regularization term for each factor map, and $\l \in \R^{+}$ is a regularization hyperparameter that controls
the trade-off between loss reduction and
inducing structure. 

Notice that we will minimize the population risk $\NC_{\mu}(\{W_j\})$ over both $r$ and $\{W_j\}_{j=1}^r$.  More explicitly, we will allow for problems where, in addition to optimizing over the model parameters, one also optimizes over the number of prediction functions $r$ (e.g., the network width) during training. However, our results will also apply to a value of $r$ that is fixed \textit{a priori}.

Estimating $\NC_{\mu}(\{{W}_j\})$ directly is challenging due to 
(i) the lack of access to the distribution $\mu$,
(ii) the fact that $( \{W_j\})$ (and potentially the number $r$) are random variables dependent on the 
training data $\{(X_i, Y_i)\}$,
and (iii) the non-linearity and potential non-convexity of $\NC_{\mu}$.  We address the first point (as is standard) via empirical minimization of $\NC_{\mu}(\cdot)$ using the \emph{empirical risk} (or \emph{training error}) defined via:
\be
\begin{split}
\NC_{\mu_N}(\{W_j\}) &:= \underbrace{\frac{1}{N}\sum_{i=1}^{N}\ell(Y_i, \Phi_r(\{W_j\}_{j=1}^{r})(X_i))}_{ =: \ell(g, \Phi_r(\{W_j\}_{j=1}^{r}))_{\mu_N}} \\
& \hspace{50pt}+ \l \Theta_r(\{W_j\}_{j=1}^{r}),
\end{split}
\ee
where $\mu_N$ denotes the empirical distribution of the samples $\{X_i,Y_i\}_{i=1}^{N}$.  
We define empirical risk minimization (ERM) via the $\arg\min$ of $\NC_{\mu_N}(\{W_j\})$.
%where our analysis will apply to any %the $\arg\min$ includes all first-order stationary point.  
For concreteness, recall we also allow for the minimization over $r$ (provided $r$ is bounded above by some quantity independent of the data), though our results hold for any fixed~$r$. 

Note that if we minimize the objective $\NC_{\mu_N}(\cdot)$, there is no guarantee that we will also minimize $\NC_{\mu}(\cdot)$.
This discrepancy is quantified by the \textit{Generalization Error:}
\be\label{eq:ge}
\begin{split}
&\left|\NC_{\mu}(\{W_j\})
- \NC_{\mu_N}(\{W_j\})\right| \\ 
 &= \left| \ell(g,\Phi_r(\{W_j\}_{j=1}^{r}))_{\mu} - \ell(g,\Phi_r(\{W_j\}_{j=1}^{r}))_{\mu_N} \right|.
\end{split}
\ee
Note that the regularization terms containing $\Theta_r$ are the same between the two objectives, giving the typical difference between the empirical and population losses.

In this work, we compute an upper bound for the generalization error at any stationary point of the empirical problem, $\NC_{\mu_N}(\{W_j\})$, under certain technical assumptions. To build our main results, we relate these non-convex objectives $\NC_{\mu}(\{W_j\})$ and $\NC_{\mu_N}(\{W_j\})$ to closely related \textit{convex} objectives in the prediction space, respectively $\CP_{\mu}(f_{\mu})$ and $\CP_{\mu_N}(f_{\mu_N})$, whose definitions will be introduced in \textsection\ref{sec:convex_lower_bound}. This allows us to decompose the generalization error in \eqref{eq:ge} as:
\be\label{eq:decom}
\begin{split}
&\NC_{\mu}(\{W_j\}) - \NC_{\mu_N}(\{W_j\})\hspace{-2pt}=\hspace{-2pt}
\underbrace{\Big[\NC_{\mu}(\{W_j\}) - \CP_{\mu}(f_{\mu})\Big]}_{\text{Population Gap}} \\
&\hspace{-5pt}- \underbrace{\Big[\NC_{\mu_N}(\{W_j\}) - \CP_{\mu_N}(f_{\mu_N})\Big]}_{\text{Empirical Gap}} 
\hspace{-2pt}+\hspace{-2pt}\underbrace{\Big[\CP_{\mu}(f_{\mu}) - \CP_{\mu_N}(f_{\mu_N})\Big]}_{\text{Convex Generalization Gap}}
\hspace{-4pt}.
\end{split}
\ee
Our Theorem
\ref{thm:thm1} bounds the \textit{Empirical Gap} and the \textit{Population Gap}.
With these bounds, we then apply concentration techniques to bound the \textit{Convex Generalization Gap}
and obtain our main Theorem \ref{thm:master} which gives bounds for the generalization error in \eqref{eq:ge}.

\section{CONVEX BOUNDS FOR LEARNING}\label{sec:convex_lower_bound}
In this section, we present bounds for the \textit{Empirical Gap} and \textit{Population Gap} 
through Theorem \ref{thm:thm1}, linking our learning problem of interest to functions that are 
convex in the space of prediction functions.
To begin, we state 
several requirements for our framework.

\begin{assumption}[Regularization]\label{ass:a1}
The regularization function $\theta$ is \textit{positive semidefinite}; i.e,
$\theta(0) = 0$ and $\theta(W) \geq 0, \forall W \in \mathcal{W}$.
\end{assumption}
This is a mild assumption; it only ensures we do not impose negative regularization on the parameters $\{W_j\}$.
Our next assumption is our main functional assumption on $\phi$ and $\theta$.
\begin{assumption}[Balanced Homogeneity of $\phi$ and $\theta$]\label{ass:a2}
The factor map $\phi$ and the regularization map $\theta$ can be scaled equally by non-negative scaling of (a subset of) the parameters. 
Formally, we assume that there exists sub-parameter spaces
$(\mathcal{K}, \H)$ from the parameter space $\mathcal{W}$
such that $\mathcal{K} \times \H = \mathcal{W}$, $\forall ({\mathbf{k}}, {\mathbf{h}}) \in (\mathcal{K}, \H)$, and $\beta \geq 0$ we have 
$\phi((\beta {\mathbf{k}}, {\mathbf{h}})) = \beta^{p}\phi(({\mathbf{k}}, {\mathbf{h}}))$ and 
$\theta((\beta {\mathbf{k}}, {\mathbf{h}})) = \beta^{p}\theta(({\mathbf{k}}, {\mathbf{h}}))$ for some $p > 0$.
Further, we assume that for bounded input $X$ the set $\{ \phi(W)(X) : \forall W \in \mathcal{W}\text{ s.t. }\theta(W) \leq 1 \}$ is bounded.
\end{assumption}

This is a slight generalization of positive homogeneity, which only requires positive homogeneity in a subset of parameters, provided the image of the factor map for parameters with $\theta(W) \leq 1$ is bounded
\footnote{For example, $\phi(v)(X)$ can take the form 
$v^{a}g(X)$ for $a \geq 0$, and $\theta(v) = |v|^{a}$, where $g: \R^{n_x} \to \R$ is some fixed function. More 
generally, we can choose $\phi(v_1, v_2) = v_1^{a}g_{v_2}(X)$ and $\theta(v_1, v_2) = |v_1|^{a}+\delta_{\V_2}
(v_2)$, where $g_{v_2}: \R^{n_x} \to \R$ is a function parameterized by $v_2 \in \V_2$ and has a bounded range 
for bounded inputs.}.

Our next assumption concerns the loss function $\ell$.

\begin{assumption}[Convex Loss]\label{ass:a4}
The loss $\ell(Y, \hat{Y})$ is second-order differentiable (written $\ell \in \C^2)$,
and $L$-smooth w.r.t. $\hat{Y}$, i.e, for any $Y, \hat{Y} \in \R^{n_Y}$
\be
0 \preceq\nabla^2_{\hat{Y}}\ell(Y, \hat{Y}) \preceq L I_{n_Y}.
\ee
Additionally,
the gradient of the loss is bi-Lipschitz smooth; that is, for all 
$Y_1, Y_2, \hat{Y}_1, \hat{Y}_2 \in \R^{n_Y}$
\be
\begin{split}
\hspace{-10pt}\nmm{\nabla_{\hat{Y}}\ell(Y_2, \hat{Y}_2) - \nabla_{\hat{Y}}\ell(Y_1, \hat{Y}_1)}
\leq L \Big[&\nmm{Y_2 - Y_1}_{2} \\ 
&\hspace{-10pt}+\nmm{\hat{Y}_2 - \hat{Y}_1}_{2}\Big],
\end{split}
\ee
and the loss is constant if both the arguments are the same, i.e., for all $Y_1, Y_2 \in \R^{n_Y}$,
$\ell(Y_1, Y_1) = \ell(Y_2, Y_2)$.
\end{assumption}

This ensures that the loss 
function is convex and smooth. Furthermore, if the
loss is $\alpha$-strongly convex, i.e, 
$0 \prec \alpha I_{n_Y} \preceq \nabla^2_{\hat{Y}}\ell(Y, \hat{Y}) \preceq L I_{n_Y}$ we have derived tighter results (see the Appendix).

We define the \textit{induced regularization} function as
\be
\begin{split}
\Omega(f) &:= \inf_{r, \{W_j\}}\Theta_r(\{W_j\})\\
\text{s.t. }f(X) &= \Phi_r(\{W_j\})(X); \quad \forall X \in \X,
\end{split}
\ee
with the function taking value infinity if $f(X)$ cannot be realized for some choice of the parameters $(r, \{W_j\}_{j=1}^r)$.  Using similar arguments as in \cite{haeffele-vidal-arxiv15} it can be shown that under assumptions \ref{ass:a1}--\ref{ass:a2}, the function $\Omega(f)$ is convex in the space of prediction functions; see Proposition
\ref{prop:ir_cvx} in
the Appendix. Moreover, by Assumption \ref{ass:a4}, the loss function is convex with respect to the model predictions, which allows us to define the following two \textit{convex} optimization problems over the space of prediction functions: 
\be
\CP_{\mu}(f) :=  {\mathbb{E}_{(X, Y)}[\ell(Y, f(X))]} + \l \Omega(f),
\ee
where $f \in L^2(\mu)$, and
\be
\CP_{\mu_N}(f) :=  {\frac{1}{N}\sum_{i=1}^{N}\ell(Y_i, f(X_i))} + \l \Omega(f),
\ee
where $f \in L^2(\mu_N)$.

From the definition of $\Omega(f)$ we have that $\CP_{\mu}$ and $\CP_{\mu_N}$ are always lower bounds of $\NC_{\mu}$ and $\NC_{\mu_N}$, respectively, for any $(f, \{W_j\})$ such that $f(X) = \Phi_r(\{W_j\})(X)$, which becomes a tight bound for any parametrization ($\{W_j\}$) of $f$ which achieves the infimum.  As a result, we can relate solutions of the non-convex problems to the corresponding convex problem via tools from convex analysis, as we establish in the following result.

\begin{theorem}[Convex Bounds for Learning]\label{thm:thm1}
Under assumptions \ref{ass:a1}--\ref{ass:a4},
let $f^*_{\mu_N}$ (or $f^*_{\mu}$) be the 
global minimizer for $\CP_{\mu_N}(\cdot)$ (or $\CP_{\mu}(\cdot))$.
For any
stationary points $(r, \{W_j\})$ 
of the function $\NC_{\mu_{N}}(\cdot)$
and any $f \in L^2(\mu) \cap L^2(\mu_N)$ the
following are true:
\begin{enumerate}
\item[1.] Empirical optimality gap:
\end{enumerate}
\be\label{eq:ncvx_emp_gap}
\begin{split}
\CP_{\mu_{N}}(f^*_{\mu_{N}}) &\leq \NC_{\mu_{N}}(\{W_j\}) \leq 
\CP_{\mu_{N}}(f) \\ 
& \hspace{-40pt} + \l \Omega(f)\left[\Omega_{\mu_{N}}^{\circ}\left(-\frac{1}{\l}\nabla_{\hat{Y}}\ell\left(g, \Phi_r(\{W_j\})\right)\right)-1\right],\\
\end{split}
\ee
\begin{enumerate}
\item[2.] Population optimality gap:
\end{enumerate}
\be\label{eq:ncvx_pop_gap}
\begin{split}
\CP_{\mu}(f^*_{\mu}) &\leq \NC_{\mu}(\{W_j\}) \leq 
\CP_{\mu}(f) \\
&\hspace{-20pt}+ \l \Omega(f)\left[\Omega_{\mu}^{\circ}\left(-\frac{1}{\l}\nabla_{\hat{Y}}\ell\left(g, \Phi_r(\{W_j\})\right)\right)-1\right]\\
&+ \Big[\IP{\nabla_{\hat{Y}}\ell\left(g, \Phi_r(\{W_j\})\right)}{\Phi_r(\{W_j\})}_{\mu} \\
&- \IP{\nabla_{\hat{Y}}\ell\left(g, \Phi_r(\{W_j\})\right)}{\Phi_r(\{W_j\})}_{\mu_N}\Big],
\end{split}
\ee
where $\Omega_{q}^{\circ}(\cdot)$ is
referred to as polar in the measure $q$ defined as
\be\label{eq:polar}
\Omega_{q}^{\circ}(g) := \sup_{\theta(W) \leq 1}\IP{g}{\phi(W)}_{q}.
\ee
\end{theorem}

Readers are referred to Appendix \ref{sec:apdx_thm1} for the proof with extensions to strongly convex functions.

The population optimality gap is obtained by an infinite-dimensional extension of 
Proposition 3 in \cite{haeffele-vidal-tpami20}.
The additional term in the population optimality gap \eqref{eq:ncvx_pop_gap} arises from 
the fact that the stationary points of ERM, $\NC_{\mu_N}(\cdot)$, are not necessarily the same 
as those of $\NC_{\mu}(\cdot)$.

From equation \eqref{eq:ge} the goal 
is to bound the difference between the original 
non-convex formulations $\NC_{\mu_N}$ and $\NC_{\mu}$.
By Theorem \ref{thm:thm1}, we established 
the optimality gaps for both empirical and population non-convex optimization 
problems, and by computing the difference between equation \eqref{eq:ncvx_emp_gap} and \eqref{eq:ncvx_pop_gap}, with algebraic manipulation we arrive at the following quantities:
\vspace{-5pt}
\begin{itemize}
    \item \textit{Convex Generalization Gap}: The convex generalization gap is defined as
    $\left|C_{\mu}(f) - C_{\mu_N}(f)\right|$.
    
    \item \textit{Polar Gap}: By virtue of the fact that the loss functions each contain the respective polars, we define the Polar Gap as the quantity $| \Omega_{\mu}^{\circ}(\nabla_{\hat Y}\ell(g,f)) - \Omega_{\mu_N}^{\circ}(\nabla_{\hat Y}\ell(g,f))|$.

    \item \textit{Equilibria Gap}: We define the Equilibria Gap via
    $\bigg| \langle \nabla_{\hat Y} \ell(g,f), f \rangle_{\mu_N} - \langle \nabla_{\hat Y} \ell(g,f), f \rangle_{\mu} \bigg|$.

    \item \textit{Norm Gap}: The final remaining quantity is defined via
    $\bigg| \|f_{\mu}^* - f \|_{\mu_N}^2 - \| f_{\mu}^* - f \|_{\mu}^2 \bigg|$. This quantity applies only to 
    strongly convex functions (see the Appendix).
\end{itemize}
\vspace{-5pt}
A major technical contribution of this paper is to demonstrate that each of these quantities uniformly
concentrates at a rate equal to or smaller than the ``statistical error'' under certain realistic assumptions that are discussed
in \textsection\ref{sec:main_results}.
The only remaining term from Theorem \ref{thm:thm1} is the quantity $\Omega(f_{\mu})[\Omega_{\mu_N}^{\circ}(\cdot) - 1]$, which bounds the sub-optimality (in objective value) of the current stationary point for the empirical optimization problem.
This term approaches zero at the global optimum of $\NC_{\mu_N}$ (see \textsection\ref{sec:apdx_thm1} in the Appendix).

\section{STATISTICAL BOUNDS}\label{sec:main_results}

In Theorem \ref{thm:thm1}, we established bounds for the \textit{Empirical Gap} and \textit{Population Gap}. Building on these
results, we identified key quantities such as the \textit{Convex Generalization Gap}, \textit{Polar Gap}, \textit{Equilibrium
Gap}, and \textit{Norm Gap}, all of which can be controlled under certain general conditions 
(Assumptions \ref{ass:a1}--\ref{ass:a7}, along with Assumption \ref{ass:a9} from the Appendix) that we state momentarily. In this section, we present Theorem \ref{thm:master}, which consolidates these bounds to derive our main generalization error bound. 
For clarity and to minimize technical complexity, we present Theorem \ref{thm:master} with Assumption \ref{ass:a8}, which
a stronger version of Assumption \ref{ass:a9}.

To begin, we state our additional assumptions.
We assume that $\phi$ is Lipschitz.  
\begin{assumption}[Lipschitz Continuity of $\phi$]\label{ass:a3}
Let $\B$ be some compact subset of $\mathcal{W}$, and denote
\be
\F_{\theta} := \left\{W: \theta(W) \leq 1\right\} \cap \B \subseteq \mathbb{B}(r_{\theta}),
\ee
where $\mathbb{B}(r_{\theta})$ is the $L_2$ ball with radius $r_{\theta}$.\footnote{The radius $r_{\theta}$
can depend on the dimension of $W$. For instance, suppose $W \in \R^{n}$ and  $\theta(W) = \nmm{W}_1$, as $\nmm{W}_1 \leq \sqrt{n}\nmm{W}_2$, then $r_{\theta}$ must be at least $\sqrt{n}$. On another instance, suppose $W = (\u \in \R^{m}, \v \in \R^{n})$, and $\theta(W) = \nmm{\u}_2\nmm{\v}_2$; this requires $r_{\theta}$ to be at least $1/2$.} The factor map $\phi$ is Lipschitz continuous with respective to inputs for any choice of
parameters $W \in \F_{\theta}$, i.e,
\begin{align}
L_{\phi} := \sup_{W \in \F_{\theta}}\nmL{\phi(W)} < \infty.
\end{align}
\end{assumption}

Our next assumption imposes tail conditions on the random variables $(X,Y)$.
\begin{assumption}[Data Model]\label{ass:a5}
The input data $X \in \R^{n_X}$ is drawn from the $1$-Lipschitz concentrated sub-Gaussian distribution with a proxy variance $\sigma_{X}^2/n_X$; i.e., for any $1$-Lipschitz continuous function, $h: \R^{n_X} \to \R$ there exists $c> 0$ such that 
\be
P\left(|h(X) - \mathbb{E}_X[h(X)]| \geq \epsilon\right) \leq c\exp\left(-\frac{n_X\epsilon^2}{2\sigma_X^2}\right).
\ee

The target function $Y$ takes the form
$Y = g(X, \epsilon)$,
where $g \in L^2(\mu)$ is bi-Lipschitz in $X$  and $\epsilon$; that is,
\be
\begin{split}
\nmm{g(X_2, \epsilon_2)-g(X_1, \epsilon_1)}_{2} \leq 
\nmL{g}\Big[&\nmm{X_2-X_1}_{2} \\
+ &\nmm{\epsilon_2 - \epsilon_1}_{2}\Big],
\end{split}
\ee
and $\epsilon \sim \N(0, (\sigma_{Y|X}^2/n_E)I)$ in $\R^{n_E}$.
\end{assumption}
We note that the above assumption is mild. While extending our framework to 
heavy-tailed distributions are likely possible; it would require a more intricate analysis
and may result in worse error rates and larger sample complexities.

Our next assumption concerns the possible functions learned via empirical risk minimization.
\begin{assumption}[Hypothesis class]\label{ass:a7}
Stationary points of $NC_{\mu_N}(\cdot)$ have bounded regularization and bounded width, $r \leq R$, almost surely.
The input-output map, $\Phi_r(\{W_j\})$ has Lipschitz constant at most $\gamma$,
and the parameters are bounded. Let $\B_R \subseteq \mathcal{W}^R$ be some compact set; then the hypothesis
class is defined as
\be
\hspace{-5pt}\F_{\mathcal{W}} := \Big\{\{W_j\}_{j=1}^{r}: \nmL{\Phi_r(\{W_j\})} \leq \gamma\Big\} \cap \B_R.
\ee
\end{assumption}
\vspace{-5pt}
In words, the set of maps learned through ERM are essentially Lipschitz in the parameters $\{W_j\}$, and, furthermore, the $\{W_j\}$ are bounded (almost surely).  Moreover, the assumption that $r \leq R$ ensures that at most $R$ individual functions $\{W_j\}$ are needed, which implicitly imposes a ``low-complexity'' constraint on the learned function.  Finally, note that we assume that $\gamma$ does not depend on the width of the network. 
In practice, our empirical observations show that the
Lipschitz constant does not increase with width{\color{blue},}
making it a realistic assumption. For
further details, refer to the numerical simulations in \textsection\ref{sec:apdx_ms} of Appendix.

Our general master theorem, Theorem \ref{thm:gen_master} in the Appendix, requires only
Assumptions \ref{ass:a1}--\ref{ass:a7} and \ref{ass:a9} (in the Appendix).
For the sake of notational brevity{\color{blue},} we state our main results with the slightly stronger Assumption 
\ref{ass:a8} instead of Assumption \ref{ass:a9}.

\begin{assumption}[Boundedness] \label{ass:a8}
For all $(X, Y) \in \X \times \Y$, and $\{W_j\} \in \F_{\mathcal{W}}$, the predictions, and gradients are bounded; i.e,
\be
\nmm{\Phi_r(\{W_j\})(X)}\hspace{-2pt}\leq\hspace{-2pt}B_{\Phi}\text{, }
\nmm{\nabla_{\hat{Y}}\ell(Y, \Phi_r(\{W_j\})(X))}\hspace{-2pt}\leq\hspace{-2pt}B_{\ell}.
\ee
Further, for any $(X, Y)$ $\in$ $\X \times \Y$, for any $\{W_j\}, \{\tilde{W}_j'\} \in \F_{\mathcal{W}}$, $W, \tilde{W} \in \F_{\theta}$,
the network, $\phi$ and $\Phi_r$, are Lipschitz in the parameters; i.e,
\be
\begin{split}
\|\Phi_r(\{W_j\})(X) &- \Phi_r(\{\tilde{W}_j\})(X)\|_2\\
&\leq \tilde{L}_{\Phi}\max_{j}\nmm{W_j - \tilde{W}_j}_2, \text{ and }
\end{split}
\ee
\vspace{-10pt}
\be
\nmm{\phi(W)(X) - \phi(\tilde{W})(X)}_2
\leq \tilde{L}_{\phi}\nmm{W - \tilde{W}}_2.
\ee

\end{assumption}

Assumption \ref{ass:a8} ensures that predictions and its gradients are bounded while the network being Lipschitz continuous on the parameter space for any inputs.
Assumption \ref{ass:a8} implicitly indicates that
either the data points are uniformly bounded or
the search space for the parameters is of small dimension, which can restrict the potential applications.
However, as we demonstrate in the more general version (Theorem \ref{thm:gen_master}) in the Appendix, it suffices that the conditions above hold only for some convex set $\C$, though this extension requires significantly more notation and discussion, so we do not include it here.

\begin{theorem}[Master Theorem]\label{thm:master}
Suppose Assumptions \ref{ass:a1}--\ref{ass:a8} hold.  %Let $\delta\in$
%Under the assumptions \ref{ass:a1}-\ref{ass:a8}. 
Let $\delta \in (0, 1]$ be fixed,
and let $f_{\mu}^*$ be the global optimum of $\CP_{\mu}$. Suppose that $\gamma \geq \Omega(f_{\mu}^*)L_{\phi}$, and define
\be
\epsilon_1 = 16\gamma^2\sigma_X^2\max\left\{1,
\frac{L}{4}\hspace{-2pt}\left[1 + \frac{\nmL{g}^2}{\gamma^2}\left(1 + \frac{\sigma_{Y|X}^2}{\sigma_X^2}\right)\hspace{-2pt}\right]\hspace{-2pt}\right\}
\hspace{-2pt};
\ee
\be
\begin{split}
\epsilon_2 = 4\tilde{L}_{\Phi}B_{\Phi}\max\Big\{&1, 2L + 2B_{\ell}/B_{\Phi}, \\
& \hspace{-40pt} 8\Omega(f_{\mu}^*)(B_{\ell}\tilde{L}_{\phi})/(\tilde{L}_{\Phi}B_{\Phi}), 8L\Omega(f_{\mu}^*)\Big\}.
\end{split}
\ee
Let $\{W_j\}$ denote any stationary point of
$\NC_{\mu_N}(\cdot)$.
Then with probability at least $1 - \delta$, it holds that 
\begin{align}
&\frac{1}{n_Y}\left|\NC_{\mu}(\{W_j\}) - \NC_{\mu_N}(\{W_j\})\right| \lesssim\\
\nonumber &\hspace{20pt}  \underbrace{\frac{\l}{n_Y} \Omega(f_{\mu}^*)\left[\Omega_{\mu_N}^{\circ}\left(-\frac{1}{\l}\nabla_{\hat{Y}}\ell\left(g, \Phi_r(\{W_j\})\right)_{\mu_N}\right)-1\right] }_\text{Optimization Error} \\ 
\nonumber &+\hspace{-2pt}\left.\underbrace{\epsilon_1\sqrt{\frac{R\cdot{\sf dim}(\mathcal{W})
\log\left(\frac{\gamma \epsilon_2r_{\theta}}{L_{\phi}}\right)\log(N)\hspace{-2pt}+\hspace{-2pt}\log\left(\frac{1}{\delta}\right)}{N}}}_\text{Statistical Error}\right\}.
\end{align}
\end{theorem}
\vspace{-6pt}
\textbf{Remarks}: The generalization error is upper
bounded by two terms:  
\vspace{-6pt}
\begin{itemize}
    \item the \textit{Optimization Error}, which quantifies 
the distance to the globally optimal solution, and
\item the \emph{Statistical Error}, or the intrinsic error that depends on the sample complexity and the noise. 
\end{itemize}
\vspace{-5pt}
The optimization error diminishes as we approach a global optimum of the ERM problem $\NC_{\mu_N}$ and vanishes at a global optimum, whereas the statistical error diminishes as the sample size increases relative to the intrinsic dimension, i.e., when $N \gtrsim R \times {{\sf dim}(\mathcal{W})}$ (ignoring logarithmic factors).  By a naive counting argument, there are $R \times {{\sf dim}(\mathcal{W})}$ many parameters in the underlying network.  As we will see in subsequent sections, this sample complexity turns out to be optimal or nearly optimal for a number of reasonable statistical settings.
The implicit constants appearing in the result are universal and are not problem dependent.
% \vspace{-12pt}

% \vspace{-10pt}
\section{APPLICATIONS}\label{sec:applications}

In this section, we present applications of the Theorem \ref{thm:master}
for low-rank matrix sensing, two-layer ReLU neural networks, and single-layer multi-head attention. To apply Theorem \ref{thm:master}, we must compute the problem-specific quantities $\Omega(f_{\mu}^*)$, $\Omega^{\circ}_{\mu_N}(\cdot)$, $L$, $\nmL{g}$, $\sigma_X$, $\sigma_{Y|X}$, 
$\epsilon_1$, $\epsilon_2$, $r_{\theta}$, $\gamma$, $L_{\phi}$. For each application, we 
have estimated these quantities, with further details provided in the proofs located in 
Appendix \ref{sec:apdx_low_mat}, \ref{sec:apdx_2rln}, and \ref{sec:apdx_mha}, respectively.
We summarize and compare 
the obtained sample complexities for the various applications with their state-of-the-art bounds in Table \ref{tab:app}.  The additional applications to structured matrix sensing 
and two-layer linear neural networks can be found in  Appendix \ref{sec:apdx_sms} and Appendix \ref{sec:apdx_2lnn}, respectively.

\begin{table*}[ht!]
\caption{Comparisons with the state-of-the-art sample complexities. $N$ represents the number of data points.}
\label{tab:app}
\begin{center}
\renewcommand{\arraystretch}{1.05} % Increase vertical spacing
\resizebox{\textwidth}{!}{
\begin{tabular}{|c|c|c|}
\hline
\rowcolor{gray!20}
\large{Application} & \large{Our work, $N \gtrsim$} & \large{State-of-the-art, \large $N \gtrsim$} \\
\hline
\hline
Low rank matrix sensing & \multirow{5}{*}{$\colorbox{green!30}{${\tilde{\mathcal{O}}(R(m + n))}$}$} & ${\color{red}R^*}(m+n)$, \citep{stoger-zhu-arxiv24} {\color{red}(\textrm{no regularization})} \\
\cline{1-1} \cline{3-3}
Structured matrix sensing & & -- \\
\cline{1-1} \cline{3-3}
2-Layer linear NN & & $R(m+n)$ \citep{kakade-et-al-nips08} {\color{red}(bounded data-points)} \\
\cline{1-1} \cline{3-3}
2-Layer ReLU NN & & $R(m+n)\log(R(m+n))$, \citep{bartlett-et-al-jmlr19} \\
\cline{1-1} \cline{3-3}
Multi-head attention & & $R(m + n)$, \citep{trauger-tewari-aistats23} {\color{red}(bounded data-points)} \\
\hline
\end{tabular}
}
\end{center}
\end{table*}

\textbf{Low-rank matrix sensing:} We first consider 
low-rank matrix sensing \citep{candes-yaniv-arxiv10}, which is a well-studied problem in the signal processing and statistics literature.
Given a few linear measurements of an unknown low-rank matrix, the goal is to estimate the low-rank matrix in the presence of noise. 
One potential strategy is to define a convex program via nuclear-norm regularization \citep{candes-recht-fcm09}. While recovery guarantees for this convex program are well-studied, solving it is a computationally intensive procedure involving computing a full singular value decomposition at each iteration.
To address this issue, several authors have considered a non-convex variant that reparameterizes the low-rank matrix into its underlying left and right factors, which is known as the Burer-Monteiro factorization \citep{burer-monterio-mp03}.  While the new optimization
problem runs faster in practice, it is also non-convex, and its properties can be difficult to analyze theoretically.  
Corollary \eqref{crl:matsen}
provides the bounds on the 
generalization error for this non-convex program.

\begin{corollary}[Low-Rank Matrix Sensing]\label{crl:matsen}
Consider the true model for $(X,y)$, where $X\in\R^{m\times n}$ is a random matrix with i.i.d. entries $X_{lk} \sim \N(0, \frac{1}{mn})$ and $y = \IP{M^*}{X} + \epsilon$, where $M^* \in \R^{m \times n}$ and $\epsilon \sim \N(0, \sigma^2)$ is independent from $X$. For all $i \in [N]$, let $(X_i,y_i)$ be i.i.d. samples from this true model. Consider the estimator $\hat{y} = \IP{UV^T}{X}$, where $U \in \R^{m \times R}$ and $V \in \R^{n \times R}$. 
%Suppose for all $i=1,\dots,N$, $X_i\in\R^{m\times n}$ is a random matrix with entries $(X_i)_{lk} \sim \N(0, \frac{1}{mn})$, $Y_i = \IP{M^*}{X_i} + \epsilon_i$, where $M^* \in \R^{m \times n}$, and $\epsilon_i \sim \N(0, \sigma^2)$. Consider the model $\hat{Y} = \IP{UV^T}{X}$, where $U \in \R^{m \times R}$, $V \in \R^{n \times R}$. 
Let $\delta \in (0,1]$ be fixed.  Define the non-convex problem 
\be
\begin{split}
\NC^{\sf MS}_{\mu_N}(( U,V)) &:= %\min_{U, V} 
\frac{1}{2N}\sum_{i=1}^{N} \big( y_i - \IP{UV^T}{X_i} \big)^2 \\
& \hspace{30pt} + \l \sum_{j=1}^{R}{\nmm{\u_j}_2\nmm{\v_j}_2},
\end{split}
\ee
and define $\NC^{MS}_{\mu}((U,V))$ similarly with the sum over $i$ replaced by expectation taken over $(X,y)$. 

Let $(\hat{U},\hat{V})$ be a stationary point of $\NC^{\sf MS}_{\mu_N}(\cdot)$.  
Suppose there exists $C_{UV}, B_u, B_v > 0$ such that $\nmm{\hat{U}\hat{V}^T}_2 \leq C_{UV} \nmm{M^*}_*$, and for 
all $j \in [R]$, $\nmm{\hat{\u}_j}_2 \leq B_u$, $\nmm{\hat{\v}_j}_2 \leq B_v$.
% Suppose that $(U,V)$ is any stationary point satisfying $\nmm{UV^T}_2 \leq \sqrt{\sigma^2 + 1}\nmm{M^*}_*$, and that for any $j \in [R]$ we have $\nmm{\u_j}_2 \leq B_u$ and $\nmm{\v_j}_2 \leq B_v$.
Then with probability at least $1 - \delta$, it holds that
%\bd
\begin{align}
&\bigg|\NC_{\mu}^{\sf MS}((\hat{U}, \hat{V})) - \NC_{\mu_N}^{\sf MS}((\hat{U}, \hat{V}))\bigg|
\lesssim \\
\nonumber  & \nmm{M^*}_*\left[\nmm{\frac{1}{N}\sum_{i=1}^{N}(y_i - \IP{\hat{U}\hat{V}^T}{X_i})X_i}_2-\l\right] \\
\nonumber  &+ C_{UV}^2\nmm{M^*}_*^2 \times \\
\nonumber &\sqrt{\frac{R
\log\left(R(C_{UV}\hspace{-2pt}+\hspace{-2pt}B_{u}B_v)\right)(m+n)\log(N)\hspace{-2pt}+\hspace{-2pt}\log(1/\delta)}{N}}.
\end{align}
\end{corollary}

\textbf{Remarks}: Observe that at a global minimum, the right side tends to zero when $R(m + n)/N \to 0$, ignoring logarithmic terms.
Existing literature on non-convex noisy low-rank matrix sensing typically requires knowledge of true $\text{rank}(M^*)=R^*$, and the state-of-the-art sample complexity for this setting is of order $R^*(m+n)$ in the un-regularized setting \citep{stoger-zhu-arxiv24}. In contrast, Corollary \ref{crl:matsen} does not require knowledge of the true rank.  However, if the estimated rank $R$ is too small ($R < R^*$), then the optimization error still persists. In contrast, if $(R \geq R^*)$ then optimization error can vanish subject to the ability of the algorithm utilized to reach stationary points, 
\cite{haeffele-vidal-arxiv15} provides such guarantees.

\textbf{Two-layer ReLU Networks:} Next, we move on to two-layer ReLU networks, which introduce an additional nonlinearity with
respect to the inputs. ReLU networks are widely used and proven to be universal approximators \citep{huang-nc20}. Prior work on generalization analysis for ReLU networks is based on
classical measures, such as Rademacher complexity \citep{bartlett-et-al-jmlr19}. The following result circumvents the difficulty in the estimate of such classical measures.
% \begin{mdframed}[style=corollarystyle]
\begin{corollary}[Two-Layer ReLU Neural Network] \label{crl:2rlnn} Consider the true model for $(\x,\y)$, where $\x \sim \N(0, (1/n)I_{n}) \in \R^{n}$, $\y = U^*[{V^*}^T\x]_+ + \epsilon$, where $U^* \in \R^{m \times {R^*}}$, $V^* \in \R^{n \times {R^*}}$, and $\epsilon \sim \N(0, (\sigma^2/m)I_{m}) \in \R^{m}$ independent from $\x$. For all $i\in [N]$, let $(\x_i,\y_i)$ be i.i.d. samples from this true model. Consider the estimator $\hat{\y} = U[V^T\x]_+$, where $U \in \R^{m \times R}, V \in \R^{n \times R}$. Let $\delta \in (0,1]$ be fixed.  Define the non-convex problem
\be
\begin{split}
\NC_{\mu_N}^{\sf {ReLU}}( (U,V)) &:= \frac{1}{2N}\sum_{i=1}^{N}\nmeusq{\y_i - U[V^T\x_i]_+}\\
&+ \frac{\l}{2}\left(\nmF{U}^2 + \nmF{V}^2\right),
\end{split}
\ee
and define $\NC^{\sf {ReLU}}_{\mu}((U,V))$ similarly with the sum over $i$ replaced by expectation taken over $(\x, \y)$. 

Let $(\hat{U},\hat{V})$ be a stationary point of $\NC^{\sf {ReLU}}_{\mu_N}(\cdot)$. 
Suppose there exists $C_{UV}, B_u, B_v > 0$ such that $\nmm{\hat{U}\hat{V}^T}_2 \leq C_{UV}\left[\nmF{U^*}^2 + \nmF{V^*}^2\right]$, and for 
all $j \in [R]$, $\nmm{\hat{\u}_j}_2 \leq B_u$, $\nmm{\hat{\v}_j}_2 \leq B_v$.
% Suppose that $(U,V)$ is any stationary point satisfying $\nmm{UV^T}_2 \leq \frac{\sqrt{\sigma^2 + 1}}{2}\left[\nmF{U^*}^2 + \nmF{V^*}^2\right]$, and that for any 
% $j \in [R]$ we have $\nmm{\u_j}_2 \leq B_u$ and $\nmm{\v_j}_2 \leq B_v$. 
Then with probability at least $1 - \delta$, it holds that
\begin{align}
&\frac{1}{m}\left|\NC_{\mu}^{\sf {ReLU}}((\hat{U}, \hat{V})) - \NC_{\mu_N}^{\sf {ReLU}}((\hat{U}, \hat{V}))\right| \lesssim\\
\nonumber &\frac{1}{2m}\left[\nmF{U^*}^2+\nmF{V^*}^2\right]\left[\frac{1}{N}\sum_{i=1}^{N}\nmm{\y_i - \hat{\y}_i}_2\nmm{\x_i}_2\hspace{-1pt}-\hspace{-1pt}\l\right] \\
\nonumber &+ C_{UV}^2\left[\nmF{U^*}^2 + \nmF{V^*}^2\right] \times \\
\nonumber & \left[\hspace{-2pt}\frac{R(m+n)
\log\left(R(m\hspace{-2pt}+\hspace{-2pt}n)(C_{UV}\hspace{-2pt}+\hspace{-2pt}B_u^2\hspace{-2pt}+\hspace{-2pt}B_v^2)\right)\log(N)}{N}\right. \\
\nonumber & \hspace{50pt}+\left. \frac{\log(1/\delta)}{N}\right]^{1/2}.
\end{align}
\end{corollary}

\textbf{Remarks}: Analogous to matrix sensing, when $ R (m + n)/N \to 0$, the right side tends to zero at global optimality (ignoring logarithmic terms).  Furthermore, Corollary \ref{crl:2rlnn} recovers the state-of-the-art
result by \cite{bartlett-et-al-jmlr19}.
% \vspace{-4pt}

\textbf{Transformers:} Finally, we move on to our last application (though of course, the applications are in fact myriad in principle) to a single layer
multi-head attention, which are backbones for transformer-style architecture \citep{vaswani-et-al-arxiv23}.
In practice, transformers are shown to have remarkable generalization capabilities \citep{zhou-arxiv2024}.
However, there is a lack of intensive theoretical analysis for this architecture. Few attempts
on estimating the capacities of the attention mechanisms have been made in \cite{edelman-arxiv2022} and 
\cite{trauger-tewari-aistats23}, among others.  For our analysis, we consider the case where the output of the model is one particular token within the input (e.g., transformers use a dedicated class token for the output initialized as a constant vector).  The output for one attention head is modeled as $V X \sigma((K X)^\top Q \mathbf{x}_{out})$ where $\mathbf{x}_{out}$ is the column of $X$ corresponding to the transformer output.  We then reparameterize $K^T Q \mathbf{x}_{out} = \mathbf{z}$ and present the following result.

\begin{corollary}[Transformers]\label{crl:tf}
Consider the true model for $(X,\y)$, where $X\in\R^{n \times T}$ is a random matrix with i.i.d. entries $X_{lk} \sim \N(0, {1}/{(nT)})$ and $\y = A^*X\b^* + \epsilon$, where $A^* \in \R^{m \times n}$, $\b^* \in \mathbb{S}^{T-1}$ and $\epsilon \sim \N(0, (\sigma^2/m)I_{m})$ is independent from $X$. For all $i \in [N]$, let $(X_i, \y_i)$ be i.i.d. samples from this true model. Consider the estimator
$\hat{\y} = \sum_{j=1}^{R}V_jX\sigma(X^T\z_j)$, 
$V_j \in \R^{n}, \z_j \in \R^{n}$.  Let $\delta \in (0,1]$ be fixed.  Define the non-convex problem
% \vspace{-0.4cm}
%
\begin{align}
\nonumber \NC_{\mu_N}^{\sf TF}\hspace{-2pt}(\{(V_j,\z_j)\})\hspace{-3pt}&:=\hspace{-2pt}\frac{1}{2N}\hspace{-2pt}\sum_{i=1}^{N}\nmeusq{\y_i\hspace{-2pt}-\hspace{-2pt}\sum_{j=1}^{R}V_jX_i\sigma_t(X_i^T\z_j)}\\
&\hspace{-40pt}+ \l \sum_{j=1}^{R}\left[\nmF{V_j} + \delta_{\{\z: \nmm{\z}_2 \leq 1\}}(\z_j)\right],
\end{align}
%
% \vspace{-0.5cm}
where, $\sigma_t(\cdot)$ is softmax function with temperature $t$, for $k \in [T]$ defined $\sigma_t(\u)_k := \exp(tu_k)/\sum_{l=1}^{T}\exp(tu_l)$
and define $\NC_{\mu}^{\sf TF}( \{(V_j,\z_j)\})$ similarly with the sum over $i$ replaced by expectation taken over $(X, \y)$. 

Let $\{(\hat{V}_j,\hat{\z}_j)\}$ be a stationary point of $\NC^{\sf TF}_{\mu_N}(\cdot)$. 
Suppose there exists $C_{V}, B_V > 0$ such that $\sum_{j=1}^{R}\nmF{\hat{V}_j} \leq C_{V}\nmF{A^*}$, and for 
all $j \in [R]$, $\nmF{\hat{V}_j} \leq B_V$.
% Let $\{V_j,\z_j\}$ denote any stationary point such that 
% $\sum_{j=1}^{R}\nmF{V_j} \leq \sqrt{\sigma^2 + 1}\nmF{A^*},$ with $\nmF{V_j} \leq B_V$. 
Then with probability at least $1 - \delta$ it holds that
\begin{align}\label{eq:gen_tf}
&\frac{1}{m}\left|\NC_{\mu}^{\sf TF}( \{(\hat{V}_j,\hat{\z}_j)\}) - \NC_{\mu_N}^{\sf TF}( \{(\hat{V}_j,\hat{\z}_j)\})\right| \lesssim\\
\nonumber &\frac{1}{2m}\nmF{A^*}\left[\frac{1}{N}\sum_{i=1}^{N}\nmm{\y_i - \hat{\y}_i}_2\nmm{X_i}_2-\l\right] \\
\nonumber & + C_{V}^2\nmF{A^*}^2\times \\
\nonumber & \sqrt{\hspace{-2pt}\frac{R(m\hspace{-2pt}+\hspace{-2pt}n)
\hspace{-2pt}\log\hspace{-2pt}\left(R(m\hspace{-2pt}+\hspace{-2pt}n)(C_{V}\hspace{-2pt}+\hspace{-2pt}B_V)\right)\hspace{-2pt}\log(N)\hspace{-2pt}+\hspace{-2pt}\log(1/\delta)}{N}}.
\end{align}
\end{corollary}

% \vspace{-5pt}
\textbf{Remarks:} 
The dependence on $\b^*$ is not explicitly reflected in
Equation \eqref{eq:gen_tf} because the ground truth model is bilinear. Consequently,
assuming $\b^*$ is unit-norm without loss of generality, as its norm can be
absorbed into $A^*$. Thus, the dependence on $\b^*$ is implicitly captured by the norm of
$A^*$ in Equation \eqref{eq:gen_tf}.
As in the previous two applications, we can achieve consistency at global optimality when $N \gtrsim R (m + n)$,
ignoring logarithmic terms. Note
that the sample complexity has no dependency on the number of tokens, $T$, which suggests an explanation for the 
success behind the prediction capabilities of transformers for longer length inputs 
\citep{zhou-arxiv2024}. Our sample complexity matches the state-of-the-art bounds
on the transformers by \cite{trauger-tewari-aistats23}.

% \vspace{-8pt}
\section{CONCLUSIONS}
% \vspace{-8pt}
In this work, we provide generalization bounds for non-convex problems
of the form of sums of (\textit{slightly generalized}) positively homogeneous functions with a general objective. Our bounds
provide sample complexities that are near-optimal and applicable to various problems, such as low-rank matrix
sensing, two-layer neural networks, and single-layer multi-head attention. 
The sample complexity of our bounds
grows almost linear with the total number of parameters in the model, and for matrix sensing, this sample complexity is optimal, as demonstrated in \cite{candes-yaniv-arxiv10}. Our proofs are based on analyzing closely related convex programs in the prediction space; this perspective enabled us to provide near-optimal sample complexities due to existing results on generalization properties for
convex functions.
In future work, it would be interesting to sharpen the dependence of our bounds on all the relevant parameters and apply our techniques to other machine learning problems.

\subsubsection*{Acknowledgments}

UKRT gratefully acknowledges Pratik Chaudhari, Hancheng Min,
Kyle Poe and Ziqing Xu for their valuable discussions and 
constructive feedback.
His research was supported by the Leggett Family 
Fellowship and the Dean’s Fellowship programs. Other authors acknowledge the support of the Research Collaboration on the Mathematical and Scientific Foundations of Deep Learning (NSF grant 2031985 and Simons Foundation grant 814201).

\bibliography{refs}

% \newpage

\section*{Checklist}

 \begin{enumerate}

 \item For all models and algorithms presented, check if you include:
 \begin{enumerate}
   \item A clear description of the mathematical setting, assumptions, algorithm, and/or model. [\textbf{Yes}/No/Not Applicable]
   \item An analysis of the properties and complexity (time, space, sample size) of any algorithm. [\textbf{Yes}/No/Not Applicable]
   \item (Optional) Anonymized source code, with specification of all dependencies, including external libraries. [Yes/No/\textbf{Not Applicable}]
 \end{enumerate}

 \item For any theoretical claim, check if you include:
 \begin{enumerate}
   \item Statements of the full set of assumptions of all theoretical results. [\textbf{Yes}/No/Not Applicable]
   \item Complete proofs of all theoretical results. [\textbf{Yes}/No/Not Applicable]
   \item Clear explanations of any assumptions. [\textbf{Yes}/No/Not Applicable]     
 \end{enumerate}

 \item For all figures and tables that present empirical results, check if you include:
 \begin{enumerate}
   \item The code, data, and instructions needed to reproduce the main experimental results (either in the supplemental material or as a URL). [Yes/No/\textbf{Not Applicable}]
   \item All the training details (e.g., data splits, hyperparameters, how they were chosen). [{Yes}/No/\textbf{Not Applicable}]
         \item A clear definition of the specific measure or statistics and error bars (e.g., with respect to the random seed after running experiments multiple times). [{Yes}/No/\textbf{Not Applicable}]
         \item A description of the computing infrastructure used. (e.g., type of GPUs, internal cluster, or cloud provider). [Yes/{No}/\textbf{Not Applicable}]
 \end{enumerate}

 \item If you are using existing assets (e.g., code, data, models) or curating/releasing new assets, check if you include:
 \begin{enumerate}
   \item Citations of the creator If your work uses existing assets. [Yes/No/\textbf{Not Applicable}]
   \item The license information of the assets, if applicable. [Yes/No/\textbf{Not Applicable}]
   \item New assets either in the supplemental material or as a URL, if applicable. [Yes/No/\textbf{Not Applicable}]
   \item Information about consent from data providers/curators. [Yes/No/\textbf{Not Applicable}]
   \item Discussion of sensible content if applicable, e.g., personally identifiable information or offensive content. [Yes/No/\textbf{Not Applicable}]
 \end{enumerate}

 \item If you used crowdsourcing or conducted research with human subjects, check if you include:
 \begin{enumerate}
   \item The full text of instructions given to participants and screenshots. [Yes/No/\textbf{Not Applicable}]
   \item Descriptions of potential participant risks, with links to Institutional Review Board (IRB) approvals if applicable. [Yes/No/\textbf{Not Applicable}]
   \item The estimated hourly wage paid to participants and the total amount spent on participant compensation. [Yes/No/\textbf{Not Applicable}]
 \end{enumerate}

 \end{enumerate}

\newpage

\onecolumn
\appendix

\aistatstitle{\vspace{-0.3cm}A Convex Relaxation Approach to Generalization Analysis for Parallel Positively Homogeneous Networks: \\
Supplementary Materials\vspace{-0.3cm}}
\vspace{-2cm}

In this supplementary material, we provide a detailed discussion of the rigorous technical aspects omitted from the main text.
Additionally, we present a comprehensive review of related works and include a few numerical experiments. 
Below is the table of contents for this appendix/supplementary material.
{\small\tableofcontents}

\section{CONVEX BOUNDS FOR LEARNING}\label{sec:apdx_thm1}

In this section, we discuss the proof for Theorem \ref{thm:thm1} that establishes the optimality
gaps in the empirical and population landscapes. First, we analyze the convexity of the induced regularizer, $\Omega(\cdot)$
and properties of the stationary points in
non-convex landscape. These are the key components of
our proof for Theorem \ref{thm:thm1}. We 
state a more general version of Assumption \ref{ass:a2}
by having the flexibility of the loss being strongly
convex to derive tighter results.
\renewcommand{\theassumption}{2'}
\begin{assumption}[Convex Loss]\label{ass:apdx_a2}
The loss $\ell(Y, \hat{Y})$ is second-order differentiable (written $\ell \in \C^2)$,
$\alpha$-strong and $L$-smooth w.r.t. $\hat{Y}$, i.e, for any $Y, \hat{Y} \in \R^{n_Y}$
\be
0 \preceq \alpha I_{n_Y} \preceq \nabla^2_{\hat{Y}}\ell(Y, \hat{Y}) \preceq L I_{n_Y}.
\ee
Additionally,
the gradient of the loss is bi-Lipschitz; that is, for all 
$Y_1, Y_2, \hat{Y}_1, \hat{Y}_2 \in \R^{n_Y}$
\be
\nmm{\nabla_{\hat{Y}}\ell(Y_2, \hat{Y}_2) - \nabla_{\hat{Y}}\ell(Y_1, \hat{Y}_1)}
\leq L \Big[\nmm{Y_2 - Y_1}_{2}
+\nmm{\hat{Y}_2 - \hat{Y}_1}_{2}\Big],
\ee
and the loss is constant if both the arguments are the same, i.e., for all $Y_1, Y_2 \in \R^{n_Y}$,
$\ell(Y_1, Y_1) = \ell(Y_2, Y_2)$.
\end{assumption}
Note that we allow $\alpha=0$, in which case we recover Assumption \ref{ass:a2}.  

\subsection{Induced Regularizer in Convex Space}

First, we show that the induced regularizer is convex in the function spaces through Proposition \ref{prop:ir_cvx}.

\begin{proposition}[Convexity of induced regularizer]\label{prop:ir_cvx}
Suppose assumptions \ref{ass:a1}-\ref{ass:apdx_a2} hold. Then $\Omega(f)$ is convex in $f$
in the space of functions
$\R^{n_X} \to \R^{n_Y}$.
\end{proposition}
\begin{proof}
This proof is infinite dimensional
extension of \cite{haeffele-vidal-arxiv15}.
Recall the definition of induced
regularizer:
\be
\Omega(f) := \inf_{r, \{W_j\}}\Theta_r(\{W_j\})\text{ such that }f(X) = \Phi_r(\{W_j\}); \forall X \in \X.
\ee

Define the function class
\be
\F_{\Phi} := \{\Phi_r(\{W_j\}): r \in \mathbb{N}, W_j \in \mathcal{W}\}.
\ee
By definition if $f \notin \F_{\Phi}$ 
then $\Omega(f)$ evaluates to infinity.
Now suppose that
$\beta \geq 0$ and for any $f \in \F_{\Phi}$,

\be
\Omega(\beta f) = \inf_{r, \{W_j\}}\Theta_r(\{W_j\})\text{ such that }\beta f(X) = \Phi_r(\{W_j\}); \forall X \in \X,
\ee
Now by
Assumption \ref{ass:apdx_a2}, there exists $\hat{\beta}$ such that
$\beta \Phi_r(\{W_j\}) = \Phi_r(\{\hat{\beta} W_j\})$, and $\beta \Theta_r(\{W_j\}) = \Theta_r(\{\hat{\beta} W_j\})$ (throughout note that this scaling is applied only to the $\mathcal{W}_p$ subset of parameters from Assumption \ref{ass:a2}, but we do not notate this explicitly for brevity of notation). Now 
we perform a change of variables in the induced regularizer, obtaining
\be
\Omega(\beta f) = \inf_{r,  \{\hat{\beta} W_j\}}\Theta_r( \{\hat{\beta} W_j\})\text{ such that } \beta f(X) = \Phi_r( \{\hat{\beta} W_j\}); \forall X \in \X.
\ee
Then we have that
\be
\Omega(\beta f) = \inf_{r,  \{ W_j\}}\beta\Theta_r( \{W_j\})\text{ such that }  \beta f(X) = \beta \Phi_r( \{ W_j\}); \forall X \in \X
= \beta \Omega(f).
\ee
We have established
that the function $\Omega(\cdot)$
is 1-degree homogeneous. Now we prove
that the function $\Omega(\cdot)$
is sub-additive. Choose any $f_1, f_2 \in \F_{\Phi}$, because the case when either of
them is not in $\F_{\Phi}$ is trivially
sub-additive. Recall
\be
\Omega(f_1)=\inf_{r,  \{W_j\}}\Theta_r( \{W_j\})\text{ such that }  f_1(X) = \Phi_r( \{ W_j\}); \forall X \in \X,
\ee
\be
\Omega(f_2)=\inf_{r,  \{W_j\}}\Theta_r( \{ W_j\})\text{ such that }  f_2(X) = \Phi_r( \{ W_j\}); \forall X \in \X,
\ee
\be
\Omega(f_1 + f_2)=\inf_{r,  \{W_j\}}\Theta_r( \{ W_j\})\text{ such that }  f_1(X) + f_2(X) = \Phi_r( \{ W_j\}); \forall X \in \X.
\ee

For any $\epsilon > 0$ let $(r_1,\{W_j^1\})$ and $(r_2,\{W_j^2\})$ be parameters which come within $\epsilon$ of the infimum in the optimization problems for $\Omega(f_1)$ and $\Omega(f_2)$ respectively. Then note that 
\be
\Omega(f_1 + f_2) \leq \Theta_{r_1}(\{W_j^1\}) + \Theta_{r_2}(\{W_j^2\}) \leq \Omega(f_1) + \Omega(f_2) + 2\epsilon.
\ee
Letting $\epsilon \rightarrow 0$ gives that $\Omega(f_1 + f_2) \leq \Omega(f_1) + \Omega(f_2)$. Thus, as $\Omega(\cdot)$ is both positively homogenous with degree one
and sub-additive, it is convex.
\end{proof}

From the above proposition we have that $\Omega(\cdot)$ is a convex function, therefore we have that $C_{\cdot}(\cdot)$
is indeed a convex function in the prediction functions space. Our results primarily depend upon the optimal
regularization of the globally optimal solution of a convex function, $C_{\cdot}(\cdot)$. 
As we operate in the space of functions, it is very unlikely that we have the knowledge of
the global optima. Nevertheless, by exploiting the convexity of $C_{\cdot}(\cdot)$ we can
upper bound the optimal regularization. Proposition \ref{prop:regular} establishes
the upper bound for the optimal regularization for regression loss.

\begin{proposition}\label{prop:regular}
Consider $\ell(Y_1, Y_2) = \frac{1}{2}\nmm{Y_1-Y_2}_2^2$,
$\{W_j\} \in \F_{\mathcal{W}}$. Suppose
$X \sim \mu$, $\epsilon$ is random variable
such that $\mathbb{E}[\epsilon] = 0$ and independent
from $x$. Let $Y = \Phi_{r}(\{W_j\})(X) + \epsilon$, and suppose $f_{\mu}^*$ is the global optimal
solution of $C_{\mu}(\cdot)$. Then we have
\be
\Omega(\Phi_{r}(\{W_j\})) \geq \Omega(f_{\mu}^*).
\ee
\end{proposition}
\begin{proof}
As $f_{\mu}^*$ is the global optimal solution, we have that
\be
\begin{split}
\mathbb{E}\left[\frac{1}{2}\nmm{\Phi_r(\{W_j\})(X)+\epsilon-f_{\mu}^*(X)}_2^2\right] + \l \Omega(f_{\mu})
&\leq \mathbb{E}\left[\frac{1}{2}\nmm{\Phi_r(\{W_j\})(X)
+\epsilon-\Phi_r(\{W_j\})(X)}_2^2\right]\\
&\quad + \l \Omega(\Phi_{r}(\{W_j\}))
\end{split}
\ee
Now, by re-arranging the terms we obtain
\be
\mathbb{E}\left[\frac{1}{2}\IP{\Phi_r(\{W_j\})(X)-f_{\mu}^*(X)}{\epsilon}\right] 
+ \l \Omega(f_{\mu})
\leq 
\l \Omega(\Phi_{r^*}(\{W_j\})).
\ee
As $\epsilon$ is independent of $X$,
\be
\frac{1}{2}\IP{\mathbb{E}_X\left[\Phi_r(\{W_j\})(X)-f_{\mu}^*(X)\right]}{\mathbb{E}_{\epsilon}\left[\epsilon\right]}
+ \l \Omega(f_{\mu})
\leq 
\l \Omega(\Phi_{r^*}(\{W_j\})).
\ee
Then we have
\be
\Omega(f_{\mu}) \leq \Omega(\Phi_{r}(\{W_j\})).
\ee
\end{proof}

\subsection{Proof of Theorem \ref{thm:thm1}}

Optimization algorithms used to optimize DNNs try to find the set of parameters that 
are first-order optimal. However, we do not have a guarantee that these points are saddle/local minima/global minima.
In proposition \ref{prop:stat_points}, we provide properties that
any first-order optimal satisfies for positively homogeneous networks.

\begin{proposition}[Stationary Points]\label{prop:stat_points}
Under assumption \ref{ass:a2}, if $\{W_j\}$ are stationary points of $\NC_{\mu}(\cdot)$, then for all $j \in [r],$
\be
\IP{-\frac{1}{\l}\nabla_{\hat{Y}}\ell(g, \Phi_r(\{W_j\}))}{\phi(W_j)} = \theta(W_j).
\ee
\end{proposition}

\begin{proof}
This proof is similar to that of Proposition 2 in \cite{haeffele-vidal-tpami20} but applied to a general class of
(\textit{slightly}) positively homogeneous functions (see assumption \ref{ass:a2}).

From assumption \ref{ass:a2}, there exists a subset of parameters where both $\theta$ and $\phi$ are positively homogeneous.  Let $\w_i$ be the subset of parameters in $\mathcal{W}_p$ from assumption \ref{ass:a2}.  Then we have
\be
\begin{split}
\IP{\w_i}{\partial_{\w_i}\theta(\w_1, \dots, \w_i, \dots, \w_n)}
= \lim_{\epsilon \to 0}&\left[\frac{\theta(\w_1, \dots, (1+\epsilon)\w_i, \dots, \w_n)}{\epsilon} 
-\frac{\theta(\w_1, \dots, \w_i, \dots, \w_n)}{\epsilon}\right].
\end{split}
\ee

Let $p_i$ be the homogeneous degree of the parameters $\w_i$.
Note that $\partial_{\w_i}\theta(\w_1, \dots, \w_i, \dots, \w_n) \in \R^{dim(\w_i) \times 1}$,
$\partial_{\w_i}\phi(\w_1, \dots, \w_i, \dots, \w_n) \in \R^{dim(\w_i) \times n_Y}$. 
Then

\begin{align}
\IP{\w_i}{\partial_{\w_i}\theta(\w_1, \dots, \w_i, \dots, \w_n)} &= \theta(\w_1, \dots, \w_i, \dots, \w_n)\lim_{\epsilon \to 0}\frac{(1+\epsilon)^{p_i}-1}{\epsilon}, \\
&= p_i\theta(\w_1, \dots, \w_i, \dots, \w_n). \label{derivative_expression}
\end{align}
Similarly, following a similar argument for $\phi$ we obtain 
\begin{align}
\IP{\partial_{\w_i}\phi(\w_1, \dots, \w_i, \dots, \w_n)}{\w_i} = p_i\phi(\w_1, \dots, \w_i, \dots, \w_n).
\end{align}
As $W_j$ are the stationary points we have that
\begin{align}
0 \in \partial_{W_j}\Phi_r(\{W_j\})\nabla_{\hat{Y}}\ell(g, \Phi_r(\{W_j\}))_{\mu} + \l \partial_{W_j}\Theta_r(\{W_j\}).
\end{align}
Since $\Phi_r(\{W_j\}) = \sum_{j=1}^{r}\phi(W_j)$, we have that $ \partial_{W_j}\Phi_r(\{W_j\}) = \partial_{W_j}\phi(W_j)$. Similarly, $\partial_{W_j}\Theta_r(\{W_j\}) = \partial_{W_j}\theta(W_j)$ holds true.  Consequently, 
\begin{align}
0 \in \partial_{W_j}\phi(W_j)\nabla_{\hat{Y}}\ell(g, \Phi_r(\{W_j\}))_{\mu} + \l \partial_{W_j}\theta(W_j).
\end{align}
Letting $W_j = \begin{bmatrix}
\w_1 & \dots & \w_n
\end{bmatrix}$, for all $\w_i$ it holds that  
\begin{align}
0 \in \partial_{\w_i}\phi(W_j)\nabla_{\hat{Y}}\ell(g, \Phi_r(\{W_j\}))_{\mu} + \l \partial_{\w_i}\theta(W_j).
\end{align}
Taking the inner product of the above equation with $\w_i$ , when $p_i \neq 0$ we have that
\begin{align}
0 \in {\w_i}^T\partial_{\w_i}\phi(W_j)\nabla_{\hat{Y}}\ell(g, \Phi_r(\{W_j\}))_{\mu} + \l {\w_i}^T\partial_{\w_i}\theta(W_j).
\end{align}
From \eqref{derivative_expression} we have that
\begin{align}
0 = p_i\phi(W_j)^T\nabla_{\hat{Y}}\ell(g, \Phi_r(\{W_j\}))_{\mu} + \l p_i \theta(W_j).
\end{align}
Rearranging, we obtain
\begin{align}
\IP{-\frac{1}{\l}\nabla_{\hat{Y}}\ell(g, \Phi_r(\{W_j\}))}{\phi(W_j)} = \theta(W_j), 
\end{align}
which holds for all $j \in [r]$.
\end{proof}

Proposition \ref{prop:stat_points} establishes that at any
stationary point, the inner product between the prediction errors and the
predictions equates to the the current regularization. Next, we exploit
this property of stationary points that enable us to tie the non-convex
landscape to its convex counterpart. Lemma \ref{lemma:opt_gap}
establishes the difference between the non-convex and convex objective
values at stationary points.

\begin{lemma}[Optimality Gap]\label{lemma:opt_gap}
Let $\ell(\cdot,\cdot)$ denote any $L$-smooth, and $\alpha$-strongly convex loss function, let $q$ be some measure, and suppose that $\{W_j\}$ is a stationary point of $\NC_{q}(\{W_j\})$.  Let $f^{*}_{q}$ denote the global minimizer of $\CP_{q}(\cdot)$.  
Then for any $f \in L^2(q)$, we have
that
\begin{align}\label{eq:ncvx_cvx}
\CP_{q}(f^*_{q}) \leq \NC_{q}(\{W_j\}) \leq 
\CP_{q}(f) &+ \l \Omega_{q}(f)\left[\Omega_{q}^{\circ}\left(-\frac{1}{\l}\nabla_{\hat{Y}}\ell\left(g, \Phi_r(\{W_j\})\right)\right)-1\right]
 - \frac{\alpha}{2}\nmm{f-\Phi_r(\{W_j\})}_{q}^2
\end{align}

\end{lemma}

\begin{proof}
The loss $\ell(Y, \hat{Y})$ is $(L, \l)$-convex in $\hat{Y}$. Therefore,
for any functions $g: \X \times E \to \Y$, and $f_1, f_2  \in L^2(q)$ we have 
that
\begin{align}
\ell\left(g(X, \epsilon), f(X)\right) &\geq 
\ell\left(g(X, \epsilon), \Phi_r(\{W_j\})(X)\right) \\
&+ \IP{\nabla_{\hat{Y}}\ell\left(g(X, \epsilon), \Phi_r(\{W_j\})(X)\right)}{f(X)-\Phi_r(\{W_j\})(X)}_{\Y} \\
&+ \frac{\alpha}{2}\nmm{f(X)-\Phi_r(\{W_j\})(X)}_{\Y}^2.
\end{align}
Taking expectations of both sides with respect to the probability measure $q$, we have that
\begin{align}
\ell(g, f)_{q} \geq \ell\left(g, \Phi_r(\{W_j\})\right)_{q} + \IP{\nabla_{\hat{Y}}\ell\left(g, \Phi_r(\{W_j\})\right)}{f-\Phi_r(\{W_j\})}_{q}
+ \frac{\alpha}{2}\nmm{f-\Phi_r(\{W_j\})}_{q}^2. \label{eq:popineq}
\end{align}

As the $\{W_j\}$ are the stationary points of $\NC_{q}(\{W_j\})$ from Proposition \ref{prop:stat_points}  we have that for all $j \in [r]$
\begin{align}
\IP{-\frac{1}{\l}\nabla_{\hat{Y}}\ell\left(g, \Phi_r(\{W_j\})\right)}{\phi(W_j)}_{q} = \theta(W_j).
\end{align}
Summing the above identity up overall $j$, it holds that
\begin{align}
 \IP{-\frac{1}{\l}\nabla_{\hat{Y}}\ell\left(g, \Phi_r(\{W_j\})\right)}{\Phi_r(\{W_j\})}_{q} = \Theta_r(\{W_j\}).
\end{align}
Therefore, plugging this identity into the inequality \eqref{eq:popineq}, we have that
\begin{align}
\ell(g, f)_{q} \geq \underbrace{\ell\left(g, \Phi_r(\{W_j\})\right)_{q} + \l \Theta_r(\{W_j\})}_{\NC_{q}(\{W_j\})} + \IP{\nabla_{\hat{Y}}\ell\left(g, \Phi_r(\{W_j\})\right)}{f}_{q} + \frac{\alpha}{2}\nmm{f-\Phi_r(\{W_j\})}_{q}^2,
\end{align}
which implies that 
\begin{align}
\ell(g, f)_{q} +
\l \IP{-\frac{1}{\l}\nabla_{\hat{Y}}\ell\left(g, \Phi_r(\{W_j\})\right)}{f}_{q} \geq {NC_{q}(\{W_j\})}
+ \frac{\alpha}{2}\nmm{f-\Phi_r(\{W_j\})}_{q}^2.
\end{align}

We have established from Proposition \ref{prop:ir_cvx} that $\Omega$ is a convex function. As a well-known result from convex
analysis (see Proposition \ref{prop:polar_prop}) we have
that for any convex function $\Omega$ and any $f,g \in L^2(q)$, it holds that $\IP{f}{g}_{q} \leq \Omega_{q}(f) \Omega_{q}^{\circ}(g)$. Consequently,
\begin{align}
\ell(g, f)_{q} + \l \Omega_{q}(f)\Omega_{q}^{\circ}\left(-\frac{1}{\l}\nabla_{\hat{Y}}\ell\left(g, \Phi_r(\{W_j\})\right)\right)
\geq {NC_{q}(\{W_j\})}
+ \frac{\alpha}{2}\nmm{f-\Phi_r(\{W_j\})}_{q}^2.
\end{align}
Therefore, rearranging,
\begin{align}
\underbrace{\ell(g, f)_{q} + \l \Omega_{q}(f)}_{\CP_{q}(f)}  + \l \Omega_{q}(f)\left[\Omega_{q}^{\circ}\left(-\frac{1}{\l}\nabla_{\hat{Y}}\ell\left(g, \Phi_r(\{W_j\})\right)\right)-1\right]
\geq {NC_{q}(\{W_j\})}
+ \frac{\alpha}{2}\nmm{f-\Phi_r(\{W_j\})}_{q}^2,
\end{align}
and, as a result,
\be\label{eq:ncvx_upper}
\NC_{q}(\{W_j\}) \leq \CP_{q}(f) + \l \Omega_{q}(f)\left[\Omega_{q}^{\circ}\left(-\frac{1}{\l}\nabla_{\hat{Y}}\ell\left(g, \Phi_r(\{W_j\})\right)\right)-1\right]
- \frac{\alpha}{2}\nmm{f-\Phi_r(\{W_j\})}_{q}^2.
\ee
Let $f^*_{q} = \arg\min_{f}\CP_{q}(f)$,  and 
$(r^*, \{W_j^*\}) = \argmin_{r, \{W_j\}}NC_{q}(\{W_j\})$.  Since $f_{q}^{*}$ is the minimizer of $C_{q}(f)$, it holds that
\begin{align}
\CP_{q}(f^*_{q}) \leq \CP_{q}(\Phi_{r^*}(\{W_j^*\}))
= \ell(g, \Phi_{r^*}(\{W_j^*\}))_{q}
+ \l \Omega_{q}(\Phi_{r^*}(\{W_j^*\})).
\end{align}
Therefore, we obtain
\begin{align}
\CP_{q}(f^*_{q}) &\leq
\ell(g, \Phi_{r^*}(\{W_j^*\}))_{q}
+ \l \Omega_{q}(\Phi_{r^*}(\{W_j^*\})) \\
&\leq \ell(g, \Phi_{r^*}(\{W_j^*\}))_{q} + \l \Theta_{r^*}(\{W_j^*\}) \\
&= \NC_{q}(\{W_j^*\}) \\
&\leq \NC_{q}(\{W_j\}). \label{eq:min_cvx}
\end{align}
Therefore, combining Equations \eqref{eq:min_cvx} and \eqref{eq:ncvx_upper}, we obtain the bound
\begin{align}
\CP_{q}(f^*_{q}) \leq \NC_{q}(\{W_j\}) \leq 
\CP_{q}(f) 
&+ \l \Omega_{q}(f)\left[\Omega_{q}^{\circ}\left(-\frac{1}{\l}\nabla_{\hat{Y}}\ell\left(g, \Phi_r(\{W_j\})\right)\right)-1\right]
- \frac{\alpha}{2}\nmm{f-\Phi_r(\{W_j\})}_{q}^2.
\end{align}
\end{proof}

Lemma \ref{lemma:opt_gap} has established that
the non-convex objective, $\NC_{q}(\cdot)$ is both upper and lower bounded
by the convex function, $\CP_{q}(\cdot)$. Now, we utilize this
result to compute the empirical gap with the measure,
$\mu_N$ for the stationary
points obtained from the ERM. On these
stationary points, we bound the optimality
gap by changing the measure to $\mu$,
i.e., the behavior of ERM's first-order points
on population landscape.

\begin{theorem}[Global Optimality]%\label{thm:thm1}
Under assumptions \ref{ass:a1}, \ref{ass:apdx_a2}, \ref{ass:a4}.
Let $f^*_{\mu_N}$ (or $f^*_{\mu}$) be the 
global minimizer for $\CP_{\mu_N}(\cdot)$ (or $\CP_{\mu}(\cdot))$.
For any
stationary points, $(r, \{W_j\})$ 
of the function $\NC_{\mu_{N}}(\cdot)$
and any $f \in L^2(\mu) \cap L^2(\mu_N)$ the
following items are true:
\begin{enumerate}
\item[1.] Empirical optimality gap:
\end{enumerate}
\be \label{eq:apdx_ncvx_emp_gap}
\CP_{\mu_{N}}(f^*_{\mu_{N}}) \leq \NC_{\mu_{N}}(\{W_j\}) \leq 
\CP_{\mu_{N}}(f) + \l \Omega(f)\left[\Omega_{\mu_{N}}^{\circ}\left(-\frac{1}{\l}\nabla_{\hat{Y}}\ell\left(g, \Phi_r(\{W_j\})\right)\right)-1\right]
- \frac{\alpha}{2}\nmm{f-\Phi_r(\{W_j\})}_{\mu_{N}}^2,
\ee
\begin{enumerate}
\item[2.] Population optimality gap:
\end{enumerate}
\bd
\CP_{\mu}(f^*_{\mu}) \leq \NC_{\mu}(\{W_j\}) \leq 
\CP_{\mu}(f) + \l \Omega(f)\left[\Omega_{\mu}^{\circ}\left(-\frac{1}{\l}\nabla_{\hat{Y}}\ell\left(g, \Phi_r(\{W_j\})\right)\right)-1\right]
- \frac{\alpha}{2}\nmm{f-\Phi_r(\{W_j\})}_{\mu}^2
\ed
\be %\label{eq:ncvx_pop_gap}
+ \left[\IP{\nabla_{\hat{Y}}\ell\left(g, \Phi_r(\{W_j\})\right)}{\Phi_r(\{W_j\})}_{\mu}
- \IP{\nabla_{\hat{Y}}\ell\left(g, \Phi_r(\{W_j\})\right)}{\Phi_r(\{W_j\})}_{\mu_N}\right].
\ee
where $\Omega_{q}^{\circ}(\cdot)$ is
 the polar in the measure $q$ defined as
\be %\label{eq:polar}
\Omega_{q}^{\circ}(g) := \sup_{\theta(W) \leq 1}\IP{g}{\phi(W)}_{q}
\ee
\end{theorem}

\textbf{Remarks}: Setting $f = f_{\mu_N}^*$ in
Equation \ref{eq:apdx_ncvx_emp_gap} and taking $(r,\{W_j\})$ to be any stationary point of $\NC_{\mu_N}(\cdot)$ gives a means to verify if $\{W_j\}$ is a globally optimal solution.  We see that it suffices to check if 
$\Phi_r(\{W_j\})$ is
a first-order stationary point of $\CP_{\mu_N}(\cdot)$, 
which is a necessary condition for a local minimum of convex functions.

From convex analysis,
if a function $f \in L^2(\mu_N)$ is a first-order
solution of $\CP_{\mu_N}$ then we have that
$0$ belongs to the sub-gradient of $\CP_{\mu}(\cdot)$ at $f$.
 As the loss $\ell$ is
first-order differentiable (by Assumption \ref{ass:a4} or \ref{ass:apdx_a2}) we have that
\begin{align}
0 \in \partial \CP_{\mu_N}(f)
\iff -\frac{1}{\l}\nabla_{\hat{Y}}\ell(g, f)_{\mu_N}
\in \partial \Omega(f),
\end{align}
where $\partial \CP_{\mu_N}(f)$ denotes the subgradient of $\CP$ (viewed as a function of $f$).  The above condition for $f$
can also be verified by a dual notion
known as the polar
condition, Definition \ref{def:polar_func} 
\citep{rockafellar_convex_1970}. The sub-gradient
of a convex function can be defined through
the notion of it's polar via
\begin{align}
\partial \Omega_{\mu_N}(f) 
= \left\{g \in L_{2}(\mu_N): \IP{g}{f}_{\mu_N} = \Omega_{\mu_N}(f), \Omega^{\circ}_{\mu_N}(g) \leq 1\right\}.
\end{align}
From Lemma 1 in the supplement of \cite{haeffele2017global}  the following statements
are equivalent:
\begin{enumerate}
\item $\{W_j\}$ is an optimal factorization of $f$; i.e, $\Theta_r(\{W_j\}) = \Omega_{\mu_N}(f)$.
\item $\exists h \in L^2(\mu_N)$ such that $\Omega^{\circ}_{\mu_N}(h) \leq 1$ and
$\IP{h}{\Phi_r(\{W_j\})}_{\mu_N} = \Theta_r(\{W_j\})$.
\item $\exists h \in L^2(\mu_N)$ such that $\Omega^{\circ}_{\mu_N}(h) \leq 1$ and
$\IP{h}{\phi(W_j)}_{\mu_N} = \theta(W_j); \forall i \in [r]$.
\end{enumerate}
Further, if (2) or (3) above is satisfied then we have that $h \in \partial \Omega_{\mu_N}(f)$.
From Proposition \ref{prop:stat_points} we have that for
any stationary point $(r, \{W_j\})$ of 
$\NC_{\mu_N}$,  
\begin{align}
\IP{-\frac{1}{\l}\nabla \ell_{\mu_N}(g, \Phi_r(\{W_j\}))}{\Phi_r(\{W_j\})}_{\mu_N} = \Theta_r(\{W_j\}).
\end{align}
Consequently, to check if a stationary point is globally optimal, it then suffices to check whether the polar condition
$\Omega^{\circ}_{\mu_N}\left(-\frac{1}{\lambda}\nabla \ell(g, \Phi_r(\{W_j\}))_{\mu_N}\right) \leq  1$ holds 
at the stationary point, $(r, \{W_j\})$.
In the case when the polar condition holds true,
the upper bound evaluates 
to $\CP_{\mu_N}(f_{\mu}^*)$ matching the lower bound
of $\NC_{\mu_N}(\cdot)$, which in turn  implies global optimality.

Then, we can claim the following:
\begin{quote}
    At a stationary point  $\{W_j\}$, if $\Omega^{\circ}_{\mu_N}\left(-\frac{1}{\lambda}\nabla \ell(g, \Phi_r(\{W_j\}))_{\mu_N}\right) \leq 1$, then $\{W_j\}$  is globally optimal.
\end{quote}

Now we prove Theorem \ref{thm:thm1}.

\begin{proof}
The proof sketch is similar to Proposition 4 from \cite{haeffele-vidal-tpami20}.
Equation \eqref{eq:ncvx_emp_gap} can be obtained from the Lemma \ref{lemma:opt_gap},
for any stationary points, $(r, \{W_j\})$ of $\NC_{\mu_N}(\cdot)$.

Since, $f \in L^2(\mu) \cap L^2(\mu_N) \subseteq L^2(\mu_n)$, and the parameters satisfy the equality in Lemma
\ref{lemma:opt_gap}, we can conclude that Equation \eqref{eq:ncvx_emp_gap} holds. The local minima of
$\NC_{\mu_N}(\cdot)$ need not be local minima of $\NC_{\mu}(\cdot)$, therefore we shall obtain an
discrepency term. From the fact that $\ell$ is a $\alpha$-strongly convex function we have the inequality
\begin{align}
\ell(g, f)_{\mu} \geq \ell_{\mu}\left(g, \Phi_r(\{W_j\})\right) + \IP{\nabla_{\hat{Y}}\ell\left(g, \Phi_r(\{W_j\})\right)}{f-\Phi_r(\{W_j\})}_{\mu}
+ \frac{\alpha}{2}\nmm{f-\Phi_r(\{W_j\})}_{\mu}^2.
\end{align}
Adding $\l \Theta_r(\{W_j\})$ on both sides we obtain the inequality
\begin{align}
\ell(g, f)_{\mu} + \l \Theta_r(\{W_j\}) &\geq {\ell_{\mu}\left(g, \Phi_r(\{W_j\})\right) + \l \Theta_r(\{W_j\})}
\\
&\qquad + \IP{\nabla_{\hat{Y}}\ell\left(g, \Phi_r(\{W_j\})\right)}{f-\Phi_r(\{W_j\})}_{\mu}
+ \frac{\alpha}{2}\nmm{f-\Phi_r(\{W_j\})}_{\mu}^2.
\end{align}
Now replacing the first term term on the side with $\NC_{\mu}(\{W_j\})$ we obtain
\begin{align}
\ell(g, f)_{\mu} + \l \Theta_r(\{W_j\}) &\geq \NC_{\mu}(\{W_j\})
+ \IP{\nabla_{\hat{Y}}\ell\left(g, \Phi_r(\{W_j\})\right)}{f-\Phi_r(\{W_j\})}_{\mu}
+ \frac{\alpha}{2}\nmm{f-\Phi_r(\{W_j\})}_{\mu}^2.
\end{align}

From Proposition \ref{prop:stat_points} we have that for stationary points $\{W_j\}$, it holds that $\Theta_r(\{W_j\}) = \IP{-\frac{1}{\l}\nabla_{\hat{Y}} \ell_{\mu_N}(g, \Phi_r(\{W_j\}))}{\Phi_r(\{W_j\})}_{\mu_N}$.
Therefore, by plugging this into the inequality above, we obtain that
\begin{align}
\ell(g, f)_{\mu} + & \l \IP{-\frac{1}{\l}\nabla_{\hat{Y}} \ell_{\mu_N}(g, \Phi_r(\{W_j\}))_{\mu_N}}{\Phi_r(\{W_j\})} \\
&\geq \NC_{\mu}(\{W_j\})
+ \IP{\nabla_{\hat{Y}}\ell\left(g, \Phi_r(\{W_j\})\right)}{f-\Phi_r(\{W_j\})}_{\mu}
+ \frac{\alpha}{2}\nmm{f-\Phi_r(\{W_j\})}_{\mu}^2. \label{eq:penultimateeq}
\end{align}
Rearranging the terms we have
\begin{align}
\ell(g, f)_{\mu} + &  \l \IP{-\frac{1}{\l}\nabla_{\hat{Y}}\ell\left(g, \Phi_r(\{W_j\})\right)}{f}_{\mu} \\
&+  \IP{\nabla_{\hat{Y}} \ell_{\mu}(g, \Phi_r(\{W_j\}))}{\Phi_r(\{W_j\})}_{\mu}-\IP{\nabla_{\hat{Y}} \ell_{\mu_N}(g, \Phi_r(\{W_j\}))}{\Phi_r(\{W_j\})}_{\mu_N} \\
&\geq \NC_{\mu}(\{W_j\})
+ \frac{\alpha}{2}\nmm{f-\Phi_r(\{W_j\})}_{\mu}^2. \label{eq:penultimateeq}
\end{align}
Next, the following inequality always holds:
\begin{align}
\IP{f}{-\frac{1}{\l}\nabla_{\hat{Y}}\ell\left(g, \Phi_r(\{W_j\})\right)}_{\mu} \leq \Omega_{\mu}(f)\Omega_{\mu}^{\circ}\left(-\frac{1}{\l}\nabla_{\hat{Y}}\ell\left(g, \Phi_r(\{W_j\})\right)\right). \label{polarinequality}
\end{align}

Rearranging \eqref{eq:penultimateeq} and plugging in \eqref{polarinequality}, we obtain the inequality
\begin{align}
\ell(g, f)_{\mu} + &  \l \Omega_{\mu}(f)\Omega_{\mu}^{\circ}\left(-\frac{1}{\l}\nabla_{\hat{Y}}\ell\left(g, \Phi_r(\{W_j\})\right)\right) \\
&+  \IP{\nabla_{\hat{Y}} \ell_{\mu}(g, \Phi_r(\{W_j\}))}{\Phi_r(\{W_j\})}_{\mu}-\IP{\nabla_{\hat{Y}} \ell_{\mu_N}(g, \Phi_r(\{W_j\}))}{\Phi_r(\{W_j\})}_{\mu_N} \\
&\geq \NC_{\mu}(\{W_j\})
+ \frac{\alpha}{2}\nmm{f-\Phi_r(\{W_j\})}_{\mu}^2. \label{eq:penultimateeq}
\end{align}
We add and subtract $\Omega(f)$ to obtain
\begin{align}
\ell(g, f)_{\mu} + \l \Omega_{\mu}(f) + &  \l \Omega_{\mu}(f)\left[\Omega_{\mu}^{\circ}\left(-\frac{1}{\l}\nabla_{\hat{Y}}\ell\left(g, \Phi_r(\{W_j\})\right)\right) - 1\right] \\
&+  \IP{\nabla_{\hat{Y}} \ell_{\mu}(g, \Phi_r(\{W_j\}))}{\Phi_r(\{W_j\})}_{\mu}-\IP{\nabla_{\hat{Y}} \ell_{\mu_N}(g, \Phi_r(\{W_j\}))}{\Phi_r(\{W_j\})}_{\mu_N} \\
&\geq \NC_{\mu}(\{W_j\})
+ \frac{\alpha}{2}\nmm{f-\Phi_r(\{W_j\})}_{\mu}^2. \label{eq:penultimateeq}
\end{align}
Rearranging the right most term and using the definition of $C_{\mu}(f)$ we obtain,
\begin{align}
\NC_{\mu}(\{W_j\}) &\leq \CP_{\mu}(f) + \l \Omega_{\mu}(f)\left[\Omega_{\mu}^{\circ}\left(-\frac{1}{\l}\nabla_{\hat{Y}}\ell\left(g, \Phi_r(\{W_j\})\right)\right)-1\right]
- \frac{\alpha}{2}\nmm{f-\Phi_r(\{W_j\})}_{\mu}^2 \\
&+
\left[\IP{\nabla_{\hat{Y}} \ell(g, \Phi_r(\{W_j\}))_{\mu_N}}{\Phi_r(\{W_j\})}_{\mu} - \IP{\nabla_{\hat{Y}} \ell(g, \Phi_r(\{W_j\}))_{\mu_N}}{\Phi_r(\{W_j\})}_{\mu_N}\right].
\end{align}
This yields the right hand side of \eqref{eq:ncvx_pop_gap}.
As for the left hand side, by definition we have that for any $(r,\{W_j\})$, $\CP_{\mu}(f^*_{\mu}) \leq \NC_{\mu}(\{W_j\})$. This completes the proof.
\end{proof}

Theorem \ref{thm:thm1} provides the behavior of ERM
solutions in the population landscape. This paves a path
to bound the empirical and population objectives
at these stationary points.

% \newpage
\section{STATISTICAL BOUNDS} \label{sec:appdx_master_theorem}

This section provides a more general version of Theorem \ref{thm:master}
that does not need Assumption \ref{ass:a8} to hold uniformly
for all the data points, $(X, Y)$. Rather, we relax the assumption to the following.
%\ref{ass:a9}.

\renewcommand{\theassumption}{7'}
\begin{assumption}[Probabilistic boundedness] \label{ass:a9} 
There exists a convex set, $\C \subseteq \R^{n_X} \times \R^{n_Y}$ such that
\be
P(\cap_{i=1}^{N}(X_i, \epsilon_i) \in \C) \geq 1 - \delta_{\C}.
\ee
For all $(X, \epsilon) \in \C$, and $\{W_j\} \in \F_{\mathcal{W}}$ 
the predictions and gradients are bounded; i.e.,
\be\label{eq:bounded_apdx}
\nmm{\Phi_r(\{W_j\})(X)} \leq B_{\Phi}\text{, }
\nmm{\nabla_{\hat{Y}}\ell(g(X, \epsilon), \Phi_r(\{W_j\})(X)} \leq B_{\ell}.
\ee
Further, for any $\{W_j\}, \{\tilde{W}_j'\} \in \F_{\mathcal{W}}$, $W, \tilde{W} \in \F_{\theta}$, $(X, \epsilon) \in \C$,
the network $\phi$ and $\Phi_r$ are Lipschitz in the parameters; i.e,
\be\label{eq:lip_apdx}
\nmm{\Phi_r(\{W_j\})(X) - \Phi_r(\{\tilde{W}_j\})(X)}
\leq \tilde{L}_{\Phi}\max_{j}\nmm{W_j - \tilde{W}_j}_2,
\ee
and
\be\label{eq:lip_1_apdx}
\nmm{\phi(W)(X) - \phi(\tilde{W})(X)}
\leq \tilde{L}_{\phi}\nmm{W - \tilde{W}}_2.
\ee
Additionally, define the quantity
\begin{align}
    B(\C) &:= \Big[(1+\alpha)\sup_{\{W_j\} \in \F_{\mathcal{W}}}\left|\nmm{f_{\mu}^* \circ \P_{\C} - \Phi_r(\{W_j\}) \circ \P_{\C}}_{\mu}^2
-\nmm{f_{\mu}^* - \Phi_r(\{W_j\})}_{\mu}^2\right| \\
&+ {\scalebox{0.9}{$\sup_{\{W_j\} \in \F_{\mathcal{W}}, W' \in \F_{\theta}}\left|\IP{\nabla_{\hat{Y}}\ell\left(g \circ \P_{\C}, \Phi_r(\{W_j\}) \circ \P_{\C}\right)}{\phi(W') \circ \P_{\C}}_{\mu}
- \IP{\nabla_{\hat{Y}}\ell\left(g, \Phi_r(\{W_j\})\right)}{\phi(W')}_{\mu}\right|$}}\\
&+ {\scalebox{0.9}{$\sup_{\{W_j\} \in \F_{\mathcal{W}}}\left|\IP{\nabla_{\hat{Y}}\ell\left(g \circ \P_{\C}, \Phi_r(\{W_j\}) \circ \P_{\C}\right)}{\Phi_r(\{W_j\}) \circ \P_{\C}}_{\mu}
- \IP{\nabla_{\hat{Y}}\ell\left(g, \Phi_r(\{W_j\})\right)}{\Phi_r(\{W_j\})}_{\mu}\right|$}}\Big],
\end{align}
where $\P_{\C}(\cdot)$ is the Euclidean projection to the set $\C$. 
\end{assumption}

\textbf{Comparison with Assumption \ref{ass:a8}:}
Unlike in Assumption \ref{ass:a8}, we do not require the
equations \eqref{eq:bounded_apdx}, \eqref{eq:lip_apdx},
and \eqref{eq:lip_1_apdx} to hold for all
the inputs. However, we relax this restriction
by assuming that there exists a convex set, $\C$
which consists of the data points with probability at least $1-\delta_{\C}$.
For well-behaved probability distributions like sub-Gaussian distributions
(see Assumption \ref{ass:apdx_a2}), such a convex exists
with very high probability; i.e., very small $\delta_{\C}$.

Now we state the general master theorem that relies on Assumptions \ref{ass:a1}, \ref{ass:apdx_a2},
\ref{ass:a4}, \ref{ass:a3}, \ref{ass:a5}, \ref{ass:a7} and \ref{ass:a9} (but not on Assumption \ref{ass:a8}).

\begin{theorem}[General Master Theorem]\label{thm:gen_master}
Suppose Assumptions \ref{ass:a1}, \ref{ass:apdx_a2},
\ref{ass:a4}, \ref{ass:a3}, \ref{ass:a5}, \ref{ass:a7} and \ref{ass:a9} hold.  %Let $\delta\in$
%Under the assumptions \ref{ass:a1}-\ref{ass:a8}. 
Let $\delta \in (0, 1]$ be fixed,
and let $f_{\mu}^*$ be the global optimum of $\CP_{\mu}$. Suppose that $\gamma \geq \Omega(f_{\mu}^*)L_{\phi}$, and define
\begin{align}
    \epsilon_1 &= 16\gamma^2\sigma_X^2\max\left\{1,
\frac{L}{4}\left[1 + \frac{\nmL{g}^2}{\gamma^2}\left(1 + \frac{\sigma_{Y|X}^2}{\sigma_X^2}\right)\right]\right\}; \\
\epsilon_2 &= 4\tilde{L}_{\Phi}B_{\Phi}\max\left\{1, 2L + \frac{2B_{\ell}}{B_{\Phi}}, 8\Omega(f_{\mu}^*)\frac{B_{\ell}\tilde{L}_{\phi}}{\tilde{L}_{\Phi}B_{\Phi}}, 8L\Omega(f_{\mu}^*)\right\}.
\end{align}
Let $\{W_j\}$ denote any stationary point of
$\NC_{\mu_N}(\cdot)$.
Then with probability at least $1 - (\delta + \delta_{\C})$, it holds that 
\begin{align}
\frac{1}{n_Y}&\left|\NC_{\mu}(\{W_j\}) - \NC_{\mu_N}(\{W_j\})\right| \\
&\lesssim
\frac{\l}{n_Y} \Omega(f_{\mu}^*)\left[\Omega_{\mu_N}^{\circ}\left(-\frac{1}{\l}\nabla_{\hat{Y}}\ell\left(g, \Phi_r(\{W_j\})\right)\right)-1\right] 
- \frac{\alpha}{2n_Y}\nmm{f_{\mu}^*-\Phi_r(\{W_j\})}_{\mu_N}^2
\\
&+ \frac{B(\C)}{n_Y} + (1+\alpha)\epsilon_1\sqrt{\frac{R\cdot{\sf dim}(\mathcal{W})
\log\left({\gamma \epsilon_2r_{\theta}}/{L_{\phi}}\right)\log(N)+\log(1/\delta)}{N}}.
\end{align}
\end{theorem}

\textbf{Additional Remarks:} In addition to the discussion
in Section \ref{sec:main_results}, the general version mentioned
above (i)  takes into account unbounded
sub-Gaussian distributions, and (ii) imposes a weaker notion of
Lipschitz continuity on the parameters. For sub-Gaussian inputs,
one may choose the convex set $\C$ to be a ball with radius
$\mathbb{B}(g)$. As we grow $g$, the term $B(\C)$ decays exponentially,
and $\epsilon_2$ only grows in the order of polynomial. 
This fast decay allows us to keep the statistical error
under control while pertaining to the optimal sample complexity.
Theorem \ref{thm:gen_master} reduces to Theorem \ref{thm:master} by setting $\alpha = 0$ 
and $\C = \mathsf{conv}(\X)$. Under this choice, Assumption \ref{ass:a9} coincides with Assumption 
\ref{ass:a8}.

We discuss the proof in section \ref{sec:apdx_gen_master_proof}.
Before diving into the proof, we discuss a few preliminaries on the covering number essential to estimate the capacity of the hypotheses class. 

\subsection{Computing Function Class Capacities}

\begin{lemma}[Covering number of $\F_{\mathcal{W}}$]\label{lemma:cn_bigclass}
Under assumption \ref{ass:a7}, and \ref{ass:a9} 
the $\nu$-net covering number of the set $\F_{\mathcal{W}}$ on the metric, $\nmm{.}_{\infty, d}$
is upper bounded via
\be
\C_{\F_{\mathcal{W}}}(\nu) \leq \left(\C_{\theta}(L_{\phi}\nu/\gamma)\right)^{R},
\ee
where $\C_{\theta}(\nu) := \N(\{W: \theta(W) \leq 1\}, d(., .), \nu)$.
\end{lemma}

\begin{proof}
Recall that
\begin{align}
\C_{\F_{\mathcal{W}}}(\nu) &:= \N(\F_{\mathcal{W}}, \max_{i}d(., .), \nu); \\
\C_{\theta}(\nu) &:= \N(\F_{\theta}, d(., .), \nu).
\end{align}
By the definition of $\nu$ covering number,
\begin{align}
\N(\F_{\mathcal{W}}, \nmm{.}_{\infty, d}, \nu)
&:= \inf\left|\left\{\{W_j^0\} \in \F_{\mathcal{W}}: \forall \{W_j\} \in \F_{\mathcal{W}}: \nmm{\{W_j\} - \{W_j^0\}}_{\infty, d} = \max_{j}d(W_j, W_j^0) \leq \nu \right\}\right|;\\
\N(\F_{\theta}, d(\cdot, \cdot), \nu)
&:= \inf\left|\left\{W^0 \in \F_{\theta}: \forall W \in \F_{\theta}: d(W, W^0) \leq \nu \right\}\right|.
\end{align}
Therefore we can upper bound $\N(\F_{\mathcal{W}}, \nmm{.}_{\infty, d}, \nu)$ 
witht he product of $\N(\F_{\theta}, d(\cdot, \cdot), \nu)$ $R$ times. We have
\begin{align}
\N(\F_{\mathcal{W}}, \nmm{.}_{\infty, d}, \nu)
\leq \left[\N\left(\frac{\gamma}{L_{\phi}}\F_{\theta}, d(., .), \nu\right)\right]^{R},
\end{align}
Rewriting the above for appropriately chosen $\nu$ we get
\begin{align}
\N(\F_{\mathcal{W}}, \nmm{.}_{\infty, d}, \nu)
\leq \left[\N\left(\F_{\theta}, d(., .), \frac{L_{\phi}\nu}{\gamma}\right)\right]^{R}.
\end{align}
This concludes our proof.
\end{proof}

\begin{lemma}[Bounding covering number]\label{lemma:
bounding_cn}
Consider a metric space, $(\mathcal{W} \subseteq \R^{n}, \nmm{\cdot}_2)$
and a 
compact set, $\F_{\mathcal{W}} \subseteq \mathcal{W}$.
Suppose that there exist $r < \infty$ such that
$\F_{\mathcal{W}} \subseteq \mathbb{B}(r)$.
Then we have
\be
\N(\F_{\mathcal{W}}, \nmm{\cdot}_2, \nu)
\leq \left(1 + \frac{2r}{\nu}\right)^{n}.
\ee
\end{lemma}

\begin{proof}
We have that
$\F_{\mathcal{W}} \subseteq \mathbb{B}(r)$. By
monotonicity of covering numbers, we have that
\be
\N(\F_{\mathcal{W}}, \nmm{\cdot}_2, \nu)
\leq \N(\mathbb{B}(r), \nmm{\cdot}_2, \nu).
\ee
From Corollary 4.2.13 in
\cite{vershynin_high-dimensional_2018}
we have that,
\be
\N(\F_{\mathcal{W}}, \nmm{\cdot}_2, \nu)
\leq \N(\mathbb{B}(r), \nmm{\cdot}_2, \nu)
\leq \left(1 + \frac{2r}{\nu}\right)^{n}.
\ee
\end{proof}

\subsection{Proof of Theorem \ref{thm:master}} \label{sec:apdx_gen_master_proof}

This section discusses the proof of Theorem \ref{thm:gen_master}.
We extensively use concentration results from Section \ref{sec:good_events} that are preliminaries
for the upcoming technical details.

\begin{proof}
First, we recall the definition of generalization error:
\begin{align}
\text{Generalization Error} := \left|\NC_{\mu}(\{W_j\}) - \NC_{\mu_N}(\{W_j\})\right|.
\end{align}
We can bound the above from the optimality gaps obtained in Theorem \ref{thm:thm1} via the following decomposition:
\begin{align}
\NC_{\mu}(\{W_j\}) - \NC_{\mu_N}(\{W_j\}) &=
\underbrace{\left[\NC_{\mu}(\{W_j\}) - \CP_{\mu}(f_{\mu})\right]}_{\text{Population Gap}} -
\underbrace{\left[\NC_{\mu_N}(\{W_j\}) - \CP_{\mu_N}(f_{\mu_N})\right]}_{\text{Empirical Gap}}
\\
&\qquad + \underbrace{\left[\CP_{\mu}(f_{\mu}) - \CP_{\mu_N}(f_{\mu_N})\right]}_{\text{Convex Gap}}.
\end{align}

From Theorem \ref{thm:thm1} we have that
for any $f_{\mu}, f_{\mu_N} \in L^2(\mu) \cap L^2(\mu_N)$ and stationary points $(r, \{W_j\})$, the empirical
gap is bounded by
\begin{align}
    \CP_{\mu_N}(f_{\mu_N}^*) - \CP_{\mu_{N}}(f_{\mu_N}) &\leq
\NC_{\mu_N}(\{W_j\}) - \CP_{\mu_N}(f_{\mu_N}) \\
&\leq 
\l \Omega(f_{\mu_N})\left[\Omega_{\mu_{N}}^{\circ}\left(-\frac{1}{\l}\nabla_{\hat{Y}}\ell\left(g, \Phi_r(\{W_j\})\right)\right)-1\right]
- \frac{\alpha}{2}\nmm{f_{\mu_N}-\Phi_r(\{W_j\})}_{\mu_{N}}^2,
\end{align}
and the population gap is bounded by
\begin{align}
\CP_{\mu}(f_{\mu}^*) - \CP_{\mu}(f_{\mu})
&\leq \NC_{\mu}(\{W_j\}) - \CP_{\mu}(f_{\mu}) \\
&\leq 
\l \Omega(f_{\mu})\left[\Omega_{\mu}^{\circ}\left(-\frac{1}{\l}\nabla_{\hat{Y}}\ell\left(g, \Phi_r(\{W_j\})\right)\right)-1\right]
- \frac{\alpha}{2}\nmm{f_{\mu}-\Phi_r(\{W_j\})}_{\mu}^2
\\
&\quad + \left[\IP{\nabla_{\hat{Y}}\ell\left(g, \Phi_r(\{W_j\})\right)}{\Phi_r(\{W_j\})}_{\mu}
- \IP{\nabla_{\hat{Y}}\ell\left(g, \Phi_r(\{W_j\})\right)}{\Phi_r(\{W_j\})}_{\mu_N}\right].
\end{align}
For any $f_{\mu}, f_{\mu_N}$, subtracting the above two equations we obtain
\begin{align}
\CP_{\mu}(f_{\mu}^*) &- \CP_{\mu_N}(f_{\mu_N}) - \l\Omega(f_{\mu_N})\left[\Omega_{\mu_{N}}^{\circ}\left(-\frac{1}{\l}\nabla_{\hat{Y}}\ell\left(g, \Phi_r(\{W_j\})\right)\right)-1\right]
+ \frac{\alpha}{2}\nmm{f_{\mu_N}-\Phi_r(\{W_j\})}_{\mu_{N}}^2 \\
&\leq
\NC_{\mu}(\{W_j\}) - \NC_{\mu_N}(\{W_j\}) \\
&\leq
\l \Omega(f_{\mu})\left[\Omega_{\mu}^{\circ}\left(-\frac{1}{\l}\nabla_{\hat{Y}}\ell\left(g, \Phi_r(\{W_j\})\right)\right)-1\right]
- \frac{\alpha}{2}\nmm{f_{\mu}-\Phi_r(\{W_j\})}_{\mu}^2 \\
&\quad + \left[\IP{\nabla_{\hat{Y}}\ell\left(g, \Phi_r(\{W_j\})\right)}{\Phi_r(\{W_j\})}_{\mu}
- \IP{\nabla_{\hat{Y}}\ell\left(g, \Phi_r(\{W_j\})\right)}{\Phi_r(\{W_j\})}_{\mu_N}\right]
\\
&+ \left[\CP_{\mu}(f_{\mu}) - \CP_{\mu_N}(f_{\mu_N}^*)\right].
\end{align}
By choosing $f_{\mu} = f_{\mu_N} = f_{\mu}^*$ (as $f_{\mu}^* \in L^2(\mu) \cap L^2(\mu_N)$) and 
noting that $f_{\mu}^*$ is not a random variable unlike $f_{\mu_N}^*$ (which depends on the data points) we get
\begin{align}
C_{\mu}(f_{\mu}^*) -& C_{\mu_N}(f_{\mu}^*) - \l \Omega(f_{\mu}^*)\left[\Omega_{\mu_{N}}^{\circ}\left(-\frac{1}{\l}\nabla_{\hat{Y}}\ell\left(g, \Phi_r(\{W_j\})\right)\right)-1\right]
+ \frac{\alpha}{2}\nmm{f_{\mu}^*-\Phi_r(\{W_j\})}_{\mu_{N}}^2 \\
&\leq
\NC_{\mu}(\{W_j\}) - \NC_{\mu_N}(\{W_j\}) \\
&\leq
\l \Omega(f_{\mu}^*)\left[\Omega_{\mu}^{\circ}\left(-\frac{1}{\l}\nabla_{\hat{Y}}\ell\left(g, \Phi_r(\{W_j\})\right)\right)-1\right]
- \frac{\alpha}{2}\nmm{f_{\mu}^*-\Phi_r(\{W_j\})}_{\mu}^2
\\
&\quad + \left[\IP{\nabla_{\hat{Y}}\ell\left(g, \Phi_r(\{W_j\})\right)}{\Phi_r(\{W_j\})}_{\mu}
- \IP{\nabla_{\hat{Y}}\ell\left(g, \Phi_r(\{W_j\})\right)}{\Phi_r(\{W_j\})}_{\mu_N}\right] 
\\
&\quad +
\left[\CP_{\mu}(f_{\mu}^*) - \CP_{\mu_N}(f_{\mu_N}^*)\right].
\end{align}
Since $f_{\mu}^*$ is the global minimizer of $C_{\mu}(\cdot)$, it always holds  that $\CP_{\mu}(f_{\mu}^*) \le C_{\mu}(f_{\mu_N}^*)$.
We use this fact to upper bound the right side term, upon which we obtain the bound
\begin{align}
\CP_{\mu}(f_{\mu}^*) &- \CP_{\mu_N}(f_{\mu}^*) - \l \Omega(f_{\mu}^*)\left[\Omega_{\mu_{N}}^{\circ}\left(-\frac{1}{\l}\nabla_{\hat{Y}}\ell\left(g, \Phi_r(\{W_j\})\right)\right)-1\right]
+ \frac{\alpha}{2}\nmm{f_{\mu}^*-\Phi_r(\{W_j\})}_{\mu_{N}}^2\\
&\leq
\NC_{\mu}(\{W_j\}) - \NC_{\mu_N}(\{W_j\}) 
\\
&\leq
\l \Omega(f_{\mu}^*)\left[\Omega_{\mu}^{\circ}\left(-\frac{1}{\l}\nabla_{\hat{Y}}\ell\left(g, \Phi_r(\{W_j\})\right)\right)-1\right]
- \frac{\alpha}{2}\nmm{f_{\mu}^*-\Phi_r(\{W_j\})}_{\mu}^2 \\
&\quad 
+ \left[\IP{\nabla_{\hat{Y}}\ell\left(g, \Phi_r(\{W_j\})\right)}{\Phi_r(\{W_j\})}_{\mu}
- \IP{\nabla_{\hat{Y}}\ell\left(g, \Phi_r(\{W_j\})\right)}{\Phi_r(\{W_j\})}_{\mu_N}\right]\\
&\quad + \left[\CP_{\mu}(f_{\mu_N}^*) - \CP_{\mu_N}(f_{\mu_N}^*)\right].
\end{align}
Now we add and subtract $\Omega_{\mu}^{\circ}(\cdot)$\footnote{We are ignoring the input arguments for brevity.} and  
$\frac{\alpha}{2}\nmm{f - \Phi_r(\{W_j\})}_{\mu}^2$ on the right side.  We then have that 
\begin{align}
{\underbrace{\CP_{\mu}(f_{\mu}^*) - \CP_{\mu_N}(f_{\mu}^*)}_{=:T_1}}& - \l \Omega(f_{\mu}^*)\left[\Omega_{\mu_{N}}^{\circ}\left(-\frac{1}{\l}\nabla_{\hat{Y}}\ell\left(g, \Phi_r(\{W_j\})\right)\right)-1\right]
\\
&+ \frac{\alpha}{2}\nmm{f_{\mu}^*-\Phi_r(\{W_j\})}_{\mu_N}^2
\\
&\leq
\NC_{\mu}(\{W_j\}) - \NC_{\mu_N}(\{W_j\}) \\
&\leq
\l \Omega(f_{\mu}^*)\left[\Omega_{\mu_N}^{\circ}\left(-\frac{1}{\l}\nabla_{\hat{Y}}\ell\left(g, \Phi_r(\{W_j\})\right)\right)-1\right]
- \frac{\alpha}{2}\nmm{f_{\mu}^*-\Phi_r(\{W_j\})}_{\mu_N}^2
\\
&\quad + \frac{\alpha}{2}\left[{\underbrace{\nmm{f_{\mu}^*-\Phi_r(\{W_j\})}_{\mu_N}^2 - \nmm{f_{\mu}^*-\Phi_r(\{W_j\})}_{\mu}^2}_{=:T_2}}\right]\\
&+ \l \Omega(f_{\mu}^*)\left[{\underbrace{\Omega_{\mu}^{\circ}\left(-\frac{1}{\l}\nabla_{\hat{Y}}\ell\left(g, \Phi_r(\{W_j\})\right)\right)
- \Omega_{\mu_N}^{\circ}\left(-\frac{1}{\l}\nabla_{\hat{Y}}\ell\left(g, \Phi_r(\{W_j\})\right)\right)}_{=:T_3}}\right]
\\
&\quad 
+ \left[{\underbrace{\IP{\nabla_{\hat{Y}}\ell\left(g, \Phi_r(\{W_j\})\right)}{\Phi_r(\{W_j\})}_{\mu}
- \IP{\nabla_{\hat{Y}}\ell\left(g, \Phi_r(\{W_j\})\right)}{\Phi_r(\{W_j\})}_{\mu_N}}_{=:T_5}}\right] 
\\
& \label{eq:opt_form}+ \left[{\underbrace{\CP_{\mu}(f_{\mu_N}^*) - \CP_{\mu_N}(f_{\mu_N}^*)}_{=:T_5}}\right].
\end{align}
Now we apply uniform concentration on the
quantities $T_1, T_2, T_3, T_4$, and $T_5$ to get bound the statistical error terms.

From assumption \ref{ass:a9} we assume that $\C$ is some convex set in $\R^{n_X} \times \R^{n_Y}$ such that the following hold true:
\begin{enumerate}
\item For any i.i.d. samples $\{X_i, \epsilon_i\}$ the $P(\bigcap_{i=1}^{N}(X_i, \epsilon_i) \in \C) \geq 1 - \delta_{\C}$.
\item For all $(X, \epsilon) \in \C$ and $\forall \zeta \in \F_{\mathcal{W}}: \nmm{f_{\zeta}(X)} \leq B_{\Phi}$.
\item For all $(X, \epsilon) \in \C$ we have $\forall \zeta \in \F_{\mathcal{W}}:
\nmm{\nabla_{\hat{Y}}\ell(g(X, \epsilon), f_{\zeta}(X))} \leq B_{\ell}$.
\item For all $(X, \epsilon) \in \C$ and  $\forall \zeta, \zeta' \in \F_{\mathcal{W}}:
\nmm{f_{\zeta}(X)-f_{\zeta'}(X)} \leq \tilde{L}_{\Phi}\nmm{\zeta- \zeta'}_{\infty, 2}$.
\item For all $(X, \epsilon) \in \C$ and  $\forall \zeta, \zeta' \in \F_{\theta}:
\nmm{f_{\zeta}(X)-f_{\zeta'}(Z)} \leq \tilde{L}_{\phi}\nmm{\zeta- \zeta'}_2$.
\item For any $\hat{Y}_1, \hat{Y}_2 \in \R^{n_Y}$ we have  $\nmm{\nabla_{\hat{Y}}\ell(Y, \hat{Y}_1)
-\nabla_{\hat{Y}}\ell(Y, \hat{Y}_2)} \leq L\nmm{\hat{Y}_1 - \hat{Y}_2}$.
\item $B_{nrm}(\C) := \sup_{\zeta \in \F_{\mathcal{W}}}\left|\nmm{f_{\mu}^* \circ \P_{\C} - f_{\zeta} \circ \P_{\C}}_{\mu}^2
-\nmm{f_{\mu}^* - f_{\zeta}}_{\mu}^2\right| < \infty$.
\item $
%\hspace{-20pt}
B_{plr}(\C) := \sup_{\zeta \in \F_{\mathcal{W}}, \zeta' \in \F_{\theta}}\left|\IP{\nabla_{\hat{Y}}\ell\left(g \circ \P_{\C}, f_{\zeta} \circ \P_{\C}\right)}{f_{\zeta'} \circ \P_{\C}}_{\mu}
- \IP{\nabla_{\hat{Y}}\ell\left(g, f_{\zeta}\right)}{f_{\zeta'}}_{\mu}\right| < \infty.
$
\item 
$B_{eql}(\C) := \sup_{\zeta \in \F_{\mathcal{W}}}\left|\IP{\nabla_{\hat{Y}}\ell\left(g \circ \P_{\C}, f_{\zeta} \circ \P_{\C}\right)}{f_{\zeta} \circ \P_{\C}}_{\mu}
- \IP{\nabla_{\hat{Y}}\ell\left(g, f_{\zeta}\right)}{f_{\zeta}}_{\mu}\right| < \infty.$

\end{enumerate}

Next, we define the events
\begin{align*}
    \E_{cvx}(\epsilon) &:= \{\forall \zeta \in \F_{\mathcal{W}}: \left|\CP_{\mu_N}(f_{\zeta}) - \CP_{\mu}(f_{\zeta})\right| \leq \epsilon
+ B_{nrm}(\C)\}; \\
\E_{eql}(\epsilon) &:= \{\forall \zeta \in \F_{\mathcal{W}}: \left|\IP{\nabla_{\hat{Y}}\ell\left(g, f_{\zeta}\right)}{f_{\zeta}}_{\mu}
- \IP{\nabla_{\hat{Y}}\ell\left(g, f_{\zeta}\right)}{f_{\zeta}}_{\mu_N}\right| \leq \epsilon + B_{eql}(\C)\}.
\end{align*}
Since $\Omega^{\circ}(\cdot)$ is positively homogeneous function we can ignore the scalar $-\frac{1}{\l}$ while defining the events below:
\begin{align}
\E_{plr}(\epsilon) &:= \{\forall \zeta \in \F_{\mathcal{W}}: \left|\Omega_{\mu_N}^{\circ}\left(\nabla_{\hat{Y}}\ell\left(g, f_{\zeta}\right)\right)
- \Omega_{\mu}^{\circ}\left(\nabla_{\hat{Y}}\ell\left(g, f_{\zeta}\right)\right)\right| \leq \epsilon
+ B_{plr}(\C)\}; \\
\E_{nrm}(\epsilon) &:= \{\forall \zeta \in \F_{\mathcal{W}}: \left|\nmm{f_{\mu}^*-f_{\zeta}}_{\mu_N}^2 - \nmm{f_{\mu}^*-f_{\zeta}}_{\mu}^2\right| \leq \epsilon
+ B_{nrm}(\C)\}.
\end{align}
Finally, define the following good event:
\begin{align}
\E_{good}(\epsilon) := \E_{cvx}\left(\frac{\epsilon}{4}\right) \cap \E_{eql}\left(\frac{\epsilon}{4}\right) \cap \E_{plr}\left(\frac{\epsilon}{4\Omega(f_{\mu}^*)}\right)
\cap \E_{nrm}\left(\frac{\epsilon}{2}\right).
\end{align}
When the event $\E_{good}(\epsilon)$ holds then we obtain following from the inequality \eqref{eq:opt_form},
\begin{align}
{-\epsilon/4 - B_{cvx}(\C)} - &\l \Omega(f_{\mu}^*)\left[\Omega_{\mu_{N}}^{\circ}\left(-\frac{1}{\l}\nabla_{\hat{Y}}\ell\left(g, \Phi_r(\{W_j\})\right)\right)-1\right] + \frac{\alpha}{2}\nmm{f_{\mu}^*-\Phi_r(\{W_j\})}_{\mu_N}^2 \\
&\leq
\NC_{\mu}(\{W_j\}) - \NC_{\mu_N}(\{W_j\}) \\
&\leq
\l \Omega(f_{\mu}^*)\left[\Omega_{\mu_N}^{\circ}\left(-\frac{1}{\l}\nabla_{\hat{Y}}\ell\left(g, \Phi_r(\{W_j\})\right)\right)-1\right]
- \frac{\alpha}{2}\nmm{f_{\mu}^*-\Phi_r(\{W_j\})}_{\mu_N}^2  \\
&+ \frac{\alpha}{2}\left[{\epsilon/2} + B_{nrm}(\C)\right] + \l \Omega(f_{\mu}^*)\left[{\epsilon/(4\l\Omega(f_{\mu}^*)) + B_{plr}(\C)}\right] \\
&+ \left[{\epsilon/4 + B_{eql}(\C)}\right] + \left[{\epsilon/4 + B_{nrm}(\C)}\right].
\end{align}
For $\alpha \geq 0$, these inequalities  imply that
\begin{align}
\left|NC_{\mu}(\{W_j\}) - NC_{\mu_N}(\{W_j\})\right| &\leq
\l \Omega(f_{\mu}^*)\left[\Omega_{\mu_N}^{\circ}\left(-\frac{1}{\l}\nabla_{\hat{Y}}\ell\left(g, \Phi_r(\{W_j\})\right)\right)-1\right] \\
\label{eq:final_genb}
&\quad - \frac{\alpha}{2} \nmm{f_{\mu}^*-\Phi_r(\{W_j\})}_{\mu_N}^2 + (1+\alpha)\epsilon
+ (1 + \alpha )B_{nrm}(\C) + B_{eql}(\C) + \l\Omega(f_{\mu}^*)B_{plr}(\C).
\end{align}
Equation \eqref{eq:final_genb} holds with probability $\mathbb{P}(\E_{good}(\epsilon))$. We can bound
the good event with union bound via
\begin{align}\label{eq:e_good}
\mathbb{P}(\E_{good}(\epsilon)) \geq 1 
- \underbrace{\mathbb{P}\left(\E_{nrm}^{c}\left(\frac{\epsilon}{2}\right)\right)}_{\text{Lemma \ref{lemma:cnc_nrms}}}
- \underbrace{\mathbb{P}\left(\E_{cvx}^{c}\left(\frac{\epsilon}{4}\right)\right)}_{\text{Lemma \ref{lemma:cnc_cvx_func}}}
- \underbrace{\mathbb{P}\left(\E_{eql}^{c}\left(\frac{\epsilon}{4}\right)\right)}_{\text{Lemma \ref{lemma:cnc_eql}}}
- \underbrace{\mathbb{P}\left(\E_{plr}^{c}\left(\frac{\epsilon}{4\l\Omega(f_{\mu}^*)}\right)\right)}_{\text{Lemma \ref{lemma:cnc_plr}}}
\end{align}
Under Assumptions \ref{ass:a1}-\ref{ass:a7} and \ref{ass:a9} we can apply 
Lemma \ref{lemma:cnc_nrms} , \ref{lemma:cnc_cvx_func}, \ref{lemma:cnc_eql}, and \ref{lemma:cnc_plr}
to bound the probability of the occurrence of the events, $\E_{cvx}(\cdot), \E_{eql}(\cdot), \E_{plr}(\cdot)$,
and $\E_{nrm}(\cdot)$.

Define the constants
\begin{align}
    B_1 &:= 4n_YL\left[(\gamma^2 + \nmL{g}^2)\sigma_X^2 + \nmL{g}^2\sigma_{Y|X}^2\right]; \\
B_2 &:= 16n_y\gamma\nmL{\nabla_{\hat{Y}}\ell}\sigma_X\sqrt{(\gamma^2 + \nmL{g}^2)\sigma_X^2 + \nmL{g}^2\sigma_{E|X}^2}; \\
B_3 &:= 16n_y\Omega(f_{\mu}^*)L_{\phi}\nmL{\nabla_{\hat{Y}}\ell}\sigma_X\sqrt{(\gamma^2 + \nmL{g}^2)\sigma_X^2 + \nmL{g}^2\sigma_{E|X}^2}; \\
B_4 &:= 128n_Y\gamma^2\sigma_X^2; \\
\epsilon_0 &:= \max\{B_1, B_2, B_3, B_4\}; \\
\epsilon_1 &:= \max\{B_1, B_2, B_3, B_4\}; \\
b_1 &:= 8B_{\ell}\tilde{L}_{\Phi}; \\
b_2 &:= 8\tilde{L}_{\Phi}\left[B_{\ell} + B_{\Phi}L\right]; \\
b_3 &:= 32\Omega(f_{\mu}^*)\max\{\tilde{L}_{\phi}B_{\ell}, L\tilde{L}_{\Phi}B_{\Phi}\}; \\
b_4 &:= 4\tilde{L}_{\Phi}B_{\Phi}; \\
\epsilon_2 &:= \max\{b_1, b_2, b_3, b_4\}.
\end{align}

Under the above conditions by Lemma \ref{lemma:cnc_nrms}, for any $\epsilon \in [0, B_4]$ we have that
\begin{align}\label{eq:cnc_nrms_p}
\mathbb{P}\left(\E_{nrm}^{c}\left(\frac{\epsilon}{2}\right)\right) \leq
\delta_{\C} + c_4\exp\left(\log\left(\C_{\F_{\mathcal{W}}}\left(\frac{\epsilon}{b_4}\right)\right)-N\left(\frac{\epsilon}{B_4}\right)^2\right).
\end{align}

By Lemma \ref{lemma:cnc_cvx_func}, for 
any $\epsilon \in [0, B_1]$, we have
\begin{align}\label{eq:cnc_cvx_func_p}
\mathbb{P}\left(\E_{cvx}^{c}\left(\frac{\epsilon}{4}\right)\right) \leq 
\delta_{\C} + 2\exp\left(\log\left(C_{\F_{\mathcal{W}}}\left(\frac{\epsilon}{b_1}\right)\right)-c_1N\left(\frac{\epsilon}{B_1}\right)^2\right),
\end{align}
for some positive constant, $c_1$.

Additionally by Lemma \ref{lemma:cnc_eql}, for any $\epsilon \in [0, B_2]$ we have that
\begin{align}\label{eq:cnc_eql_p}
\mathbb{P}\left(\E_{eql}^{c}\left(\frac{\epsilon}{4}\right)\right) \leq
\delta_{\C} + c_1\exp\left(\log\left(\C_{\F_{\mathcal{W}}}\left(\frac{\epsilon}{b_2}\right)\right)-N\left(\frac{\epsilon}{B_2}\right)^2\right),
\end{align}
for some positive constant, $c_2$. Furthermore, by Lemma \ref{lemma:cnc_plr}, for any $\epsilon \in [0, B_3]$ we have that
\begin{align}\label{eq:cnc_plr_p}
\mathbb{P}\left(\E_{plr}^{c}\left(\frac{\epsilon}{4\Omega(f_{\mu}^*)}\right)\right) \leq
\delta_{\C} + c_3\exp\left(\log\left(\C_{\F_{\mathcal{W}}}\left(\frac{\epsilon}{b_3}\right)\right) + \log\left(\C_{\F_{\theta}}\left(\frac{\epsilon}{b_3}\right)\right) -N\left(\frac{\epsilon}{B_3}\right)^2\right),
\end{align}
for some positive constant, $c_3$. 

For the inequalities \eqref{eq:cnc_cvx_func_p}, \eqref{eq:cnc_eql_p}, \eqref{eq:cnc_plr_p}, and \eqref{eq:cnc_nrms_p}
to all hold we choose $\epsilon \in [0, \epsilon_0]$ and 
we upper bound the covering numbers $C_{\F_{\mathcal{W}}}(\nu)$ as they are
strictly decreasing in $\nu$ by definition. Therefore, we have that
\begin{align}
\scalebox{0.9}{$
\max\bigg\{\log\left(C_{\F_{\mathcal{W}}}\left(\frac{\epsilon}{b_1}\right)\right),
\log\left(\C_{\F_{\mathcal{W}}}\left(\frac{\epsilon}{b_2}\right)\right),
\log\left(\C_{\F_{\mathcal{W}}}\left(\frac{\epsilon}{b_3}\right)\right),\log\left(\C_{\F_{\mathcal{W}}}\left(\frac{\epsilon}{b_4}\right)\right) \bigg\}
\leq \log\left(C_{\F_{\mathcal{W}}}\left(\frac{\epsilon}{\epsilon_2}\right)\right)$},
\end{align}
and
\be
\log\left(\C_{\F_{\theta}}\left(\frac{\epsilon}{b_3}\right)\right)
\leq \log\left(C_{\F_{\theta}}\left(\frac{\epsilon}{\epsilon_2}\right)\right).
\ee
Now we plug in inequalities \eqref{eq:cnc_cvx_func_p}, \eqref{eq:cnc_eql_p}, \eqref{eq:cnc_plr_p}, and \eqref{eq:cnc_nrms_p}
in the inequality \eqref{eq:e_good}. Denote
\be
c_5 := \max\{2, c_2, c_3, c_4\}.
\ee
Then
\begin{align}
\mathbb{P}(\E_{good}(\epsilon)) &\geq 1 -
c_5\exp(\log(\C_{\F_{\mathcal{W}}}(\epsilon/\epsilon_3))) \times
\left[\exp\left(-c_1N\left(\frac{\epsilon}{B_1}\right)^2\right)
+ \exp\left(-N\left(\frac{\epsilon}{B_2}\right)^2\right)\right. \\
&\left. + \exp\left(\log\left(C_{\F_{\theta}}\left(\frac{\epsilon}{\epsilon_3}\right)\right)-N\left(\frac{\epsilon}{B_3}\right)^2\right)
+ \exp\left(-N\left(\frac{\epsilon}{B_4}\right)^2\right)\right]-4\delta_{\C}.
\end{align}
Now we lower bound the right side by replacing $B_1, B_2, B_3, B_4$ with the upper bound $\epsilon_1$ yielding
\bd
\mathbb{P}(\E_{good}(\epsilon)) \geq 1 -
c_6\exp\left(\log\left(\C_{\F_{\mathcal{W}}}\left(\epsilon/\epsilon_2\right)\right)
+ \log\left(C_{\F_{\theta}}\left(\frac{\epsilon}{\epsilon_3}\right)\right) - c_7N\left(\frac{\epsilon}{\epsilon_1}\right)^2\right)
-4\delta_{\C},
\ed
for some positive constants, $c_6, c_7$.

From Lemma \ref{lemma:cn_bigclass} we have that for any $\nu > 0$ it holds that \begin{align}
\log(\C_{\F_{\mathcal{W}}}(\nu))
\leq R\log(\C_{\F_{\theta}}\left({L_{\phi}\nu}/{\gamma}\right)).
\end{align}
Then we obtain
\be\label{eq:final_p_wr}
\mathbb{P}(\E_{good}(\epsilon)) \geq 1 -
c_8\exp\left(R\log\left(C_{\F_{\theta}}\left(\frac{L_{\phi}\epsilon}{\gamma\epsilon_2}\right)\right) - c_9N\left(\frac{\epsilon}{\epsilon_1}\right)^2\right)
-4\delta_{\C},
\ee
for some positive constants $c_8, c_9$.

From inequality \eqref{eq:final_genb}, and \eqref{eq:final_p_wr} for any $\epsilon \in [0, \epsilon_0]$
we have that
\bd
\mathbb{P}\left(
\left|NC_{\mu}(\{W_j\}) - NC_{\mu_N}(\{W_j\})\right| \geq
\l \Omega(f_{\mu}^*)\left[\Omega_{\mu_N}^{\circ}\left(-\frac{1}{\l}\nabla_{\hat{Y}}\ell\left(g, \Phi_r(\{W_j\})\right)\right)-1\right]  \right.
\ed
\bd
\left. 
- \frac{\alpha}{2}\nmm{f_{\mu}^*-\Phi_r(\{W_j\})}_{\mu_N}^2
+ (1+\alpha)\epsilon + B_{eql}(\C) + B_{plr}(\C) + (1+\alpha)B_{nrm}(\C)
\right)
\ed
\be
\leq c_8\exp\left(R\log\left(C_{\F_{\theta}}\left(\frac{L_{\phi}\epsilon}{\gamma\epsilon_2}\right)\right) - c_9N\left(\frac{\epsilon}{\epsilon_1}\right)^2\right)
+ 4\delta_{\C}.
\ee

Next, we derive the operation conditions for $\epsilon$ in terms of $B_2$, $B_3,$ and $B_4$.
\begin{itemize}
\item $B_2 \geq B_3$: observe that %, to establish an upper bound on regularization:
\begin{align}\label{eq:item_1}
B_2 \geq B_3 \iff \gamma \geq \Omega(f_{\mu}^*)L_{\phi},
\end{align}
which establishes an upper bound on the regularization parameter. 

% Since, $f_{\mu}^* \in \F_{\Phi}$, thus,

% \begin{align}

% \end{align}

\item To establish a lower bound on regularization, we will require that $\min\{B_2, B_3\} \geq B_4$: We have that 
\begin{align}\label{eq:item_2}
\min\{B_2, B_3\} &\geq B_4 \iff\\
&\left(4\min\left\{1, \frac{ \Omega(f_{\mu}^*) L_{\phi}}{\gamma}\right\}\right) 4\gamma^2\sigma_X^2\sqrt{\left(1+\nmL{g}^2/\gamma^2\right) + {\frac{\nmL{g}^2\sigma_{Y|X}^2}{\gamma^2\sigma_X^2}}} \geq 16\gamma^2 \sigma_X^2.
\end{align}
It is sufficient to have the below inequality to hold:
\begin{align}
\min\left\{1, \frac{\Omega(f_{\mu}^*) L_{\phi}}{\gamma}\right\} \geq 
\frac{\gamma}{\sqrt{\gamma^2 + \nmL{g}^2}} \implies \min\{B_2, B_3\} \geq B_4.
\end{align}
\end{itemize}

% Therefore, 
% $\frac{\gamma}{\Omega(f_{\mu}^*)L_{\phi}}\left(\frac{1}{\sqrt{1 + \nmL{g}^2/\gamma^2}}\right) \leq \lambda \leq \frac{\gamma}{\Omega(f_{\mu}^*)L_{\phi}}$ 
% is sufficient condition for $B_2 \geq B_3 \geq B_4$.

Therefore, 
$\gamma \geq \Omega(f_{\mu}^*)L_{\phi}$ 
is sufficient condition for $B_2 \geq B_3 \geq B_4$.

Then we have that
\begin{align}
\epsilon_0 = \min\{B_1, B_4\} = 16n_Y\gamma^2\sigma_X^2\min\left\{1,
\frac{L}{4}\left[1 + \frac{\nmL{g}^2}{\gamma^2}\left(1 + \frac{\sigma_{Y|X}^2}{\sigma_X^2}\right)\right]\right\},
\end{align}
and
\begin{align}
\epsilon_1 = 16n_Y\gamma^2\sigma_X^2\max\left\{1,
\frac{L}{4}\left[1 + \frac{\nmL{g}^2}{\gamma^2}\left(1 + \frac{\sigma_{Y|X}^2}{\sigma_X^2}\right)\right]\right\}.
\end{align}

Now rescale the quantities $\epsilon \gets n_Y\frac{\epsilon}{(1+\alpha)}$,
$\epsilon_0 \gets n_Y\epsilon_0$, and $\epsilon_1 \gets n_Y\epsilon_1$. Then we have
\begin{align}
\epsilon_0 &= 16\gamma^2\sigma_X^2\min\left\{1,
\frac{L}{4}\left[1 + \frac{\nmL{g}^2}{\gamma^2}\left(1 + \frac{\sigma_{Y|X}^2}{\sigma_X^2}\right)\right]\right\};\\
\epsilon_1&= 16\gamma^2\sigma_X^2\max\left\{1,
\frac{L}{4}\left[1 + \frac{\nmL{g}^2}{\gamma^2}\left(1 + \frac{\sigma_{Y|X}^2}{\sigma_X^2}\right)\right]\right\}; \\
\epsilon_2 &= \max\{8B_{\ell}\tilde{L}_{\Phi}, 8[\tilde{L}_{\Phi}B_{\ell} + \tilde{L}_{\Phi}B_{\Phi}L], 32\Omega(f_{\mu}^*)\tilde{L}_{\Phi}\max\{\tilde{L}_{\phi}B_{\ell}/\tilde{L}_{\Phi}, LB_{\Phi}\}, 4\tilde{L}_{\Phi}B_{\Phi}\},
\end{align}
and
\bd
\mathbb{P}\left(
\frac{1}{n_Y}\left|\NC_{\mu}(\{W_j\}) - \NC_{\mu_N}(\{W_j\})\right| \geq
\frac{\l}{n_Y} \Omega(f_{\mu}^*)\left[\Omega_{\mu_N}^{\circ}\left(-\frac{1}{\l}\nabla_{\hat{Y}}\ell\left(g, \Phi_r(\{W_j\})\right)\right)-1\right]  \right.
\ed
\bd
\left. 
- \frac{\alpha}{2n_Y}\nmm{f_{\mu}^*-\Phi_r(\{W_j\})}_{\mu_N}^2
+  \frac{1}{n_Y}\left[B_{eql}(\C) + \l \Omega(f_{\mu}^*)B_{plr}(\C) + (1 + \alpha) B_{nrm}(\C)\right] + \epsilon
\right)
\ed
\be
\leq c_8\exp\left(R\log\left(C_{\F_{\theta}}\left(\frac{L_{\phi}\epsilon}{\gamma\epsilon_2}\right)\right) - c_9N\left(\frac{\epsilon}{(1+\alpha)\epsilon_1}\right)^2\right) + 4\delta_{\C}.
\ee

Now we bound the covering number
under Assumption \ref{ass:a3} and Lemma \ref{lemma:cn_bigclass} via
\be
\log\left(C_{\F_{\theta}}\left(\frac{L_{\phi}\epsilon}{\gamma\epsilon_2}\right)\right)
\leq {\sf dim}(\mathcal{W})\log(1 + 2\gamma\epsilon_2r_{\theta}/(L_{\phi}\epsilon)) \leq 
c_{11}{\sf dim}(\mathcal{W})\log(\gamma\epsilon_2r_{\theta}/(L_{\phi}\epsilon))
\ee
for some positive constant $c_{11}$.

Define
\be
B(\C) := B_{eql}(\C) + \l \Omega(f_{\mu}^*)B_{plr}(\C) + (1 + \alpha) B_{nrm}(\C).
\ee
Then we have that

\bd
\mathbb{P}\left(
\frac{1}{n_Y}\left|\NC_{\mu}(\{W_j\}) - \NC_{\mu_N}(\{W_j\})\right| \geq
\frac{\l}{n_Y} \Omega(f_{\mu}^*)\left[\Omega_{\mu_N}^{\circ}\left(-\frac{1}{\l}\nabla_{\hat{Y}}\ell\left(g, \Phi_r(\{W_j\})\right)\right)-1\right]  \right.
\ed
\bd
\left. 
- \frac{\alpha}{2n_Y}\nmm{f_{\mu}^*-\Phi_r(\{W_j\})}_{\mu_N}^2
+  \frac{1}{n_Y}B(\C) + \epsilon
\right)
\ed
\be\label{eq:p_mt}
\leq c_{12}\exp\left(c_{11}R{\sf dim}(\mathcal{W})\log(\gamma\epsilon_2r_{\theta}/(L_{\phi}\epsilon)) - c_{12}N\left(\frac{\epsilon}{(1+\alpha)\epsilon_1}\right)^2\right) + 4\delta_{\C}.
\ee

For some fixed $\delta \in (0, 1]$ choose
\be
\epsilon
= \tilde{\Theta}\left(
(1+\alpha)\epsilon_1\sqrt{\frac{R{\sf dim}(\mathcal{W})\log(\gamma \epsilon_2 r_{\theta}/L_{\phi})\log(N) + \log(1/\delta)}{N}}\right).
\ee

Then the right side term of inequality \eqref{eq:p_mt} will be

\be
\exp\left(R{\sf dim}(\mathcal{W})\log(\gamma\epsilon_2r_{\theta}/(L_{\phi}\epsilon)) - c_{12}N\left(\frac{\epsilon}{(1+\alpha)\epsilon_1}\right)^2\right)
= \tilde{\mathcal{O}}(\delta).
\ee
Rewriting the equation \eqref{eq:p_mt}, we have

\bd
\mathbb{P}\left(
\frac{1}{n_Y}\left|\NC_{\mu}(\{W_j\}) - \NC_{\mu_N}(\{W_j\})\right| \gtrsim
\frac{\l}{n_Y} \Omega(f_{\mu}^*)\left[\Omega_{\mu_N}^{\circ}\left(-\frac{1}{\l}\nabla_{\hat{Y}}\ell\left(g, \Phi_r(\{W_j\})\right)\right)-1\right]  \right.
\ed
\bd
\left. 
- \frac{\alpha}{2n_Y}\nmm{f_{\mu}^*-\Phi_r(\{W_j\})}_{\mu_N}^2
+  \frac{1}{n_Y}B(\C) \right.
\ed
\be\label{eq:p_mt_final}
\left. +(1+\alpha)\epsilon_1\sqrt{\frac{R{\sf dim}(\mathcal{W})\log(\gamma \epsilon_2 r_{\theta}/L_{\phi})\log(N) + \log(1/\delta)}{N}}
\right)
\lesssim \delta + \delta_{\C}.
\ee
\end{proof}

Theorem \ref{thm:gen_master} has established generalization error for a generic parallel positively
homogeneous network. Theorem \ref{thm:master} mentioned in the main text is a special case of Theorem \ref{thm:gen_master},
with the choice of convex set $\C = {\sf conv}(\X \times \R^{n_Y})$
by changing Assumption \ref{ass:a9} to Assumption \ref{ass:a8}.
Further, $B(\X \times \R^{n_Y})$ will evaluate to $0$, as $\P_{{\sf conv}(\X \times \R^{n_Y})}(\cdot)$ is just an identity operator.

\section{APPLICATIONS}

In this section, we apply our Theorem \ref{thm:gen_master} for various applications. We apply our general theorem to low-rank matrix sensing, structured matrix sensing,
two-layer linear neural network, two-layer ReLU neural network, and multi-head attention.

\subsection{Low-Rank Matrix Sensing}\label{sec:apdx_low_mat}

In this section, we state the corollary and its proof for matrix sensing, which is a direct consequence of Theorem \ref{thm:gen_master}.
Firstly, we need to choose a convex set, $\C$, such that the Assumption \ref{ass:a9} is satisfied. For matrix
sensing we choose, $\C = \{(X, \epsilon): \nmF{X} \leq g, \nmF{\epsilon} \leq g\}$ to verify Assumption \ref{ass:a9}. We need
to compute, $B(\C)$. This involves computing the expectation over the projection. Lemma \ref{lemma:proj_gauss}
is pivotal for estimating $B(\C)$ in all the applications that are going to be discussed here.

\begin{lemma}[Projection of Gaussian vector on balls]\label{lemma:proj_gauss}
Consider a $n$-dimensional Gaussian vector $\x \sim \N(0, (1/n)I_{n})$.
Let $M$ be a fixed matrix in $\R^{n \times n}$ and $\A$ be any set. Then
\begin{eqnarray}
\left|\IP{M}{\mathbb{E}\left[\left[\x\x^T - \P_{\B(g)}(\x)\P_{\B(g)}(\x)^T\right]\right]{\mathbf{1}}_{\A}(\x)}\right|
\leq \begin{cases}
ge^{-g^2/2}\nmm{M}_2 & \text{if }g \geq 1\\
\frac{1}{g}e^{-g^2/2}\nmm{M}_2 & \text{otherwise}
\end{cases}
\end{eqnarray}
where $\P_{\mathbb{B}(g)}(\cdot)$ is Euclidean projection onto the ball $\mathbb{B}(g) := \{\x: \nmm{\x}_2 \leq g\}$.
\end{lemma}
\begin{proof}
Define an event $\E := \{\x \in \B(g)\}$.  When $\E$ holds the function evaluates to zero,
\begin{equation}
    {\mathbb{E}\left[\left[\x\x^T - \P_{\B(g)}(\x)\P_{\B(g)}(\x)^T\right]{\mathbf{1}}_{\A}(\x)\right]}
=  {\mathbb{E}\left[\left(\x\x^T - \P_{\B(g)}(\x)\P_{\B(g)}(\x)^T\right) \mathbf{1}_{\E^{c}}(\x){\mathbf{1}}_{\A}(\x)\right]},
\end{equation}
so it suffices to consider the complement of the event $\E$.
Now we take the inner product with $M$ yielding
\begin{align*}
\left|\IP{M}{\mathbb{E}\left[\left[\x\x^T - \P_{\B(g)}(\x)\P_{\B(g)}(\x)^T\right]{\mathbf{1}}_{\A}(\x)\right]}\right|
&= \left|\IP{M}{\mathbb{E}\left[\left(\x\x^T - \P_{\B(g)}(\x)\P_{\B(g)}(\x)^T\right) \mathbf{1}_{\E^{c}}(\x){\mathbf{1}}_{\A}(\x)\right]}\right|\\
&\overset{(a)}{\leq} \nmm{M}_2\nmm{\mathbb{E}\left[\left(\x\x^T - \P_{\B(g)}(\x)\P_{\B(g)}(\x)^T\right) \mathbf{1}_{\E^{c}}(\x){\mathbf{1}}_{\A}(\x)\right]}\\
&\overset{(b)}{\leq} \nmm{M}_2\mathbb{E}\left[\nmm{\left(\x\x^T - \P_{\B(g)}(\x)\P_{\B(g)}(\x)^T\right) \mathbf{1}_{\E^{c}}(\x)}{\mathbf{1}}_{\A}(\x)\right]\\
&\overset{(c)}{=} \nmm{M}_2\mathbb{E}\left[\nmm{\x}_2^2 - g^2 | \mathbf{1}_{\E^{c}}(\x){\mathbf{1}}_{\A}(\x)\right]\\
&\overset{(d)}{=} \nmm{M}_2\mathbb{E}\left[\left|\nmm{\x}_2^2 - g^2\right| | \mathbf{1}_{\E^{c}}(\x){\mathbf{1}}_{\A}(\x)\right]\\
&\overset{(e)}{\leq}\nmm{M}_2\mathbb{E}\left[\left|\nmm{\x}_2^2 - g^2\right| | \mathbf{1}_{\E^{c}}(\x)\right]\\
&\overset{(f)}{=}\nmm{M}_2 \int_{\x \in \E^{c}}\left|\nmm{\x}_2^2 - g^2\right|\frac{1}{\sqrt{2\pi}}e^{-\frac{\nmm{\x}_2^2}{2}}d\x\\
&\overset{(g)}{=} \nmm{M}_2\left[ge^{-g^2/2} - \sqrt{\frac{\pi}{2}}(g^2-1){\sf erfc}(g/\sqrt{2})\right].
\end{align*}
The aforementioned computations involves (a) Cauchy-Schwartz inequality, (b) Jensen's inequality,
(c) the norm of $\P_{\B(g)}(\x)$ when $\x in \E^{c}$ is $g$, (d) conditioning on indicator functions,
(e) removing the conditioning increases the expectation over non-negative terms,
(f) we apply the density of Gaussian, (g) standard normal integral.
As a consequence of Theorem 1 from \cite{zhang-et-alel20} we bound the complement error function,
\be
\frac{e^{-z^2}}{\sqrt{\pi}z} \geq {\sf erfc}(z) \geq \frac{2}{\sqrt{\pi}}\frac{e^{-z^2}}{z + \sqrt{z^2 + 2}}.
\ee

Then we have that
\begin{eqnarray}
\left|\IP{M}{\mathbb{E}\left[\left[\x\x^T - \P_{\B(g)}(\x)\P_{\B(g)}(\x)^T\right]{\mathbf{1}}_{\A}(\x)\right]}\right|
\leq \begin{cases}
ge^{-g^2/2}\nmm{M}_2 & \text{if }g \geq 1\\
\frac{1}{g}e^{-g^2/2}\nmm{M}_2 & \text{otherwise}.
\end{cases}
\end{eqnarray}
\end{proof}

Now, we state the generalization bound for the low-rank matrix sensing followed by its proof.

\begin{corollary}[Low-Rank Matrix Sensing]
Consider the true model for $(X,y)$, where $X\in\R^{m\times n}$ is a random matrix with i.i.d. entries $X_{lk} \sim \N(0, \frac{1}{mn})$ and $y = \IP{M^*}{X} + \epsilon$, where $M^* \in \R^{m \times n}$ and $\epsilon \sim \N(0, \sigma^2)$ is independent from $X$. For all $i \in [N]$, let $(X_i,y_i)$ be i.i.d. samples from this true model. Consider the estimator $\hat{y} = \IP{UV^T}{X}$, where $U \in \R^{m \times R}$ and $V \in \R^{n \times R}$. 
%Suppose for all $i=1,\dots,N$, $X_i\in\R^{m\times n}$ is a random matrix with entries $(X_i)_{lk} \sim \N(0, \frac{1}{mn})$, $Y_i = \IP{M^*}{X_i} + \epsilon_i$, where $M^* \in \R^{m \times n}$, and $\epsilon_i \sim \N(0, \sigma^2)$. Consider the model $\hat{Y} = \IP{UV^T}{X}$, where $U \in \R^{m \times R}$, $V \in \R^{n \times R}$. 
Let $\delta \in (0,1]$ be fixed.  Define the non-convex problem 
\be
\NC^{\sf MS}_{\mu_N}(( U,V)) := %\min_{U, V} 
\frac{1}{2N}\sum_{i=1}^{N} \big( y_i - \IP{UV^T}{X_i} \big)^2 + \l \sum_{j=1}^{R}\nmm{\u_j}_2\nmm{\v_j}_2,
\ee
and define $\NC^{MS}_{\mu}((U,V))$ similarly with the sum over $i$ replaced by expectation taken over $(X,y)$. 

Let $(U,V)$ be a stationary point of $\NC^{\sf MS}_{\mu_N}(( U,V))$.  
Suppose there exists $C_{UV}, B_u, B_v > 0$ such that $\nmm{UV^T}_2 \leq C_{UV} \nmm{M^*}_*$, and for 
all $j \in [R]$, $\nmm{\u_j}_2 \leq B_u$, $\nmm{\v_j}_2 \leq B_v$.
Then, with probability at least $1 - \delta$, it holds that
\begin{align}
\bigg|\NC_{\mu}^{\sf MS}((U, V))) - \NC_{\mu_N}^{\sf MS}((U, V))\bigg|
&\lesssim 
\nonumber  \nmm{M^*}_*\left[\nmm{\frac{1}{N}\sum_{i=1}^{N}(y_i - \IP{UV^T}{X_i})X_i}_2-\l\right] \\
\nonumber  &+ C_{UV}^2\nmm{M^*}_*^2 \times
\nonumber \sqrt{\frac{R
\log\left(R(C_{UV}\hspace{-2pt}+\hspace{-2pt}B_{u}B_v)\right)(m+n)\log(N)\hspace{-2pt}+\hspace{-2pt}\log(1/\delta)}{N}}.
\end{align}
\end{corollary}

\begin{proof}
We set the following to obtain a generalization bound from Theorem \ref{thm:gen_master} for the case of matrix sensing. First,
\be
\ell(Y, \hat{Y}) = \frac{1}{2}\nmm{Y - \hat{Y}}_2^2 \implies (\alpha, L) = (0, 1);
\ee
\be
\phi(W) = \IP{\u\v^T}{X};
\ee
\be
\theta(W) = \nmm{\u}_2\nmm{\v}_2.
\ee
\textbf{Estimating $\Omega(f_{\mu}^*)$}:
Since, $M^*$ is the true matrix the regularizer at globally optimal solution can be upper bounded by Proposition 
\ref{prop:regular},
\be
\Omega(f_{\mu}^*) \leq \nmm{M^*}_*.
\ee
\textbf{Estimating $\Omega_{\mu_N}^{\circ}(\cdot)$}: Now we move on to compute the polar.
We have
\begin{align*}
\Omega_{\mu_N}^{\circ}\left(-\frac{1}{\l}\nabla_{\hat{Y}}\ell(g, \Phi_r(\{W_j\}))\right)
&= \Omega_{\mu_N}^{\circ}\left(\frac{1}{\l}(g - \Phi_r(\{W_j\}))\right)\\
&= \sup_{\nmm{\u} \leq 1; \nmm{\v} \leq 1}\frac{1}{N \l}\sum_{i=1}^{N}\IP{Y_i - \IP{UV^T}{X_i}}{\u^TX_i\v}\\
&= \sup_{\nmm{\u} \leq 1; \nmm{\v} \leq 1}\frac{1}{N \l}\IP{\v}{\sum_{i=1}^{N}(Y_i - \IP{UV^T}{X_i})^T\u^TX_i}\\
&= \sup_{\nmm{\u} \leq 1}\frac{1}{N\l}\nmm{\sum_{i=1}^{N}(Y_i - \IP{UV^T}{X_i})\u^TX_i}\\
&= \sup_{\nmm{\u} \leq 1}\frac{1}{N\l}\nmm{\sum_{i=1}^{N}(Y_i - \IP{UV^T}{X_i})X_i^T\u}\\
&= \frac{1}{\l}\nmm{\frac{1}{N}\sum_{i=1}^{N}(Y_i - \IP{UV^T}{X_i})X_i}_2.
\end{align*}

\textbf{Defining $\F_{\theta}$}: Next, we estimate the relevant constants. %, $\epsilon_0, \epsilon_1, \epsilon_2, B(\C), $, etc.
First we estimate the constants from Assumption \ref{ass:a3},
suppose that $\B := \{(\u, \v): \nmm{\u}_2 \leq 1, \nmm{\v}_2 \leq 1\}$
\be
\F_{\theta} := \{\u\v^T: \nmm{\u}_2\nmm{\v}_2 \leq 1\} \cap \B
\ee

\textbf{Estimating $L_{\phi}$}: The Lipschitz constant $L_{\phi}$ in the function $\F_{\theta}$ is
$L_{\phi} = \sup_{(\u, \v) \in \F_{\theta}}\nmL{\IP{\u\v^T}{.}} = 1$. 

\textbf{Estimating $r_{\theta}$}: We have that from A.M-G.M inequality,

\be
\nmm{\u}_2\nmm{\v}_2 \leq \frac{1}{{2}}[\nmm{\u}^2_2 + \nmm{\v}^2_2].
\ee
Now for any $(\u, \v) \in \F_{\theta}$ we have
that $0.5[\nmm{\u}_2 + \nmm{\v}_2] \leq 1$. Therefore, $\forall (\u, \v) \in \F_{\theta}$ we have
\be
\nmm{\u}_2\nmm{\v}_2 \leq 
\frac{1}{{2}}[\nmm{\u}^2_2 + \nmm{\v}^2_2]
\leq \sqrt{\frac{1}{{2}}[\nmm{\u}^2_2 + \nmm{\v}^2_2]}
\ee
Then we need that $\F_{\theta} \subseteq \mathbb{B}(r_{\theta})$, then  must be $r_{\theta} = \frac{1}{\sqrt{2}}$.

\textbf{Defining $\F_{\mathcal{W}}$}: From the corollary's assumptions we have that,
$\B_R := \left\{(\u, \v): \nmm{\u}_2 \leq B_u, \nmm{\v} \leq B_v\right\}$; our hypothesis class is defined as
\be
\F_{\mathcal{W}} := \left\{\{(\u_j, \v_j)\}: \nmL{\IP{UV^T}{\cdot}} = \nmm{UV^T}_2 \leq \gamma\right\} \cap \B_R.
\ee

As $\gamma \geq \Omega(f_{\mu}^*) L_{\phi}$, we choose $\gamma = C_{UV}\nmm{M^*}_*$ for $C_{UV} \geq 1$.
We have that
\be
\F_{\mathcal{W}} = \left\{\{(\u_j, \v_j)\}: \nmL{\IP{UV^T}{\cdot}} = \nmm{UV^T}_2 \leq C_{UV}\nmm{M^*}_*, \nmm{\u_j} \leq B_u, \nmm{\v_j} \leq B_v\right\}.
\ee

\textbf{Estimating $\epsilon_0$:}
From the data generating mechanism we have $\nmL{g} = \nmm{M^*}_2$, $\sigma_X = 1$, $\sigma_{Y|X} = \sigma$.
Then we have the following constants from Theorem \ref{thm:gen_master}:
\be
\epsilon_0 = 16\gamma^2\sigma_X^2\min\left\{1,
\frac{L}{4}\left[1 + \frac{\nmL{g}^2}{\gamma^2}\left(1 + \frac{\sigma_{Y|X}^2}{\sigma_X^2}\right)\right]\right\},
\ee
which evaluates to
\be
\epsilon_0 = 16C_{UV}^2\nmm{M^*}_*^2\min\left\{1,
\frac{1 + \sigma^2}{4C_{UV}^2}\right\} .
\ee
From the corollary assumption we have that $C_{UV} \leq 0.5\sqrt{1+\sigma^2}$, which implies that 
\be
\epsilon_0 = 4(1+\sigma^2)\nmm{M^*}_*^2.
\ee

\textbf{Estimating $\epsilon_1$:}
Similarly, we evaluate
\be
\epsilon_1 = 16\gamma^2\sigma_X^2\max\left\{1,
\frac{L}{4}\left[1 + \frac{\nmL{g}^2}{\gamma^2}\left(1 + \frac{\sigma_{Y|X}^2}{\sigma_X^2}\right)\right]\right\},
\ee
obtaining
\be
\epsilon_1 = 16C_{UV}^2\nmm{M^*}_*^2\max\left\{1,
\frac{1 + \sigma^2}{4C_{UV}^2}\right\} .
\ee
From corollary assumption we have that $C_{UV} \leq 0.5\sqrt{1+\sigma^2}$ which gives
\be
\epsilon_1 = 16C_{UV}^2\nmm{M^*}_*^2.
\ee

\textbf{Defining convex set $\C$:}
Consider a convex set $\C = \mathbb{B}(g) = \{X: \nmm{\vecc(X)}_2 \leq g\}$. 

First and foremost we need to estimate $\delta_{\C}$ for the following inequality to hold:
\be
P(\cap_{i=1}^{N}X_i \in \C) \geq 1 - \delta_{\C}.
\ee

The probability of $X \in \C = \mathbb{B}(g)$ is equivalent to saying the
probability of the event when $\nmm{\vecc(X)}_2 \leq g$. Since, $X_{ij} \sim \N(0, 1/(m \times n))$ 
as a consequence of Bernstein's Inequality \citep[Corollary 2.8.3]{vershynin_high-dimensional_2018} we have that
for any $t \geq 0$,
\be
P(\left| \nmm{\vecc(X)}_2 -1 \right|\leq t) \geq 1 - 2\exp\left(-cn_Xt^2\right)
\ee
for some constant $c \geq 0$. Now  we have
\be
P(\nmm{\vecc(X)}_2 \leq g) \begin{cases}
\geq 1 - 2\exp\left(-cn_X(g-1)^2\right) & \text{ if }g \geq 1\\
\leq 2\exp\left(-cn_X(g-1)^2\right) & \text{ otherwise. }
\end{cases}
\ee
We consider the case where $g \geq 1$, then we have that
\be
P(\cap_{i=1}^{N}X_i \in \C) = P(\cap_{i=1}^{N}\nmm{\vecc(X)}_2 \leq g) \geq 1 - \underbrace{2N\exp\left(-cn_X(g-1)^2\right)}_{=\delta_{\C}}.
\ee
We have that $\delta_{\C} = 2N\exp\left(-cn_X(g-1)^2\right)$.

Now we evaluate $B_{\ell}, B_{\Phi}, \tilde{L}_{\Phi}, \tilde{L}_{\phi}$.

\textbf{Estimating $B_{\Phi}$}:
Recall that $r_{\theta} = \frac{1}{\sqrt{2}}$. Then we have
\begin{eqnarray}
B_{\Phi} &=& \sup_{Z \in \C, \{(\u_j, \v_j)\} \in \F_{\mathcal{W}}}\nmm{\IP{UV^T}{Z}}\\
&=& \sup_{Z \in \C, \{(\u_j, \v_j)\} \in \F_{\mathcal{W}}}\nmm{\IP{\vecc(UV^T)}{\vecc(Z)}}\\
&=& g\sup_{\{(\u_j, \v_j)\} \in \F_{\mathcal{W}}}\nmm{\vecc(UV^T)}_2\\
&=& g\sup_{\{(\u_j, \v_j)\} \in \F_{\mathcal{W}}}\nmm{\sum_{j=1}^{R}\vecc(\u_j\v_j^T)}_2\\
&=& gR\sup_{\{(\u_j, \v_j)\} \in \F_{\mathcal{W}}}\nmm{\vecc(\u_j\v_j^T)}_2\\
&=& gR\sup_{\{(\u_j, \v_j)\} \in \F_{\mathcal{W}}}\nmm{\u_j\v_j^T}_F\\
&=& gR\sup_{\{(\u_j, \v_j)\} \in \F_{\mathcal{W}}}\nmm{\u_j}_2\nmm{\v_j}_2\\
&=& gB_uB_vR.
\end{eqnarray}

\textbf{Estimating $B_{\ell}$}:
Similarly, we have
\begin{eqnarray}
B_{\ell} &=& \sup_{Z \in \C, \{(\u_j, \v_j)\} \in \F_{\mathcal{W}}}\nmm{\IP{UV^T - M^*}{Z}}\\
&=& g\sup_{\{(\u_j, \v_j)\} \in \F_{\mathcal{W}}}\nmm{\vecc(UV^T-M^*)}_2\\
&\leq&g[\nmm{M^*}_F + B_uB_vR].
\end{eqnarray}

\textbf{Estimating $\tilde{L}_{\Phi}$}:
Now, we compute the Lipschitz constant with respect to $U, V$.
We have that
\begin{eqnarray}
\tilde{L}_{\Phi} &=& \sup_{Z \in \C, (U, V), (U', V') \in \F_{\mathcal{W}}}\frac{\nmm{\IP{UV^T-U'V'^T}{Z}}}{\max_{j}\sqrt{\nmm{\u_j-\u_j'}^2
+ \nmm{\v_j-\v_j'}^2}}\\
&=& g\sup_{(U, V), (U', V') \in \F_{\mathcal{W}}}\frac{\nmm{UV^T-U'V'^T}_F}{\max_{j}\sqrt{\nmm{\u_j-\u_j'}^2
+ \nmm{\v_j-\v_j'}^2}}\\
&=& gR\sup_{(U, V), (U', V') \in \F_{\mathcal{W}}}\frac{\nmm{\u_j\v_j^T - \u_j'\v_j'^T}_F}{\sqrt{\nmm{\u_j-\u_j'}^2
+ \nmm{\v_j-\v_j'}^2}}\\
&=& gR\sup_{(U, V), (U', V') \in \F_{\mathcal{W}}}\frac{\nmm{(\u_j-\u_j')\v_j^T - \u_j'(\v_j'-\v_j)^T}_F}{\sqrt{\nmm{\u_j-\u_j'}^2
+ \nmm{\v_j-\v_j'}^2}}\\
&\leq& gR\sup_{(U, V), (U', V') \in \F_{\mathcal{W}}}\frac{\nmm{(\u_j-\u_j')}_2\nmm{\v_j}_2 + \nmm{\u_j'}_2\nmm{(\v_j'-\v_j)}_2}{\sqrt{\nmm{\u_j-\u_j'}^2
+ \nmm{\v_j-\v_j'}^2}}\\
&\leq& gR\sup_{(U, V), (U', V') \in \F_{\mathcal{W}}}\sqrt{\nmm{\v_j}_2^2 + \nmm{\u_j'}_2^2}\\
&=& g\sqrt{B_u^2 + B_v^2}R.
\end{eqnarray}

\textbf{Estimating $\tilde{L}_{\phi}$}:
Similarly we get $\tilde{L}_{\phi} = g\sqrt{B_u^2 + B_v^2}$. 

\textbf{Estimating $\epsilon_2$}:
Recall that
\be
\epsilon_2 = \max\{8B_{\ell}\tilde{L}_{\Phi}, 8\tilde{L}_{\Phi}[B_{\ell} + B_{\Phi}L], 32\Omega(f_{\mu}^*)\tilde{L}_{\phi}\max\{B_{\ell}, LB_{\Phi}\}, 4\tilde{L}_{\Phi}B_{\Phi}\}.
\ee
From all the constants computed earlier, we have that
\be
\epsilon_2 = k_1g^2R^2(\nmm{M^*}_F + B_uB_v)^2
\ee
for some constant $k_1 \geq 0$.

Next we move on estimating $B(\C)$ we need to analyze three terms:
\\
\textbf{The First Term}: We define the first term via
\begin{align*}
T_1 := \sup_{\{W_j\} \in \F_{\mathcal{W}}}\left|\nmm{f_{\mu}^* \circ \P_{\C} - \Phi_r(\{W_j\}) \circ \P_{\C}}_{\mu}^2
-\nmm{f_{\mu}^* - \Phi_r(\{W_j\})}_{\mu}^2\right|.\end{align*}
For fixed $(U,V)$, we have
\begin{align*}
\bigg|&\nmm{f_{\mu}^* \circ \P_{\C} - \Phi_r(\{W_j\}) \circ \P_{\C}}_{\mu}^2
-\nmm{f_{\mu}^* - \Phi_r(\{W_j\})}_{\mu}^2\bigg| \\
&= \left|\mathbb{E}\left[\IP{M^* - UV^T}{\P_{\C}(X)}^2 - \IP{M^* - UV^T}{X}^2\right]\right|\\
&= \bigg|\mathbb{E}\bigg[\IP{\vecc(M^* - UV^T)\vecc(M^* - UV^T)^T}{\vecc(\P_{\C}(X))\vecc(\P_{\C}(X))^T} \\
&-\IP{\vecc(M^* - UV^T)\vecc(M^* - UV^T)^T}{vec(X)vec(X)^T}\bigg]\bigg|\\
&= \left|\IP{\vecc(M^* - UV^T)\vecc(M^* - UV^T)^T}{\mathbb{E}\left[\vecc(\P_{\C}(X))\vecc(\P_{\C}(X))^T-vec(X)vec(X)^T\right]}\right|
\end{align*}
From Lemma \ref{lemma:proj_gauss},  taking $g \geq 1$,
\be
\begin{split}
\Big|\nmm{f_{\mu}^* \circ \P_{\C} - \Phi_r(\{W_j\}) \circ \P_{\C}}_{\mu}^2
&-\nmm{f_{\mu}^* - \Phi_r(\{W_j\})}_{\mu}^2\Big|\\
& \leq 
ge^{-g^2/2}\nmm{\vecc(M^* - UV^T)\vecc(M^* - UV^T)^T}_2,
\end{split}
\ee
whereupon further simplifying, we obtain
\begin{eqnarray}
\left|\nmm{f_{\mu}^* \circ \P_{\C} - \Phi_r(\{W_j\}) \circ \P_{\C}}_{\mu}^2
-\nmm{f_{\mu}^* - \Phi_r(\{W_j\})}_{\mu}^2\right|
\leq ge^{-g^2/2}\nmm{M^* - UV^T}_F^2.
\end{eqnarray}

Now, applying triangular inequality and taking the supremum, we obtain
\begin{eqnarray}\label{eq:b1_ms}
T_1 \leq ge^{-g^2/2}(\nmm{M^*}_F + RB_uB_v)^2.
\end{eqnarray}
\textbf{The Second Term:} We define the second term via 
\be
\begin{split}
T_2 := \sup_{\{W_j\} \in \F_{\mathcal{W}}, W' \in \F_{\theta}}&\Big|\IP{\nabla_{\hat{Y}}\ell\left(g \circ \P_{\C}, \Phi_r(\{W_j\}) \circ \P_{\C}\right)}{\phi(W') \circ \P_{\C}}_{\mu}\\
&- \IP{\nabla_{\hat{Y}}\ell\left(g, \Phi_r(\{W_j\})\right)}{\phi(W')}_{\mu}\Big|.
\end{split}
\ee
We have
\begin{align*}
\bigg|\IP{\nabla_{\hat{Y}}&\ell\left(g \circ \P_{\C}, \Phi_r(\{W_j\}) \circ \P_{\C}\right)}{\phi(W') \circ \P_{\C}}_{\mu}
- \IP{\nabla_{\hat{Y}}\ell\left(g, \Phi_r(\{W_j\})\right)}{\phi(W')}_{\mu}\bigg| \\
&= 
\left|\mathbb{E}\left[\IP{UV^T-M^*}{\P_{\C}(X)}\IP{\u\v^T}{\P_{\C}(X)}-\IP{UV^T-M^*}{X}\IP{\u\v^T}{X}\right]\right| \\
&=  \left|\IP{\vecc(M^*-UV^T)\vecc(\u\v^T)^T}{\mathbb{E}\left[\vecc(X)\vecc(X)^T-\vecc(\P_{\C}(X))\vecc(\P_{\C}(X))^T\right]}\right|.
\end{align*}
As a consequence of Lemma \ref{lemma:proj_gauss} we have
\bd
\left|\IP{\nabla_{\hat{Y}}\ell\left(g \circ \P_{\C}, \Phi_r(\{W_j\}) \circ \P_{\C}\right)}{\phi(W') \circ \P_{\C}}_{\mu}
- \IP{\nabla_{\hat{Y}}\ell\left(g, \Phi_r(\{W_j\})\right)}{\phi(W')}_{\mu}\right|
\ed
\be
\leq ge^{-g^2/2}\nmm{\vecc(M^*-UV^T)\vecc(\u\v^T)^T}_2 = ge^{-g^2/2}\nmm{M^*-UV^T}_F\nmm{\u\v^T}_F.
\ee

Now we apply supremum over $(\u, \v) \in \F_{\theta}$ and then $(U,V)$
obtaining
\be\label{eq:b2_ms}
T_2 \leq ge^{-g^2/2}\left[\nmm{M^*}_F + RB_uB_v\right].
\ee

\textbf{The Third Term:} We define
\be
\begin{split}
T_3 := \sup_{\{W_j\} \in \F_{\mathcal{W}}}&\Big|\IP{\nabla_{\hat{Y}}\ell\left(g \circ \P_{\C}, \Phi_r(\{W_j\}) \circ \P_{\C}\right)}{\Phi_r(\{W_j\}) \circ \P_{\C}}_{\mu}\\
&- \IP{\nabla_{\hat{Y}}\ell\left(g, \Phi_r(\{W_j\})\right)}{\Phi_r(\{W_j\})}_{\mu}\Big|.
\end{split}
\ee
Similarly to the earlier item, we rewrite the above as
\be
  \left|\IP{\vecc(M^*-UV^T)\vecc(UV^T)^T}{\mathbb{E}\left[\vecc(X)\vecc(X)^T-\vecc(\P_{\C}(X))\vecc(\P_{\C}(X))^T\right]}\right|
\ee
As a consequence of Lemma \ref{lemma:proj_gauss} we have
\bd
\left|\IP{\nabla_{\hat{Y}}\ell\left(g \circ \P_{\C}, \Phi_r(\{W_j\}) \circ \P_{\C}\right)}{\Phi_r(\{W_j\}) \circ \P_{\C}}_{\mu}
- \IP{\nabla_{\hat{Y}}\ell\left(g, \Phi_r(\{W_j\})\right)}{\Phi_r(\{W_j\})}_{\mu}\right|
\ed
\be
\leq ge^{-g^2/2}\nmm{\vecc(M^*-UV^T)\vecc(UV^T)^T}_2 = ge^{-g^2/2}\nmm{M^*-UV^T}_F\nmm{UV^T}_F.
\ee
Finally, we apply supremum over $(U, V) \in \F_{\mathcal{W}}$, obtaining
\be\label{eq:b3_ms}
T_3 \leq ge^{-g^2/2}B_uB_vR\left[\nmm{M^*}_F + RB_uB_v\right].
\ee
Now combining equations \eqref{eq:b1_ms}, \eqref{eq:b2_ms}, \eqref{eq:b3_ms} we obtain that
\be
B(\C)
\leq ge^{-g^2/2}\left[\alpha(\nmm{M^*}_F + RB_uB_v)^2 + \nmm{M^*}_F + RB_uB_v + B_uB_vR\left[\nmm{M^*}_F + RB_uB_v\right]\right]
\ee
We further upper bound for simplicity as
\be%\label{eq:b_ms}
B(\C)
\leq 4ge^{-g^2/2}(\nmm{M^*}_F + RB_uB_v)^2.
\ee

From Theorem \ref{thm:gen_master} we have that 
\begin{align}
\frac{1}{n_Y}\left|\NC_{\mu}(\{W_j\}) - \NC_{\mu_N}(\{W_j\})\right| &\lesssim
\frac{\l}{n_Y} \Omega(f_{\mu}^*)\left[\Omega_{\mu_N}^{\circ}\left(-\frac{1}{\l}\nabla_{\hat{Y}}\ell\left(g, \Phi_r(\{W_j\})\right)\right)-1\right] \\
&\quad 
+  \frac{4}{n_Y}ge^{-g^2/2}\{ \nmm{M^*}_F + RB_uB_v\}^2 + 16C_{UV}^2\nmm{M^*}_*^2 \times \Big(
\\ &\quad \left.
\sqrt{\frac{R{{{\sf dim}(\mathcal{W})}}
log\left({C_{UV}\nmm{M^*}_* k_1g^2R^2(\nmm{M^*}_F + B_uB_v)^2\frac{1}{\sqrt{2}}}\right) \log(N)+ \log(1/\delta)}{N}}\right)
\end{align}
holds true w.p at least $1- \delta - 2N\exp\left(-cn_X(g-1)^2\right)$.

Now choose
\be
g = 1 + \mathcal{O}\left(\sqrt{\log(\sqrt{NR^{100}}) + \log(1/\delta)}\right).
\ee
Then we get that
\begin{align}
\frac{1}{n_Y}&\left|\NC_{\mu}(\{W_j\}) - \NC_{\mu_N}(\{W_j\})\right|\\
&\lesssim
\frac{\l}{n_Y} \Omega(f_{\mu}^*)\left[\Omega_{\mu_N}^{\circ}\left(-\frac{1}{\l}\nabla_{\hat{Y}}\ell\left(g, \Phi_r(\{W_j\})\right)\right)-1\right] \\
&\quad +  \frac{4}{n_Y}\{\nmm{M^*}_F + RB_uB_v\}^2\frac{\delta\sqrt{\log(NR^{100}) + \log(1/\delta)}}{NR^{100}} \\
&\quad + \scalebox{0.9}{$16C_{UV}^2\nmm{M^*}_*^2\sqrt{\frac{R(m+n)
\log\left({C_{UV}\nmm{M^*}_* k_1\left[\log(NR) + \log(1/\delta)\right]R^2(\nmm{M^*}_F + B_uB_v)^2\frac{1}{\sqrt{2}}}\right)\log(N)+ \log(1/\delta)}{N}}$}
\end{align}
holds true w.p at least $1- \delta$. Now ignoring \textit{loglog} terms and keeping the right most term
because of the dominance, we obtain 
\bd
\frac{1}{n_Y}\left|\NC_{\mu}(\{W_j\}) - \NC_{\mu_N}(\{W_j\})\right| \lesssim
\frac{\l}{n_Y} \Omega(f_{\mu}^*)\left[\Omega_{\mu_N}^{\circ}\left(-\frac{1}{\l}\nabla_{\hat{Y}}\ell\left(g, \Phi_r(\{W_j\})\right)\right)-1\right] 
\ed
\be
+ C_{UV}^2\nmm{M^*}_*^2\sqrt{\frac{R
log\left(R(C_{UV} + B_uB_v)\right)(m+n)\log(N)+ \log(1/\delta)}{N}}
\ee
holds true w.p at least $1- \delta$.
\end{proof}

\subsection{Structured Matrix Sensing}\label{sec:apdx_sms}
Next, we move on to a slightly more generalized matrix sensing problem through which we impose certain structure
in the factor $U$. Consider an atomic set $\U$ that represents the set of structured columns, and suppose that $U$ consists of columns that are affine combinations of the atoms in $\U$. We consider a gauge function $\gamma_{\U}(\cdot)$ which is defined via
\be
\gamma_{\U}(\u) := \inf\left\{t, t\geq 0{\text{ such that }}\u \in tconv(\U)\right\}
\ee
For instance, $\U$ can be the intersection of $L_2$ unit ball and $L_{1}$ unit ball, which 
induces $UV^T$ to be low-rank and $U$ to be sparse. 
Imposing such structures has been well studied for convex problems
by \cite{chandrasekaran-et-al-fcm12}. \cite{bach-arxiv13} analyzed
such structures for non-convex matrix factorization problems. However, their work was focused primarily on the optimization guarantees 
whereas our result below provides
generalization/recovery guarantees for structured matrix sensing problems.  We have the following corollary.

\begin{corollary}[Structured matrix sensing]\label{crl:smatsen}
Consider the true model for $(X,y)$, where $X\in\R^{m\times n}$ is a random matrix with i.i.d. entries $X_{lk} \sim \N(0, \frac{1}{mn})$ and $y = \IP{U^*{V^*}^T}{X} + \epsilon$, where $U^* \in \R^{m \times R^*}$, $V^* \in \R^{n \times R^*}$ and $\epsilon \sim \N(0, \sigma^2)$ is independent from $X$. For all $i \in [N]$, let $(X_i,y_i)$ be i.i.d. samples from this true model. Consider the estimator $\hat{y} = \IP{UV^T}{X}$, where $U \in \R^{m \times R}$ and $V \in \R^{n \times R}$. 
Let $\delta \in (0,1]$ be fixed.  Define the non-convex problem with the atomic set, $\U$
\be
\begin{split}
\NC^{\sf SMS}_{\mu_N}(( U,V)) &:= %\min_{U, V} 
\frac{1}{2N}\sum_{i=1}^{N} \big( y_i - \IP{UV^T}{X_i} \big)^2 
+ \l \sum_{j=1}^{R}\gamma_{\U}(\u_j)\nmm{\v_j}_2,
\end{split}
\ee
and define $\NC^{SMS}_{\mu}((U,V))$ similarly with the sum over $i$ replaced by expectation taken over $(X,y)$.  Here $\gamma_{\U}(\u) := \inf\left\{t; t \geq 0, \u \in tconv(\U)\right\}$ for some specified atomic set, $\U$.  Define
\begin{align}
K_1 := \sum_{j=1}^{r^*}\gamma_{\U}(\u^*_j)\nmm{\v^*_j}_2; 
K_2 := \sup_{\nmm{\u} \leq 1}\gamma_{\U}(\u).
\end{align}

Let $(U,V)$ be a stationary point of $\NC^{\sf SMS}_{\mu_N}(( U,V))$.  
Suppose there exists $C_{UV}, B_u, B_v > 0$ such that $\nmm{UV^T}_2 \leq C_{UV} K_1$, and for 
all $j \in [R]$, $\nmm{\u_j}_2 \leq B_u$, $\nmm{\v_j}_2 \leq B_v$.
Then, with probability at least $1 - \delta$, it holds that
%\bd
\begin{align}
&\bigg|\NC_{\mu}^{\sf SMS}((U, V))) - \NC_{\mu_N}^{\sf SMS}((U, V))\bigg|
\lesssim
K_1\left[K_2\nmm{\frac{1}{N}\sum_{i=1}^{N}(y_i - \IP{UV^T}{X_i})X_i}_2-\l\right] \\
\nonumber  & \hspace{30pt} + C_{UV}^2K_1^2 \sqrt{\frac{R
\log\left(R(C_{UV}\hspace{-2pt}+\hspace{-2pt}B_{u}B_v)\right)(m+n)\log(N)\hspace{-2pt}+\hspace{-2pt}\log(1/\delta)}{N}}.
\end{align}
\end{corollary}

\textbf{Remarks:} Similar to matrix sensing, the sample complexity required for consistency is only that $N \gtrsim R(m+n)$ up to logarithmic terms, assuming a global minimum is found.
The sample complexity is similar to low-rank matrix sensing (ignoring the scale and logarithmic dependency). To the best of our knowledge, this problem has not been studied from a statistical perspective, and our sample complexities match the corresponding convex
slightly structured matrix sensing of \cite{kakade-et-al-nips08}.
Unlike low-rank matrix sensing, the main technical challenge is to compute the polar/supremum term in the optimization error. In general, such a computation is  NP-hard when the atomic
set $\U$ has non-negative atoms \citep{hendrickx-arxiv10}.

\begin{proof}
The proof is similar to that of Corollary \ref{crl:matsen},
except for the computation of the polar. Therefore,  we only compute the
polar. 

\textbf{Estimating $\Omega(f_{\mu}^*)$}:
Since $M^*$ is the true matrix the globally optimal solution would be $M^*$; therefore, from Proposition \ref{prop:regular} we have
\be
\Omega(f_{\mu}^*) \leq \left(\sum_{j=1}^{r^*}\gamma_{\U}(\u^*_j)\nmm{\v^*_j}_2\right).
\ee

\textbf{Estimating $\Omega_{\mu_N}^{\circ}(\cdot)$}:
Now we move on to compute the polar.
\begin{align*}
\Omega_{\mu_N}^{\circ}\left(-\frac{1}{\l}\nabla_{\hat{Y}}\ell(g, \Phi_r(\{W_j\}))\right)
&= \Omega_{\mu_N}^{\circ}\left(\frac{1}{\l}(g - \Phi_r(\{W_j\}))\right)
\Omega_{\mu_N}^{\circ}\left(-\frac{1}{\l}\nabla_{\hat{Y}}\ell(g, \Phi_r(\{W_j\}))\right) \\
&= \Omega_{\mu_N}^{\circ}\left(\frac{1}{\l}(g - \Phi_r(\{W_j\}))\right)\\
&= \sup_{\gamma_{\U}(\u) \leq 1; \nmm{\v} \leq 1}\frac{1}{N \l}\sum_{i=1}^{N}\IP{Y_i - \IP{UV^T}{X_i}}{\u^TX_i\v}\\
&= \sup_{\gamma_{\U}(\u) \leq 1; \nmm{\v} \leq 1}\frac{1}{N \l}\IP{\v}{\sum_{i=1}^{N}(Y_i - \IP{UV^T}{X_i})^T\u^TX_i}\\
&= \sup_{\gamma_{\U}(\u) \leq 1}\frac{1}{N\l}\nmm{\sum_{i=1}^{N}(Y_i - \IP{UV^T}{X_i})\u^TX_i}\\
&= \sup_{\gamma_{\U}(\u) \leq 1}\frac{1}{N\l}\nmm{\sum_{i=1}^{N}(Y_i - \IP{UV^T}{X_i})X_i^T\u}.
\end{align*}
This yields
\be
\Omega_{\mu_N}^{\circ}\left(-\frac{1}{\l}\nabla_{\hat{Y}}\ell(g, \Phi_r(\{W_j\}))\right)
\leq \left[\sup_{\nmm{\u} \leq 1}\gamma_{\U}(\u)\right]
\frac{1}{N\l}\nmm{\sum_{i=1}^{N}(Y_i - \IP{UV^T}{X_i})X_i^T}_2.
\ee

The rest of the proof is the same as that of low-rank matrix sensing (see
section \ref{sec:apdx_low_mat}).
\end{proof}

\subsection{Two-Layer Linear NN}\label{sec:apdx_2lnn}

Next, we consider the closely related problem of 2-Layer Linear Neural Networks, which is essentially a multi-dimensional
matrix sensing problem; this is also referred to as non-convex linear regression. 
In practice, this approach has seemed to have better linear convergence \citep{arora-et-al-arxiv18}
and generalization capabilities \citep{allen-zhu-colt20} than vanilla linear regression. 
Corollary \ref{crl:2lnn} provides generalization error 
upper bounds.

\begin{corollary}[2-Layer Linear Neural Network] \label{crl:2lnn}  Consider the true model for $(\x,\y)$, where $\x \sim \N(0, (1/n)I_{n}) \in \R^{n}$, $\y = U^*{V^*}^T\x + \epsilon$, where $U^* \in \R^{m \times R^*}$,
$V^* \in \R^{n \times R^*}$, and $\epsilon \sim \N(0, (\sigma^2/m)I_{m}) \in \R^{m}$ independent from $\x$. For all $i\in [N]$, let $(\x_i,\y_i)$ be i.i.d. samples from this true model. Consider the estimator $\hat{\y} = UV^T\x$, where $U \in \R^{m \times R}, V \in \R^{n \times R}$. Let $\delta \in (0,1]$ be fixed.  Define the non-convex problem
\be
\begin{split}
\NC_{\mu_N}^{\sf {2LNN}}( (U,V)) := \frac{1}{2N}\sum_{i=1}^{N}\nmeusq{\y_i - U[V^T\x_i]_+}
+ \frac{\l}{2}\left(\nmF{U}^2 + \nmF{V}^2\right),
\end{split}
\ee
and define $\NC^{\sf {2LNN}}_{\mu}((U,V))$ similarly with the sum over $i$ replaced by expectation taken over $(\x, \y)$. 

Let $(U,V)$ be a stationary point of $\NC^{\sf {2LNN}}_{\mu_N}(( U,V))$. 
Suppose there exists $C_{UV}, B_u, B_v > 0$ such that $\nmm{UV^T}_2 \leq C_{UV}\left[\nmF{U^*}^2 + \nmF{V^*}^2\right]$, and for 
all $j \in [R]$, $\nmm{\u_j}_2 \leq B_u$, $\nmm{\v_j}_2 \leq B_v$.
% Suppose that $(U,V)$ is any stationary point satisfying $\nmm{UV^T}_2 \leq \frac{\sqrt{\sigma^2 + 1}}{2}\left[\nmF{U^*}^2 + \nmF{V^*}^2\right]$, and that for any 
% $j \in [R]$ we have $\nmm{\u_j}_2 \leq B_u$ and $\nmm{\v_j}_2 \leq B_v$. 
Then, with probability at least $1 - \delta$, it holds that
\begin{align}
&\frac{1}{m}\left|\NC_{\mu}^{\sf {2LNN}}((U, V)) - \NC_{\mu_N}^{\sf {2LNN}}((U, V))\right| \lesssim
\nonumber \frac{1}{2m}\left[\nmF{U^*}^2+\nmF{V^*}^2\right]\left[\frac{1}{N}\sum_{i=1}^{N}\nmm{\y_i - \hat{\y}_i}_2\nmm{\x_i}_2 - \l\right] \\
\nonumber &\hspace{10pt}C_{UV}^2\left[\nmF{U^*}^2 + \nmF{V^*}^2\right]^2\sqrt{\frac{R
\log\left(R\left(C_{UV} + B_{u}^2 + B_{v}^2\right)\right)(m+n)\log(N)+ \log(1/\delta)}{N}}.
\end{align}
\end{corollary}

Similar to matrix sensing, we require that $N \gtrsim R(m+n)$, with $\frac{R(m+n)}{N} \to 0$ for consistency at a global minimum.  This matches classical results for (convex) linear regression.

\begin{proof}

To obtain a generalization bound from Theorem \ref{thm:gen_master} for this setting, we set the following problem
parameters:
\be
\ell(Y, \hat{Y}) = \frac{1}{2}\nmm{Y - \hat{Y}} \implies (\alpha, L) = (0, 1);
\ee
\be
\phi(W) = \IP{\v}{\x}\u;
\ee
\be
\theta(W) = \frac{1}{2}\left[\nmm{\u}_2^2 + \nmm{\v}_2^2\right].
\ee

\textbf{Estimating $\Omega(f_{\mu}^*)$}:
From Proposition \ref{prop:regular} we have that
\be
\Omega(f_{\mu}^*)
\leq \frac{\nmF{U^*}^2 + \nmF{V^*}^2}{2}
\ee
\textbf{Choosing $\F_{\theta}$}:
\be
\F_{\theta} 
:= \{(\u, \v): \nmm{\u}^2 + \nmm{\v}^2 \leq 2, \nmm{\u}_2 \leq 1, \nmm{\v}_2 \leq 1\}.
\ee
\textbf{Estimating $L_{\phi}$}:
\be
L_{\phi} = \sup_{(\u, \v) \in \F_{\theta}}\nmL{\u\v^T(\cdot)} = \sup_{(\u, \v) \in \F_{\theta}}\nmm{\u\v^T}_2
= 1.
\ee

\textbf{Estimating $r_{\theta}$}:
For any $(\u, \v) \in \F_{\theta}$, we have that,
\be
\frac{\nmm{\u}^2 + \nmm{\v}^2}{2} \leq \sqrt{\frac{\nmm{\u}^2 + \nmm{\v}^2}{2}}
\implies \F_{\theta} \subseteq \mathbb{B}(1/\sqrt{2}).
\ee
Then we have $r_{\theta} = 1/\sqrt{2}$.

\textbf{Choosing $\F_{\mathcal{W}}$}:
\be
\F_{\mathcal{W}} 
:= \{(U, V): \nmm{UV^T}_2 \leq \gamma, \nmm{\u_j} \leq B_u,
\nmm{\v_j} \leq B_v\}.
\ee
As $\gamma \geq \Omega(f_{\mu}^*) L_{\phi} = \frac{\nmF{U^*}^2 + \nmF{V^*}^2}{2}$,  we may take $\gamma = C_{UV}\left[ \frac{\nmF{U^*}^2 + \nmF{V^*}^2}{2}\right]$ for some $C_{UV}$.
We have that
\be
\F_{\mathcal{W}} = \left\{\{(\u_j, \v_j)\}: \nmL{\IP{UV^T}{\cdot}} = \nmm{UV^T}_2 \leq C_{UV}\left[ \frac{\nmF{U^*}^2 + \nmF{V^*}^2}{2}\right], \nmm{\u_j} \leq B_u, \nmm{\v_j} \leq B_v\right\}.
\ee

\textbf{Estimating $\epsilon_0$}:
From the data generating mechanism we have $\nmL{g} = \nmm{M^*}_2$, $\sigma_X = 1$, $\sigma_{Y|X} = \sigma$,
which yields the following constants from Theorem \ref{thm:gen_master}:
\be
\epsilon_0 = 16\gamma^2\sigma_X^2\min\left\{1,
\frac{L}{4}\left[1 + \frac{\nmL{g}^2}{\gamma^2}\left(1 + \frac{\sigma_{Y|X}^2}{\sigma_X^2}\right)\right]\right\},
\ee
which evaluates to when $C_{UV} \leq 0.5\sqrt{(1+\sigma^2)}$
\be
\epsilon_0 = 8C_{UV}^2\left[ \nmF{U^*}^2 + \nmF{V^*}^2\right]\min\left\{1,
\frac{1 + \sigma^2}{4C_{UV}^2}\right\} = 2\left[ \nmF{U^*}^2 + \nmF{V^*}^2\right].
\ee

\textbf{Estimating $\epsilon_1$}:
Similarly, we evaluate
\be
\epsilon_1 = 16\gamma^2\sigma_X^2\max\left\{1,
\frac{L}{4}\left[1 + \frac{\nmL{g}^2}{\gamma^2}\left(1 + \frac{\sigma_{Y|X}^2}{\sigma_X^2}\right)\right]\right\},
\ee
obtaining
\be
\epsilon_1 = 8C_{UV}^2\left[ \nmF{U^*}^2 + \nmF{V^*}^2\right].
\ee

\textbf{Choosing the convex set $\C$}:
Consider a convex set $\C = \mathbb{B}(g) = \{X: \nmm{\vecc(X)}_2 \leq g\}$.

First and foremost we need to estimate $\delta_{\C}$ for the following inequality to hold:
\be
P(\cap_{i=1}^{N}X_i \in \C) \geq 1 - \delta_{\C}.
\ee
The probability of $\x \in \C = \mathbb{B}(g)$ is equivalent to saying the
probability of the event when $\nmm{\x}_2 \leq g$. Since, $x_{i} \sim \N(0, 1/n)$ 
as a consequence of Bernstein's Inequality \citep[Corollary 2.8.3]{vershynin_high-dimensional_2018} we have that
for any $t \geq 0$,
\be
P(\left|\nmm{\x}_2 -1 \right|\leq t) \geq 1 - 2\exp\left(-cn_Xt^2\right)
\ee
for some constant $c \geq 0$. Now we have
\be
P(\nmm{\x}_2 \leq g) \begin{cases}
\geq 1 - 2\exp\left(-cn_X(g-1)^2\right) & \text{ if }g \geq 1\\
\leq 2\exp\left(-cn_X(g-1)^2\right) & \text{ otherwise }
\end{cases}
\ee
We consider the case where $g \geq 1$, then we have that
\be
P(\cap_{i=1}^{N}X_i \in \C) = P(\cap_{i=1}^{N}\nmm{\x}_2 \leq g) \geq 1 - \underbrace{2N\exp\left(-cn_X(g-1)^2\right)}_{=\delta_{\C}}.
\ee
We have that $\delta_{\C} = 2N\exp\left(-cn(g-1)^2\right)$.

Now we evaluate $B_{\ell}, B_{\Phi}, \tilde{L}_{\Phi}, \tilde{L}_{\phi}$.

\textbf{Estimating $B_{\Phi}$:} We have
\begin{eqnarray}
B_{\Phi} &=& \sup_{\z \in \C, \{(\u_j, \v_j)\} \in \F_{\mathcal{W}}}\nmm{UV^T\z}\\
&=& g\sup_{\{(\u_j, \v_j)\} \in \F_{\mathcal{W}}}\nmm{UV^T}_2\\
&=& g\gamma
\end{eqnarray}

\textbf{Estimating $B_{\ell}$:}
Similarly, we have
\begin{eqnarray}
B_{\ell} &=& \sup_{\z \in \C, \{(\u_j, \v_j)\} \in \F_{\mathcal{W}}}\nmm{(UV^T - U^*{V^*}^T)\z}\\
&=& g\sup_{\{(\u_j, \v_j)\} \in \F_{\mathcal{W}}}\nmm{UV^T-U^*{V^*}^T}_2\\
&=& g(\nmm{U^*{V^*}^T}_2+\gamma).
\end{eqnarray}

\textbf{Estimating $\tilde{L}_{\Phi}$:}
Now, we compute the Lipschitz constant with respect to $U, V$. We have
\begin{eqnarray}
\tilde{L}_{\Phi} &=& \sup_{\z \in \C, (U, V), (U', V') \in \F_{\mathcal{W}}}\frac{\nmm{(UV^T-U'V'^T)\z}}{\max_{j}\sqrt{\nmm{\u_j-\u_j'}^2
+ \nmm{\v_j-\v_j'}^2}}\\
&=& g\sup_{(U, V), (U', V') \in \F_{\mathcal{W}}}\frac{\nmm{UV^T-U'V'^T}_2}{\max_{j}\sqrt{\nmm{\u_j-\u_j'}^2
+ \nmm{\v_j-\v_j'}^2}}\\
&\leq& g\sup_{(U, V), (U', V') \in \F_{\mathcal{W}}}\frac{\nmm{UV^T-U'V'^T}_F}{\max_{j}\sqrt{\nmm{\u_j-\u_j'}^2
+ \nmm{\v_j-\v_j'}^2}}\\
&=& g\sup_{(U, V), (U', V') \in \F_{\mathcal{W}}}\frac{\nmm{\sum_{j=1}^{R}\u_j\v_j^T - \u_j'\v_j'^T}_F}{\max_{j}\sqrt{\nmm{\u_j-\u_j'}^2
+ \nmm{\v_j-\v_j'}^2}}\\
&=& gR\sup_{(U, V), (U', V') \in \F_{\mathcal{W}}}\frac{\nmm{\u_j\v_j^T - \u_j'\v_j'^T}_F}{\sqrt{\nmm{\u_j-\u_j'}^2
+ \nmm{\v_j-\v_j'}^2}}\\
&=& gR\sup_{(U, V), (U', V') \in \F_{\mathcal{W}}}\frac{\nmm{(\u_j-\u_j')\v_j^T - \u_j'(\v_j'-\v_j)^T}_2}{\sqrt{\nmm{\u_j-\u_j'}^2
+ \nmm{\v_j-\v_j'}^2}}\\
&\leq& gR\sup_{(U, V), (U', V') \in \F_{\mathcal{W}}}\frac{\nmm{(\u_j-\u_j')}_2\nmm{\v_j}_2 + \nmm{\u_j'}_2\nmm{(\v_j'-\v_j)}_2}{\sqrt{\nmm{\u_j-\u_j'}^2
+ \nmm{\v_j-\v_j'}^2}}\\
&=& gR\sup_{(U, V), (U', V') \in \F_{\mathcal{W}}}\sqrt{\nmm{\v_j}_2^2 + \nmm{\u_j'}_2^2}\\
&=& g\sqrt{B_u^2 + B_v^2}R.
\end{eqnarray}

\textbf{Estimating $\tilde{L}_{\phi}$:}
Similarly we get $\tilde{L}_{\phi} = g\sqrt{B_u^2 + B_v^2}$.

\textbf{Estimating $\epsilon_2$:}
Recall that
\be
\epsilon_2 = \max\{8B_{\ell}\tilde{L}_{\Phi}, 8\tilde{L}_{\Phi}[B_{\ell} + B_{\Phi}L], 32\Omega(f_{\mu}^*)\tilde{L}_{\phi}\max\{B_{\ell}, LB_{\Phi}\}, 4\tilde{L}_{\Phi}B_{\Phi}\}.
\ee
From all the constants computed earlier, we have that
\be
\epsilon_2 = k_1g^2R^2C_{UV}^2(\nmF{U^*}^2 + \nmF{V^*}^2)\sqrt{B_u^2 + B_u^2}
\ee
for some constant $k_1 \geq 0$.

Next, we move on to estimating $B(\C)$. We need to analyze three terms:
\\
\textbf{The First Term:} Define
\begin{align*}
T_1 := \sup_{\{W_j\} \in \F_{\mathcal{W}}}\left|\nmm{f_{\mu}^* \circ \P_{\C} - \Phi_r(\{W_j\}) \circ \P_{\C}}_{\mu}^2
-\nmm{f_{\mu}^* - \Phi_r(\{W_j\})}_{\mu}^2\right|
\end{align*}
We have
\be
\left|\nmm{f_{\mu}^* \circ \P_{\C} - \Phi_r(\{W_j\}) \circ \P_{\C}}_{\mu}^2
-\nmm{f_{\mu}^* - \Phi_r(\{W_j\})}_{\mu}^2\right| = 
\ee
\be
= \left|\mathbb{E}\left[\nmm{(U^*{V^*}^T - UV^T)\P_{\C}(\x)}^2 - \nmm{(U^*{V^*}^T - UV^T)\x}^2\right]\right|
\ee
\be
\begin{split}
= \Big|\mathbb{E}&\Big[\IP{(U^*{V^*}^T - UV^T)(U^*{V^*}^T - UV^T)^T}{(\P_{\C}(\x))(\P_{\C}(\x))^T} \\
&- \IP{(U^*{V^*}^T - UV^T)(U^*{V^*}^T - UV^T)^T}{\x\x^T}\Big]\Big|
\end{split}
\ee
\be
= \left|\IP{(U^*{V^*}^T - UV^T)(U^*{V^*}^T - UV^T)^T}{\mathbb{E}\left[(\P_{\C}(\x))(\P_{\C}(\x))^T-\x\x^T\right]}\right|
\ee
From Lemma \ref{lemma:proj_gauss} we obtain that (taking $g \geq 1$)
\bd
\left|\nmm{f_{\mu}^* \circ \P_{\C} - \Phi_r(\{W_j\}) \circ \P_{\C}}_{\mu}^2
-\nmm{f_{\mu}^* - \Phi_r(\{W_j\})}_{\mu}^2\right|
\ed
\be
\leq ge^{-g^2/2}\nmm{(U^*{V^*}^T - UV^T)(U^*{V^*}^T - UV^T)^T}_2,
\ee
on further simplifying, we get
\begin{eqnarray}
\left|\nmm{f_{\mu}^* \circ \P_{\C} - \Phi_r(\{W_j\}) \circ \P_{\C}}_{\mu}^2
-\nmm{f_{\mu}^* - \Phi_r(\{W_j\})}_{\mu}^2\right|
\leq ge^{-g^2/2}\nmm{U^*{V^*}^T - UV^T}_2^2.
\end{eqnarray}
Now applying triangular inequality and taking the supremum, we obtain
\begin{eqnarray}\label{eq:b1_mf}
T_1 
\leq ge^{-g^2/2}(\nmm{U^*{V^*}^T}_2 + \gamma)^2,
\end{eqnarray}

\textbf{The Second Term}: Define
\be
\begin{split}
T_2 := \sup_{\{W_j\} \in \F_{\mathcal{W}}, W' \in \F_{\theta}}&\Big|\IP{\nabla_{\hat{Y}}\ell\left(g \circ \P_{\C}, \Phi_r(\{W_j\}) \circ \P_{\C}\right)}{\phi(W') \circ \P_{\C}}_{\mu}\\
&- \IP{\nabla_{\hat{Y}}\ell\left(g, \Phi_r(\{W_j\})\right)}{\phi(W')}_{\mu}\Big|.
\end{split}
\ee
We have
\bd
\left|\IP{\nabla_{\hat{Y}}\ell\left(g \circ \P_{\C}, \Phi_r(\{W_j\}) \circ \P_{\C}\right)}{\phi(W') \circ \P_{\C}}_{\mu}
- \IP{\nabla_{\hat{Y}}\ell\left(g, \Phi_r(\{W_j\})\right)}{\phi(W')}_{\mu}\right|
= 
\ed
\be
\left|\mathbb{E}\left[\IP{(UV^T-U^*{V^*}^T)\P_{\C}(\x)}{\u\v^T\P_{\C}(\x)}-\IP{(UV^T-U^*{V^*}^T)\x}{\u\v^T\x}\right]\right|
\ee
is the same as
\be
=  \left|\IP{(U^*{V^*}^T-UV^T)(\u\v^T)^T}{\mathbb{E}\left[\x\x^T-(\P_{\C}(\x))(\P_{\C}(\x))^T\right]}\right|
\ee
As a consequence of Lemma \ref{lemma:proj_gauss} we have
\bd
\left|\IP{\nabla_{\hat{Y}}\ell\left(g \circ \P_{\C}, \Phi_r(\{W_j\}) \circ \P_{\C}\right)}{\phi(W') \circ \P_{\C}}_{\mu}
- \IP{\nabla_{\hat{Y}}\ell\left(g, \Phi_r(\{W_j\})\right)}{\phi(W')}_{\mu}\right|
\ed
\be
\leq ge^{-g^2/2}\nmm{(U^*{V^*}^T-UV^T)(\u\v^T)^T}_2 = ge^{-g^2/2}\nmm{U^*{V^*}^T-UV^T}_F\nmm{\u\v^T}_F.
\ee

Now we apply supremum over $(\u, \v) \in \F_{\theta}$, and then over $(U,V) \in \F_{\mathcal{W}}$, yielding
\be\label{eq:b2_mf}
T_2 \leq ge^{-g^2/2}\left[\nmm{U^*{V^*}^T}_2 + \gamma\right].
\ee
\textbf{The Third Term:} Define 
\be
\begin{split}
T_3 := \sup_{\{W_j\} \in \F_{\mathcal{W}}}&\Big|\IP{\nabla_{\hat{Y}}\ell\left(g \circ \P_{\C}, \Phi_r(\{W_j\}) \circ \P_{\C}\right)}{\Phi_r(\{W_j\}) \circ \P_{\C}}_{\mu}\\
&- \IP{\nabla_{\hat{Y}}\ell\left(g, \Phi_r(\{W_j\})\right)}{\Phi_r(\{W_j\})}_{\mu}\Big|.
\end{split}
\ee

Similarly to the earlier item, we rewrite the above as
\be
=  \left|\IP{(U^*{V^*}^T-UV^T)(UV^T)^T}{\mathbb{E}\left[\x\x^T-(\P_{\C}(\x))(\P_{\C}(\x))^T\right]}\right|
\ee
As a consequence of Lemma \ref{lemma:proj_gauss} we have
\bd
\left|\IP{\nabla_{\hat{Y}}\ell\left(g \circ \P_{\C}, \Phi_r(\{W_j\}) \circ \P_{\C}\right)}{\Phi_r(\{W_j\}) \circ \P_{\C}}_{\mu}
- \IP{\nabla_{\hat{Y}}\ell\left(g, \Phi_r(\{W_j\})\right)}{\Phi_r(\{W_j\})}_{\mu}\right|
\ed
\be
\leq ge^{-g^2/2}\nmm{(U^*{V^*}^T-UV^T)(UV^T)^T}_2 \leq ge^{-g^2/2}\nmm{U^*{V^*}^T-UV^T}_2\nmm{UV^T}_2.
\ee

Finally, we apply supremum over $(U, V) \in \F_{\mathcal{W}}$, obtaining
\be\label{eq:b3_mf}
T_3 \leq ge^{-g^2/2}\gamma\left[\nmm{U^*{V^*}^T}_2 + \gamma\right].
\ee

Now combining equations \eqref{eq:b1_mf}, \eqref{eq:b2_mf}, \eqref{eq:b3_mf} we obtain that
\be
B(\C)
\leq ge^{-g^2/2}\left[\alpha(\nmm{U^*{V^*}^T}_2 + \gamma)^2 + \nmm{U^*{V^*}^T}_2 + \gamma + \gamma\left[\nmm{U^*{V^*}^T}_2 + \gamma\right]\right].
\ee
We further upper bound for simplicity as
\be %\label{eq:b_ms}
B(\C)
\leq ge^{-g^2/2}(1+\gamma)(\nmm{U^*{V^*}^T}_2 + \gamma).
\ee

From Theorem \ref{thm:gen_master} we have that 
\bd
\frac{1}{n_Y}\left|\NC_{\mu}(\{W_j\}) - \NC_{\mu_N}(\{W_j\})\right| \lesssim
\frac{\l}{n_Y} \Omega(f_{\mu}^*)\left[\Omega_{\mu_N}^{\circ}\left(-\frac{1}{\l}\nabla_{\hat{Y}}\ell\left(g, \Phi_r(\{W_j\})\right)\right)-1\right] 
\ed
\bd
+  \frac{4}{n_Y}ge^{-g^2/2}(1+\gamma)(\nmm{U^*{V^*}^T}_2 + \gamma) + 8C_{UV}^2\left[\nmF{U^*}^2 + \nmF{V^*}^2\right]^2 \times \Big(
\ed
\be
\left.
\sqrt{\frac{R(m+n)
\log\left(k_1g^2R^2C_{UV}^2(\nmF{U^*}^2 + \nmF{V^*}^2)\sqrt{B_u^2 + B_u^2}C_{UV}\left[\nmF{U^*}^2 + \nmF{V^*}^2\right]\right)\log(N)+ \log(1/\delta)}{N}}\right)
\ee
holds true w.p at least $1- \delta - 2N\exp\left(-cn_X(g-1)^2\right)$.

Now choose
\be
g = 1 + \mathcal{O}\left(\sqrt{\log(N) + \log(1/\delta)}\right).
\ee
Now ignoring \textit{loglog} terms and keep the right most term
because of the dominance,
\bd
\frac{1}{n_Y}\left|\NC_{\mu}(\{W_j\}) - \NC_{\mu_N}(\{W_j\})\right| \lesssim
\frac{\l}{n_Y} \Omega(f_{\mu}^*)\left[\Omega_{\mu_N}^{\circ}\left(-\frac{1}{\l}\nabla_{\hat{Y}}\ell\left(g, \Phi_r(\{W_j\})\right)\right)-1\right] 
\ed
\be
+ C_{UV}^2\left[\nmF{U^*}^2 + \nmF{V^*}^2\right]^2\sqrt{\frac{R
log\left(R\left(C_{UV} + B_{u}^2 + B_{v}^2\right)\right)(m+n)\log(N)+ \log(1/\delta)}{N}}
\ee

holds true w.p at least $1- \delta$.

\end{proof}

\subsection{Two-Layer ReLU NN}\label{sec:apdx_2rln}

Next, we present and prove the generalization bound for the two-layer ReLU neural network. This is one
step ahead of all the linear models that were discussed earlier. Similarly to the Gaussian projections
discussed in matrix sensing, we discuss ReLU projection results that will be used in the main proof.

\begin{lemma}[ReLU projection 1]\label{lemma:proj_relu_1}
Consider $U_1, U_2 \in \R^{m \times r}$,
$V_1, V_2 \in \R^{n \times r}$. Denote,
convex set $\C = \mathbb{B}(g)$ that is $g$-radius hyper
sphere, then we have that
\be
\begin{split}
\Big|\mathbb{E}\Big[\nmm{U_1[V_1^T\P_{\C}(\x)]_+ &- U_2[V_2^T\P_{\C}(\x)]_+}^2
-\nmm{U_1[V_1^T\x]_+ - U_2[V_2^T\x]_+}^2
\Big]\Big| \\ 
&\leq 2ge^{-g^2/2}[\nmm{U_1}_F^2\nmm{V_1}_F^2 + 
\nmm{U_2}_F^2\nmm{V_2}_F^2].
\end{split}
\ee
\end{lemma}

\begin{proof}
First, we re-write
\begin{eqnarray}
\nmm{U_1[V_1^T\x]_+ - U_2[V_2^T\x]_+}^2
&=& \nmm{\sum_{j=1}^{r}\u_{j1}[\v_{j1}^T\x]_+ - \u_{j2}[\v_{j2}^T\x]_+}^2\\
&=& \nmm{\sum_{j=1}^{r}\u_{j1}\v_{j1}^T\x{\mathbf{1}_{\v_{j1}^T\x \geq 0}} - \u_{j2}\v_{j2}^T\x{\mathbf{1}_{\v_{j1}^T\x \geq 0}}}^2
\end{eqnarray}
\be\label{eq:re_1}
\begin{split}
= \sum_{j=1}^{r}\sum_{l=1}^{r}\Big[
\IP{(\u_{j1}\v_{j1}^T)^T(\u_{j1}\v_{j1}^T)}{\x\x^T{\mathbf{1}_{\v_{j1}^T\x \geq 0}}}
& +\IP{(\u_{j2}\v_{j2}^T)^T(\u_{j2}\v_{j2}^T)}{\x\x^T{\mathbf{1}_{\v_{j2}^T\x \geq 0}}}\\
&-2\IP{(\u_{j2}\v_{j2}^T)^T(\u_{j1}\v_{j1}^T)}{\x\x^T{\mathbf{1}_{\v_{j1}^T\x \geq 0}}{\mathbf{1}_{\v_{j2}^T\x \geq 0}}}
\Big].
\end{split}
\ee

Note that ${\mathbf{1}_{\v^T\x > 0}} = {\mathbf{1}_{\v^T\P_{\C}(\x) > 0}}$. Similarly, we have
\be\label{eq:re_2}
\begin{split}
\nmm{U_1[V_1^T\P_{\C}(\x)]_+ - U_2[V_2^T\P_{\C}(\x)]_+}^2 &= \sum_{j=1}^{r}\sum_{l=1}^{r}\Big[
\IP{(\u_{j1}\v_{j1}^T)^T(\u_{j1}\v_{j1}^T)}{\P_{\C}(\x)\P_{\C}(\x)^T{\mathbf{1}_{\v_{j1}^T\x \geq 0}}}\\
& +\IP{(\u_{j2}\v_{j2}^T)^T(\u_{j2}\v_{j2}^T)}{\P_{\C}(\x)\P_{\C}(\x)^T{\mathbf{1}_{\v_{j2}^T\x \geq 0}}}\\
& -2\IP{(\u_{j2}\v_{j2}^T)^T(\u_{j1}\v_{j1}^T)}{\P_{\C}(\x)\P_{\C}(\x)^T{\mathbf{1}_{\v_{j1}^T\x \geq 0}}{\mathbf{1}_{\v_{j2}^T\x \geq 0}}}\Big].
\end{split}
\ee

Now computing the difference between equations \eqref{eq:re_1},
and \eqref{eq:re_2} we get
\bd
\left|\mathbb{E}\left[\nmm{U_1[V_1^T\P_{\C}(\x)]_+ - U_2[V_2^T\P_{\C}(\x)]_+}^2-\nmm{U_1[V_1^T\x]_+ - U_2[V_2^T\x]_+}^2\right]\right| = 
\ed
\be
\begin{split}
\Big|
\sum_{j=1}^{r}\sum_{l=1}^{r}\mathbb{E}\Big[&
\IP{(\u_{j1}\v_{j1}^T)^T(\u_{j1}\v_{j1}^T)}{(\P_{\C}(\x)\P_{\C}(\x)^T - \x\x^T){\mathbf{1}_{\v_{j1}^T\x \geq 0}}}\\
&+\IP{(\u_{j2}\v_{j2}^T)^T(\u_{j2}\v_{j2}^T)}{(\P_{\C}(\x)\P_{\C}(\x)^T - \x\x^T){\mathbf{1}_{\v_{j2}^T\x \geq 0}}}\\
&-2\IP{(\u_{j2}\v_{j2}^T)^T(\u_{j1}\v_{j1}^T)}{(\P_{\C}(\x)\P_{\C}(\x)^T - \x\x^T){\mathbf{1}_{\v_{j1}^T\x \geq 0}}{\mathbf{1}_{\v_{j2}^T\x \geq 0}}}\Big]\Big|.
\end{split}
\ee

After applying Lemma \ref{lemma:proj_gauss} and the triangular inequality
we obtain, for $g\geq 1$,
\be
\begin{split}
\Big|\mathbb{E}\Big[\nmm{U_1[V_1^T\P_{\C}(\x)]_+ &- U_2[V_2^T\P_{\C}(\x)]_+}^2-\nmm{U_1[V_1^T\x]_+ - U_2[V_2^T\x]_+}^2\Big]\Big|\\
&\overset{(a)}{\leq} ge^{-g^2/2}\sum_{j=1}^{r}\Big[\sum_{l=1}^{r}\nmm{\u_{j1}}^2\nmm{\v_{j1}}^2\\
&+ 2\nmm{\u_{j1}}\nmm{\u_{j2}}\nmm{\v_{j1}}\nmm{\v_{j2}}
+ \nmm{\u_{j2}}^2\nmm{\v_{j2}}^2\Big]\\
&\overset{(b)}{\leq} 2ge^{-g^2/2}\sum_{j=1}^{r}\sum_{l=1}^{r}\Big[\nmm{\u_{j1}}^2\nmm{\v_{j1}}^2
+ \nmm{\u_{j2}}^2\nmm{\v_{j2}}^2\Big]\\
&\overset{(c)}{\leq} 2ge^{-g^2/2}[\nmm{U_1}_F^2\nmm{V_1}_F^2 
+ \nmm{U_2}_F^2\nmm{V_2}_F^2].
\end{split}
\ee
In (a) we apply triangular inequality and apply Lemma \ref{lemma:proj_gauss}, 
(b) we use the identity that $a^2 + 2ab +b^2 = (a+b)^2$, and 
(c) we use the identity that $(a^2c^2 + b^2d^2) \leq (a^2 + b^2)(c^2 + d^2)$. This completes the proof.

\end{proof}

\begin{lemma}[ReLU projection 2]\label{lemma:proj_relu_2}
Consider $U_1, U_2 \in \R^{m \times r}$,
$V_1, V_2 \in \R^{n \times r}$. Denote the
convex set $\C = \mathbb{B}(g)$; that is, the $g$-radius hyper
sphere. Then we have that
\be
\begin{split}
\Big|&\mathbb{E}\Big[\IP{U_1[V_1^T\P_{\C}(\x)]_+ - U_2[V_2^T\P_{\C}(\x)]_+}{U'_1[{V'_1}^T\P_{\C}(\x)]_+ 
- U'_2[{V'_2}^T\P_{\C}(\x)]_+}\\
& -\IP{U_1[V_1^T\x]_+ - U_2[V_2^T\x]_+}{U'_1[{V'_1}^T\x]_+ - U'_2[{V'_2}^T\x]_+}
\Big]\Big| \\
&\leq 2ge^{-g^2/2}[\nmF{U_1}\nmF{U'_1}\nmF{V_1}\nmF{V'_1} + 
\nmF{U_2}\nmF{U'_2}\nmF{V_2}\nmF{V'_2}].
\end{split}
\ee
\end{lemma}

\begin{proof}
The proof is similar to the proof of Lemma \ref{lemma:proj_relu_1}.
\end{proof}

\begin{corollary}[Two-Layer ReLU Neural Network]  Consider the true model for $(\x,\y)$, where $\x \sim \N(0, (1/n)I_{n}) \in \R^{n}$, $\y = U^*[{V^*}^T\x]_+ + \epsilon$, where $U^* \in \R^{m \times {R^*}}$, $V^* \in \R^{n \times {R^*}}$, and $\epsilon \sim \N(0, (\sigma^2/m)I_{m}) \in \R^{m}$ independent from $\x$. For all $i\in [N]$, let $(\x_i,\y_i)$ be i.i.d. samples from this true model. Consider the estimator $\hat{\y} = U[V^T\x]_+$, where $U \in \R^{m \times R}, V \in \R^{n \times R}$. Let $\delta \in (0,1]$ be fixed.  Define the non-convex problem
\be
\begin{split}
\NC_{\mu_N}^{\sf {ReLU}}( (U,V)) := \frac{1}{2N}\sum_{i=1}^{N}\nmeusq{\y_i - U[V^T\x_i]_+}
+ \frac{\l}{2}\left(\nmF{U}^2 + \nmF{V}^2\right),
\end{split}
\ee
and define $\NC^{\sf {ReLU}}_{\mu}((U,V))$ similarly with the sum over $i$ replaced by expectation taken over $(\x, \y)$. 

Let $(U,V)$ be a stationary point of $\NC^{\sf {ReLU}}_{\mu_N}(( U,V))$. 
Suppose there exists $C_{UV}, B_u, B_v > 0$ such that $\nmm{UV^T}_2 \leq C_{UV}\left[\nmF{U^*}^2 + \nmF{V^*}^2\right]$, and for 
all $j \in [R]$, $\nmm{\u_j}_2 \leq B_u$, $\nmm{\v_j}_2 \leq B_v$.
% Suppose that $(U,V)$ is any stationary point satisfying $\nmm{UV^T}_2 \leq \frac{\sqrt{\sigma^2 + 1}}{2}\left[\nmF{U^*}^2 + \nmF{V^*}^2\right]$, and that for any 
% $j \in [R]$ we have $\nmm{\u_j}_2 \leq B_u$ and $\nmm{\v_j}_2 \leq B_v$. 
Then with probability at least $1 - \delta$, it holds that
\begin{align}
&\frac{1}{m}\left|\NC_{\mu}^{\sf {ReLU}}((U, V)) - \NC_{\mu_N}^{\sf {ReLU}}((U, V))\right| \lesssim
\nonumber \frac{1}{2m}\left[\nmF{U^*}^2+\nmF{V^*}^2\right]\left[\frac{1}{N}\sum_{i=1}^{N}\nmm{\y_i - \hat{\y}_i}_2\nmm{\x_i}_2\hspace{-1pt}-\hspace{-1pt}\l\right] \\
\nonumber &+ C_{UV}^2\left[\nmF{U^*}^2 + \nmF{V^*}^2\right] 
\nonumber \left[\hspace{-2pt}\frac{R(m+n)
\log\left(R(m\hspace{-2pt}+\hspace{-2pt}n)(C_{UV}\hspace{-2pt}+\hspace{-2pt}B_u^2\hspace{-2pt}+\hspace{-2pt}B_v^2)\right)\log(N) + \log(1/\delta)}{N}\right]^{1/2}.
\end{align}
\end{corollary}

\begin{proof}

To obtain a generalization bound from Theorem \ref{thm:gen_master} for this setting, we set the following problem
parameters:
\be
\ell(Y, \hat{Y}) = \frac{1}{2}\nmm{Y - \hat{Y}} \implies (\alpha, L) = (0, 1)
\ee
\be
\phi(W) = [\IP{\v}{\x}]_+\u;
\ee
\be
\theta(W) = \frac{1}{2}\left[\nmm{\u}_2^2 + \nmm{\v}_2^2\right].
\ee

\textbf{Estimating $\Omega(f_{\mu}^*)$}:
From Proposition \ref{prop:regular} we have that
\be
\Omega(f_{\mu}^*)
\leq \frac{\nmF{U^*}^2 + \nmF{V^*}^2}{2}
\ee
\textbf{Choosing $\F_{\theta}$}:
\be
\F_{\theta} 
:= \{(\u, \v): \nmm{\u}^2 + \nmm{\v}^2 \leq 2, \nmm{\u}_2 \leq 1, \nmm{\v}_2 \leq 1\}.
\ee
\textbf{Estimating $L_{\phi}$}:
The Lipschtiz constant $L_{\phi}$ in the function $\F_{\theta}$ is $L_{\phi} := \sup_{\nmm{\u} \leq 1,
\nmm{\v} \leq 1}\nmL{\u[\v^T.]_+} \leq \sup_{\nmm{\u} \leq 1,
\nmm{\v} \leq 1}\nmm{\u}\nmm{\v} = 1$.

\textbf{Estimating $r_{\theta}$}:
For any $(\u, \v) \in \F_{\theta}$, we have that,
\be
\frac{\nmm{\u}^2 + \nmm{\v}^2}{2} \leq \sqrt{\frac{\nmm{\u}^2 + \nmm{\v}^2}{2}}
\implies \F_{\theta} \subseteq \mathbb{B}(1/\sqrt{2}).
\ee
Then we have $r_{\theta} = 1/\sqrt{2}$.

\textbf{Choosing $\F_{\mathcal{W}}$}:
From the corollary's assumptions we have that $\B_R := \{(\u, \v): \nmm{\u}_2 \leq B_u, \nmm{\v}_2 \leq B_v\}$;
our hypothesis class is defined as
\be
\F_{\mathcal{W}} 
:= \{(U, V): \nmL{U[V^T.]_+} \leq \nmm{U}_2\nmm{V}_2 \leq \gamma, \nmm{\u_j} \leq B_u,
\nmm{\v_j} \leq B_v\}.
\ee

From Proposition \ref{prop:regular}, we have that,
$\Omega(f_{\mu}^*) \leq \frac{1}{2}\left[\nmF{U^*}^2 + \nmF{V^*}^2\right]$. 
As we require $\gamma \geq \Omega(f_{\mu}^*) L_{\phi} = \frac{1}{2}\left[\nmF{U^*}^2 + \nmF{V^*}^2\right]$,  
we set $\gamma = C_{UV}\frac{\left[\nmF{U^*}^2 + \nmF{V^*}^2\right]}{2}$.
\be
\begin{split}
\F_{\mathcal{W}} = \Big\{\{(\u_j, \v_j)\}: \nmL{U[V^T]_+{\cdot}} &= \nmm{UV^T}_2 \leq \frac{C_{UV}}{2}\left[\nmF{U^*}^2 + \nmF{V^*}^2\right], \\
& \nmm{\u_j} \leq B_u, \nmm{\v_j} \leq B_v\Big\}.
\end{split}
\ee

\textbf{Estimating $\Omega_{\mu_N}^{\circ}(\cdot)$}: We have
\begin{eqnarray}
\Omega_{\mu_N}^{\circ}\left(-\frac{1}{\l}\nabla_{\hat{Y}}\ell(g, \Phi_r(\{W_j\}))\right)
&=& \Omega_{\mu_N}^{\circ}\left(\frac{1}{\l}(g - \Phi_r(\{W_j\}))\right)\\
&=& \hspace{-8pt} \sup_{\nmm{\u} \leq 1; \nmm{\v} \leq 1}\frac{1}{N \l}\sum_{i=1}^{N}\IP{Y_i - U[V^T\x_i]_+}{\u[\v^T\x_i]_+}\\
&=& \sup_{\nmm{\v} \leq 1}\frac{1}{N \l}\sum_{i=1}^{N}[\v^T\x_i]_+\nmm{Y_i - \hat{Y}_i}_2\\
&\leq& \frac{1}{N \l}\sum_{i=1}^{N}\sup_{\nmm{\v} \leq 1}[\v^T\x_i]_+\nmm{Y_i - \hat{Y}_i}_2\\
&=& \frac{1}{N \l}\sum_{i=1}^{N}\nmm{\x_i}_2\nmm{Y_i - \hat{Y}_i}_2.
\end{eqnarray}

\textbf{Estimating $\epsilon_0$}:
From the data generating mechanism we have $\nmL{g} \leq \nmm{U^*}_2\nmm{V^*}_2 \leq \frac{1}{2}\left[\nmF{U^*} + \nmF{V^*}\right]$, $\sigma_X = 1$, $\sigma_{Y|X} = \sigma$. Then we have the following constants from Theorem \ref{thm:gen_master}:
\be
\epsilon_0 = 16\gamma^2\sigma_X^2\min\left\{1,
\frac{L}{4}\left[1 + \frac{\nmL{g}^2}{\gamma^2}\left(1 + \frac{\sigma_{Y|X}^2}{\sigma_X^2}\right)\right]\right\},
\ee
which evaluates to
\be
\begin{split}
\epsilon_0 = 8C_{UV}^2\left[\nmF{U^*}^2 + \nmF{V^*}^2\right]  \min\left\{1, \frac{(1+\sigma^2)}{4C_{UV}^2}\right\}.
\end{split}
\ee
Let $C_{UV} \leq 0.5\sqrt{1+\sigma^2}$ then we have
\be
\epsilon_0  = 2(1+\sigma^2)\left[\nmF{U^*}^2 + \nmF{V^*}^2\right].
\ee

\textbf{Estimating $\epsilon_1$}:
Similarly,
\be
\epsilon_1 = 16\gamma^2\sigma_X^2\max\left\{1,
\frac{L}{4}\left[1 + \frac{\nmL{g}^2}{\gamma^2}\left(1 + \frac{\sigma_{Y|X}^2}{\sigma_X^2}\right)\right]\right\},
\ee
obtaining

\be
\epsilon_1 = 8C_{UV}^2\left[\nmF{U^*}^2 + \nmF{V^*}^2\right].
\ee

\textbf{Defining a convex set $\C$}:
Consider a convex set $\C = \mathbb{B}(g) = \{\x: \nmm{\x}_2 \leq g\}$. 

First and foremost we need to estimate, $\delta_{\C}$ for the following inequality to hold:
\be
P(\cap_{i=1}^{N}\x_i \in \C) \geq 1 - \delta_{\C}.
\ee
The probability of $\x \in \C = \mathbb{B}(g)$ is equivalent to saying the
probability of the event when $\nmm{\x}_2 \leq g$. Since, $x_{i} \sim \N(0, 1/n)$ 
as a consequence of Bernstein's Inequality \citep[Corollary 2.8.3]{vershynin_high-dimensional_2018} we have that,
for any $t \geq 0$,
\be
P(\left|\nmm{\x}_2 -1 \right|\leq t) \geq 1 - 2\exp\left(-cn_Xt^2\right),
\ee
for some constant $c \geq 0$. Now  we have
\be
P(\nmm{\x}_2 \leq g) \begin{cases}
\geq 1 - 2\exp\left(-cn_X(g-1)^2\right) & \text{ if }g \geq 1\\
\leq 2\exp\left(-cn_X(g-1)^2\right) & \text{ otherwise. }
\end{cases}
\ee
We consider the case where $g \geq 1$ yielding
\be
P(\cap_{i=1}^{N}\x_i \in \C) = P(\cap_{i=1}^{N}\nmm{\x}_2 \leq g) \geq 1 - \underbrace{2N\exp\left(-cn_X(g-1)^2\right)}_{=\delta_{\C}}.
\ee
We have that $\delta_{\C} = 2N\exp\left(-cn(g-1)^2\right)$.

Now we evaluate $B_{\ell}, B_{\Phi}, \tilde{L}_{\Phi}, \tilde{L}_{\phi}$.

\textbf{Estimating $B_{\Phi}$}:
We have
\begin{eqnarray}
B_{\Phi} &=& \sup_{\z \in \C, \{(\u_j, \v_j)\} \in \F_{\mathcal{W}}}\nmm{U[V^T\z]_+}_2\\
&\leq& \sup_{\z \in \C, \{(\u_j, \v_j)\} \in \F_{\mathcal{W}}}\nmm{U}_2\nmm{[V^T\z]_+}_2\\
&\leq& \sup_{\z \in \C, \{(\u_j, \v_j)\} \in \F_{\mathcal{W}}}\nmm{U}_2\nmm{V}_2\nmm{\z}_2\\
&=& g\gamma.
\end{eqnarray}

\textbf{Estimating $B_{\ell}$}:
Similarly, we have
\begin{eqnarray}
B_{\ell} &=& \sup_{\z \in \C, \{(\u_j, \v_j)\} \in \F_{\mathcal{W}}}\nmm{U[V^T\z]_+ - U^*[{V^*}^T\z]}_2\\
&=& 2g\gamma
\end{eqnarray}

\textbf{Estimating $\tilde{L}_{\Phi}$}:
Now, we compute the Lipschitz constant with respect to $U, V$. We have
\begin{eqnarray}
\tilde{L}_{\Phi} &=& \sup_{\z \in \C, (U, V), (U', V') \in \F_{\mathcal{W}}}\frac{\nmm{U[V^T\z]_+-U'[V'^T\z]_+}}{\max_{j}\sqrt{\nmm{\u_j-\u_j'}^2
+ \nmm{\v_j-\v_j'}^2}}\\
&=& R\sup_{\z \in \C, (U, V), (U', V') \in \F_{\mathcal{W}}}\frac{\nmm{\u[\v^T\z]_+-\u'[\v'^T\z]_+}}{\sqrt{\nmm{\u-\u'}^2
+ \nmm{\v-\v'}^2}}\\
&=& R\sup_{\z \in \C, (U, V), (U', V') \in \F_{\mathcal{W}}}\frac{\nmm{(\u-\u')[\v^T\z]_+-\u'[[\v'^T\z]_+-[\v^T\z]_+]}}{\sqrt{\nmm{\u-\u'}^2
+ \nmm{\v-\v'}^2}}\\
&\leq&
R\sup_{\z \in \C, (U, V), (U', V') \in \F_{\mathcal{W}}}\frac{\nmm{(\u-\u')[\v^T\z]_+} + \nmm{\u'[[\v'^T\z]_+-[\v^T\z]_+]}}{\sqrt{\nmm{\u-\u'}^2
+ \nmm{\v-\v'}^2}}\\
&\leq&
R\sup_{\z \in \C, (U, V), (U', V') \in \F_{\mathcal{W}}}\frac{\nmm{(\u-\u')[\v^T\z]_+} + \nmm{\u'[[\v'^T\z]_+-[\v^T\z]_+]}}{\sqrt{\nmm{\u-\u'}^2+ \nmm{\v-\v'}^2}}\\
&\leq&
gR\sup_{(U, V), (U', V') \in \F_{\mathcal{W}}}\frac{B_v\nmm{(\u-\u')} + B_u\nmm{\v-\v'}}{\sqrt{\nmm{\u-\u'}^2+ \nmm{\v-\v'}^2}}\\
&=& g\sqrt{B_u^2 + B_v^2}R.
\end{eqnarray}

\textbf{Estimating $\tilde{L}_{\phi}$}:
Similarly we get $\tilde{L}_{\phi} = g\sqrt{B_u^2 + B_v^2}$. %$$ as we have only one slice of factor.

\textbf{Estimating $\epsilon_2$}:
Recall that
\be
\epsilon_2 = \max\{8B_{\ell}\tilde{L}_{\Phi}, 8\tilde{L}_{\Phi}[B_{\ell} + B_{\Phi}L], 32\Omega(f_{\mu}^*)\tilde{L}_{\phi}\max\{B_{\ell}, LB_{\Phi}\}, 4\tilde{L}_{\Phi}B_{\Phi}\}.
\ee

From all the constants computed earlier, we have that
\be
\epsilon_2 = k_1g^2R^2C_{UV}\sqrt{B_u^2 + B_v^2}\left[\nmF{U^*}^2 + \nmF{V^*}^2\right],
\ee
for some constant $k_1 \geq 0$.

Next, we move on to estimating $B(\C)$. We need to analyze three terms:
\\
\textbf{The First Term:} Define
\be
T_1 := \sup_{\{W_j\} \in \F_{\mathcal{W}}}\left|\nmm{f_{\mu}^* \circ \P_{\C} - \Phi_r(\{W_j\}) \circ \P_{\C}}_{\mu}^2
-\nmm{f_{\mu}^* - \Phi_r(\{W_j\})}_{\mu}^2\right|.
\ee
For a fixed $(U, V)$, we have
\be
\begin{split}
&\Big|\nmm{f_{\mu}^* \circ \P_{\C} - \Phi_r(\{W_j\}) \circ \P_{\C}}_{\mu}^2
- \nmm{f_{\mu}^* - \Phi_r(\{W_j\})}_{\mu}^2\Big|\\
&= \Big|\mathbb{E}\Big[\nmm{U^*[{V^*}^T\P_{\C}(\x)]_+ - U[V^T\P_{\C}(\x)]_+}^2 
- \nmm{U^*[{V^*}^T\x)_+ - U[V^T\x]_+}^2\Big]\Big|.
\end{split}
\ee
From Lemma \ref{lemma:proj_relu_1} taking $g \geq 1$ we have
\be
\begin{split}
\Big|\nmm{f_{\mu}^* \circ \P_{\C} - \Phi_r(\{W_j\}) \circ \P_{\C}}_{\mu}^2
&-\nmm{f_{\mu}^* - \Phi_r(\{W_j\})}_{\mu}^2\Big|\\
&\leq 2ge^{-g^2/2}\left[\nmF{U^*}^2\nmF{V^*}^2 + \nmF{U}^2\nmF{V}^2\right],
\end{split}
\ee
whereupon further simplifying, we obtain
\be
\begin{split}
\sup_{(U, V) \in \F_{\mathcal{W}}}\Big|\nmm{f_{\mu}^* \circ \P_{\C} &- \Phi_r(\{W_j\}) \circ \P_{\C}}_{\mu}^2
-\nmm{f_{\mu}^* - \Phi_r(\{W_j\})}_{\mu}^2\Big|\\
&\leq 2ge^{-g^2/2}(\nmF{U^*}^2\nmF{V^*}^2 + R^2\gamma^2).
\end{split}
\ee

Now, applying triangular inequality and taking the supremum we obtain
\be\label{eq:b1_rel}
T_2 \leq 2ge^{-g^2/2}(\nmF{U^*}^2\nmF{V^*}^2 + R^2\gamma^2).
\ee

\textbf{The Second Term}: Define
\be
\begin{split}
T_2 := \sup_{\{W_j\} \in \F_{\mathcal{W}}, W' \in \F_{\theta}}&\Big|\IP{\nabla_{\hat{Y}}\ell\left(g \circ \P_{\C}, \Phi_r(\{W_j\}) \circ \P_{\C}\right)}{\phi(W') \circ \P_{\C}}_{\mu}\\
&- \IP{\nabla_{\hat{Y}}\ell\left(g, \Phi_r(\{W_j\})\right)}{\phi(W')}_{\mu}\Big|.
\end{split}
\ee
For a fixed $(U, V), (\u, \v)$ we have
\bd
\left|\IP{\nabla_{\hat{Y}}\ell\left(g \circ \P_{\C}, \Phi_r(\{W_j\}) \circ \P_{\C}\right)}{\phi(W') \circ \P_{\C}}_{\mu}
- \IP{\nabla_{\hat{Y}}\ell\left(g, \Phi_r(\{W_j\})\right)}{\phi(W')}_{\mu}\right|
= 
\ed
\be
\begin{split}
\Big|\mathbb{E}\Big[\IP{U[V^T\P_{\C}(\x)]_+-U^*[{V^*}^T\P_{\C}(\x)]_+}{\u[\v^T\P_{\C}(\x)]_+}\\-\IP{U[V^T\P_{\C}(\x)]_+-U^*[{V^*}^T\x]_+}{\u[\v^T\x]_+}\Big]\Big|.
\end{split}
\ee

As a consequence of Lemma \ref{lemma:proj_relu_2} we have

\bd
\left|\IP{\nabla_{\hat{Y}}\ell\left(g \circ \P_{\C}, \Phi_r(\{W_j\}) \circ \P_{\C}\right)}{\phi(W') \circ \P_{\C}}_{\mu}
- \IP{\nabla_{\hat{Y}}\ell\left(g, \Phi_r(\{W_j\})\right)}{\phi(W')}_{\mu}\right|
\ed
\be
\leq 2ge^{-g^2/2}\left[\nmF{U}\nmF{V}\nmm{\u}\nmm{\v}
+ \nmF{U^*}\nmF{V^*}\nmm{\u}\nmm{\v}\right].
\ee

Now we apply supremum over $(\u, \v) \in \F_{\theta}$, obtaining
\be
T_2 \leq 2ge^{-g^2/2}\left[\nmF{U}\nmF{V}
+ \nmF{U^*}\nmF{V^*}\right].
\ee
Finally, we apply supremum over $(U, V) \in \F_{\mathcal{W}}$, obtaining
\be\label{eq:b2_rel}
T_2 \leq 2ge^{-g^2/2}\left[\nmF{U^*}\nmF{V^*} + R\gamma\right].
\ee

\textbf{The Third Term:} Define 
\be
\begin{split}
T_3 := \sup_{\{W_j\} \in \F_{\mathcal{W}}}&\Big|\IP{\nabla_{\hat{Y}}\ell\left(g \circ \P_{\C}, \Phi_r(\{W_j\}) \circ \P_{\C}\right)}{\Phi_r(\{W_j\}) \circ \P_{\C}}_{\mu}\\
&- \IP{\nabla_{\hat{Y}}\ell\left(g, \Phi_r(\{W_j\})\right)}{\Phi_r(\{W_j\})}_{\mu}\Big|.
\end{split}
\ee
For a fixed $(U, V)$ we can rewrite the above to
\be
\begin{split}
\Big|\mathbb{E}\Big[\IP{U[V^T\P_{\C}(\x)]_+-U^*[{V^*}^T\P_{\C}(\x)]_+}{U[V^T\P_{\C}(\x)]_+}\\
-\IP{U[V^T\P_{\C}(\x)]_+-U^*[{V^*}^T\x]_+}{U[V^T\x]_+}\Big]\Big|.
\end{split}
\ee
As a consequence of Lemma \ref{lemma:proj_gauss} we have
\bd
\left|\IP{\nabla_{\hat{Y}}\ell\left(g \circ \P_{\C}, \Phi_r(\{W_j\}) \circ \P_{\C}\right)}{\Phi_r(\{W_j\}) \circ \P_{\C}}_{\mu}
- \IP{\nabla_{\hat{Y}}\ell\left(g, \Phi_r(\{W_j\})\right)}{\Phi_r(\{W_j\})}_{\mu}\right|
\ed
\be
\leq 2ge^{-g^2/2}\nmF{U}\nmF{V}\left[\nmF{U}\nmF{V} + \nmF{U^*}\nmF{V^*}\right].
\ee

Finally, we apply supremum over $(U, V) \in \F_{\mathcal{W}}$, obtaining
\be\label{eq:b3_rel}
T_3 \leq 2ge^{-g^2/2}R\gamma\left[\nmF{U^*}\nmF{V^*} + R\gamma\right].
\ee

Now combining $T_1, T_2$ and $T_3$ 
from equations \eqref{eq:b1_rel}, \eqref{eq:b2_rel}, \eqref{eq:b3_rel} we have
\be
B(\C)
\leq 2ge^{-g^2/2}\left[\alpha(\nmF{U^*}^2\nmF{V^*} + R^2\gamma^2) + \nmF{U}\nmF{V} + R\gamma +\hspace{-2pt}R\gamma\left[\nmF{U^*}\nmF{V^*}\hspace{-2pt}+ R\gamma\right]\right].
\ee
We further upper bound for simplicity via
\be %\label{eq:b_ms}
B(\C)
\leq 4Rge^{-g^2/2}\gamma\left[\nmF{U}\nmF{V} + \gamma\right].
\ee

From Theorem \ref{thm:gen_master} we have that 
\bd
\frac{1}{m}\left|\NC_{\mu}(\{W_j\}) - \NC_{\mu_N}(\{W_j\})\right| \lesssim
\frac{\l}{2m}\left[\nmF{U^*}^2 + \nmF{V^*}^2\right]\left[\sup_{\nmm{\v} \leq 1}\frac{1}{N}\sum_{i=1}^{N}[\v^T\x_i]_+\nmm{Y_i - U[V^T\x_i]_+}-\l\right] 
\ed
\bd
+  \frac{2}{m}Rge^{-g^2/2}C_{UV}\left[\nmF{U^*}^2 + \nmF{V^*}^2\right]\left[\nmF{U}\nmF{V} + C_{UV}\left[\nmF{U^*}^2 + \nmF{V^*}^2\right]\right]
+ C_{UV}^2\left[\nmF{U^*}^2 + \nmF{V^*}^2\right]^2
\ed
\be
\sqrt{\frac{R(m+n)
\log\left(k_1g^2R^2C_{UV}^2\sqrt{B_u^2 + B_v^2}\left[\nmF{U^*}^2 + \nmF{V^*}^2\right]^2\right)\log(N)+ \log(1/\delta)}{N}},
\ee

holds true w.p at least $1- \delta - 2N\exp\left(-cn_X(g-1)^2\right)$.

Now choose
\be
g = 1 + \mathcal{O}\left(\sqrt{\log(NR) + \log(1/\delta)}\right).
\ee
After ignoring all the log-log terms and using only dominant terms,
we have that
\bd
\frac{1}{m}\left|\NC_{\mu}(\{W_j\}) - \NC_{\mu_N}(\{W_j\})\right| \lesssim
\ed
\bd
\frac{\l}{2m}\left[\nmF{U^*}^2 + \nmF{V^*}^2\right]\left[\sup_{\nmm{\v} \leq 1}\frac{1}{N}\sum_{i=1}^{N}[\v^T\x_i]_+\nmm{Y_i - U[V^T\x_i]_+}-\l\right] 
\ed
\be
+ C_{UV}^2\left[\nmF{U^*}^2 + \nmF{V^*}^2\right]^2
\sqrt{\frac{R(m+n)
\log\left(R(m+n)(C_{UV} + B_{u}^2 + B_{v}^2)\right)\log(N)+ \log(1/\delta)}{N}},
\ee
holds true w.p at least $1- \delta$.
\end{proof}

\subsection{Multi-head Attention}\label{sec:apdx_mha}

Next, we move to applying Theorem \ref{thm:gen_master} to the single-layer multi-head attention 
problem. We require similar Gaussian projections arguments onto convex sets are needed to be established.
For this application, we require Gaussian projections
onto softmax, which is analyzed through Lemma \ref{lemma:prog_sftmax}. First,
we define the soft max operation, $\sigma_t(\cdot): \R^{n} \to \R^{n}$.
\be
[\sigma_t(\u)]_i := \exp(tu_i)/(\sum_{j=1}^{n}\exp(tu_j)),
\ee
where $t$ is called the temperature.

A discrete version of soft-max is known as hard-max that is defined as
\be
[\sigma(\u)]_i := {\mathbf{1}_{u_i = \max_{i}u_i}}.
\ee

Note that when $t \to \infty$, $\sigma_t(\u) \to \sigma(\u)$.

\begin{lemma}[Gaussian Softmax Projection] \label{lemma:prog_sftmax}
Let $X \in \R^{m \times n}$ and $X_{ij} \sim \N(0, 1/(mn))$
be independent random variables. Suppose
$\sigma_t(\cdot)$ is a softmax with temperature, $t$,
and $M$ is fixed matrix in $\R^{m \times n}$. Consider
a convex set $\C = \{X = (\x_1, \dots, \x_n): \forall j \in [n]; \nmm{\x_j} \leq g\}$ for $g \geq 1$. Then
\bd
\sup_{\z \in \F_{\z}}\left|\IP{M}{\mathbb{E}\left[X\sigma_t(X^T\z)
(X\sigma_t(X^T\z))^T - \P_{\C}(X)\sigma_t(\P_{\C}(X)^T\z)
(\P_{\C}(X)\sigma_t(\P_{\C}(X)^T\z))^T\right]}\right|
\ed
\be
\hspace{10pt}\leq c_1m^2\nmm{M}_Fg\exp\left(-c_2mg^2\right).
\ee
for some positive constant, $c_1, c_2$.
\end{lemma}

\begin{proof}
Denote
\be
T := \sup_{\z \in \F_{\z}}\left|\IP{M}{\mathbb{E}\left[X\sigma_t(X^T\z)
(X\sigma_t(X^T\z))^T - \P_{\C}(X)\sigma_t(\P_{\C}(X)^T\z)
(\P_{\C}(X)\sigma_t(\P_{\C}(X)^T\z))^T\right]}\right|.
\ee
Firstly, we upper bound the earlier term by Cauchy-Schwartz inequality:
\be
T \leq \nmm{M}_F\hspace{-2pt}\sup_{\z \in \F_{\z}}\nmm{\mathbb{E}\left[X\sigma_t(X^T\z)
(X\sigma_t(X^T\z))^T\hspace{-2pt}-\hspace{-2pt}\P_{\C}(X)\sigma_t(\P_{\C}(X)^T\z)
(\P_{\C}(X)\sigma_t(\P_{\C}(X)^T\z))^T\right]}_F.
\ee
Denote $a_t(\z, i, j) = \sigma_t(X^T\z)_i\sigma_t(X^T\z)_j$ and $\tilde{a}_t(\z, i, j) = \sigma_t(\P_{\C}(X)^T\z)_i\sigma_t(\P_{\C}(X)^T\z)_j$,
%= \frac{exp(t\IP{\x_i}{\z})\exp(t\IP{\x_j}{\z})}{(\sum_{k=1}^{T}\exp(t\IP{\x_k}{\z}))^2}$. Then we have
\be
T \leq \nmm{M}_F\sup_{\z \in \F_{\z}}\nmm{\sum_{i=1}^{T}\sum_{j=1}^{T}\mathbb{E}\left[\x_i\x_j^Ta_t(\z, i, j) - \P_{\C}(\x_i)\P_{\C}(\x_j)^T\tilde{a}_t(\z, i, j)\right]}_F.
\ee
Now we apply triangular inequality,
\be
T \leq \nmm{M}_F\sup_{\z \in \F_{\z}}\sum_{i=1}^{T}\sum_{j=1}^{T}\nmm{\mathbb{E}\left[\x_i\x_j^Ta_t(\z, i, j) - \P_{\C}(\x_i)\P_{\C}(\x_j)^T\tilde{a}_t(\z, i, j)\right]}_F.
\ee
Now we apply Cauchy-Schwartz,
\be
T \leq \nmm{M}_F\sum_{i=1}^{m}\sum_{j=1}^{m}\sup_{\z \in \F_{\z}}\nmm{\mathbb{E}\left[\x_i\x_j^Ta_t(\z, i, j) - \P_{\C}(\x_i)\P_{\C}(\x_j)^T\tilde{a}_t(\z, i, j)\right]}_F.
\ee
We can upper bound the earlier term via taking a supremum over indices $i, j \in [T]$ then we have
\be
T \leq \nmm{M}_Fm^2\sup_{i, j}\sup_{\z \in \F_{\z}}\nmm{\mathbb{E}\left[\x_i\x_j^Ta_t(\z, i, j) - \P_{\C}(\x_i)\P_{\C}(\x_j)^T\tilde{a}_t(\z, i, j)\right]}_F.
\ee
Observe that argument inside the expectation is 0 on the event $X \in \C$, thereby we only have the case where
$X \notin \C$ then we have
\be
T \leq \nmm{M}_Fm^2\sup_{i, j}\sup_{\z \in \F_{\z}}\nmm{\mathbb{E}\left[\x_i\x_j^Ta_t(\z, i, j) - \P_{\C}(\x_i)\P_{\C}(\x_j)^T\tilde{a}_t(\z, i, j) {\mathbf{1}_{\E^{c}}}\right]}_F.
\ee
On application Cauchy-Schwartz identity again we have
\be
T \leq \nmm{M}_Fm^2\sup_{i, j}\sup_{\z \in \F_{\z}}\mathbb{E}\left[\nmm{\x_i\x_j^Ta_t(\z, i, j) - \P_{\C}(\x_i)\P_{\C}(\x_j)^T\tilde{a}_t(\z, i, j) }_F {\mathbf{1}_{\E^{c}}}\right].
\ee
We have that $\x \notin \C, \P_{\C}(\x) = g\x/\nmm{\x}$. By using this fact we have
\be
T \leq \nmm{M}_Fm^2\sup_{i, j}\sup_{\z \in \F_{\z}}\mathbb{E}\left[\nmm{\left(a_t(\z, i, j) - \frac{g^2}{\nmm{\x_i}\nmm{\x_j}}\tilde{a}_t(\z, i, j)\right)\x_i\x_j^T}_F {\mathbf{1}_{\E^{c}}}\right].
\ee
Now we recall Reverse Fatou's Lemma, for any function sequence, $f_n \in L^2(\mu)$, we have
\be
\limsup_{n \to \infty}\int f_nd\mu \leq \int \limsup_{n \to \infty}
f_nd\mu.
\ee
on applying this identity we have
\be
T \leq \nmm{M}_Fm^2\sup_{i, j}\mathbb{E}\left[\sup_{\z \in \F_{\z}}\left|a_t(\z, i, j) - \frac{g^2}{\nmm{\x_i}\nmm{\x_j}}\tilde{a}_t(\z, i, j)\right|\nmm{\x_i\x_j^T}_F {\mathbf{1}_{\E^{c}}}\right].
\ee
Observe that $a_t(\z, i, j) \leq 1$ and $\frac{g^2}{\nmm{\x_i}\nmm{\x_j}}\tilde{a}_t(\z, i, j) \leq 1$ when $X \notin \C$ therefore
we have
\be
T \leq \nmm{M}_Fm^2\sup_{i, j}\mathbb{E}\left[\nmm{\x_i\x_j^T}_F {\mathbf{1}_{\E^{c}}}\right].
\ee
Now since $\x_i$ and $\x_j$ are iid we have
\be
T \leq \nmm{M}_Fm^2\sup_{i, j}\mathbb{E}\left[\nmm{\x_i}_2{\mathbf{1}_{\E^{c}}}\right]\mathbb{E}\left[\nmm{\x_j}_2{\mathbf{1}_{\E^{c}}}\right] .
\ee

By Gaussian integral over norm for $g \geq 1$ we have
\be
T \lesssim m^2\nmm{M}_Fg\exp\left(-mg^2\right)
\ee

\end{proof}

With the above result on Gaussian softmax projection, we now state the corollary for 
the single-layer multi-head attention problem and its proof.

\begin{corollary}[Transformers]
Consider the true model for $(X,\y)$, where $X\in\R^{n \times T}$ is a random matrix with i.i.d. entries $X_{lk} \sim \N(0, {1}/{(nT)})$ and $\y = A^*X\b^* + \epsilon$, where $A^* \in \R^{m \times n}$, $\b^* \in \mathbb{S}^{T-1}$ and $\epsilon \sim \N(0, (\sigma^2/m)I_{m})$ is independent from $X$. For all $i \in [N]$, let $(X_i, \y_i)$ be i.i.d. samples from this true model. Consider the estimator
$\hat{\y} = \sum_{j=1}^{R}V_jX\sigma(X^T\z_j)$, 
$V_j \in \R^{n}, \z_j \in \R^{n}$.  Let $\delta \in (0,1]$ be fixed.  Define the non-convex problem

\begin{align}
\nonumber \NC_{\mu_N}^{\sf TF}\hspace{-2pt}(\{(V_j,\z_j)\})\hspace{-3pt} := \frac{1}{2N}\sum_{i=1}^{N}\nmeusq{\y_i-\sum_{j=1}^{R}V_jX_i\sigma_t(X_i^T\z_j)}
+ \l \sum_{j=1}^{R}\left[\nmF{V_j} + \delta_{\{\z: \nmm{\z}_2 \leq 1\}}(\z_j)\right],
\end{align}

where, $\sigma_t(\cdot)$ is softmax function with temperature $t$, for $k \in [T]$ defined $\sigma_t(\u)_k := \exp(tu_k)/\sum_{l=1}^{T}\exp(tu_l)$
and define $\NC_{\mu}^{\sf TF}( \{(V_j,\z_j)\})$ similarly with the sum over $i$ replaced by expectation taken over $(X, \y)$. 

Let $\{(V_j,\z_j)\}$ be a stationary point of $\NC^{\sf TF}_{\mu_N}(\{(V_j,\z_j)\})$. 
Suppose there exists $C_{V}, B_V > 0$ such that $\sum_{j=1}^{R}\nmF{V_j} \leq C_{V}\nmF{A^*}$, and for 
all $j \in [R]$, $\nmF{V_j} \leq B_V$.

Then with probability at least $1 - \delta$, it holds that

\begin{align}
&\frac{1}{m}\left|\NC_{\mu}^{\sf TF}( \{(V_j,\z_j)\}) - \NC_{\mu_N}^{\sf TF}( \{(V_j,\z_j)\})\right| \lesssim
\nonumber \frac{1}{2m}\nmF{A^*}\left[\frac{1}{N}\sum_{i=1}^{N}\nmm{\y_i - \hat{\y}_i}_2\nmm{X_i}_2-\l\right] \\
\nonumber & \hspace{50pt} + C_{V}^2\nmF{A^*}^2
\nonumber \sqrt{\hspace{-2pt}\frac{R(m\hspace{-2pt}+\hspace{-2pt}n)
\hspace{-2pt}\log\hspace{-2pt}\left(R(m\hspace{-2pt}+\hspace{-2pt}n)(C_{V}\hspace{-2pt}+\hspace{-2pt}B_V)\right)\hspace{-2pt}\log(N)\hspace{-2pt}+\hspace{-2pt}\log(1/\delta)}{N}}.
\end{align}
\end{corollary}

\begin{proof}

To obtain a generalization bound from Theorem \ref{thm:gen_master} for the case of matrix sensing, we set the following
problem parameters.
\be
\ell(Y, \hat{Y}) = \frac{1}{2}\nmm{Y - \hat{Y}} \implies (\alpha, L) = (0, 1);
\ee
\be
\phi(W) = VX\sigma_t(X^T\z);
\ee
\be
\theta(W) = \nmF{V} + \delta_{\z \in \mathbb{B}(1)}.
\ee

\textbf{Estimating $\Omega_{\mu_N}(\cdot)$}:
Now we move on to compute the polar:
\begin{eqnarray*}
\Omega_{\mu_N}^{\circ}\left(-\frac{1}{\l}\nabla_{\hat{Y}}\ell(g, \Phi_r(\{W_j\}))\right)
&=& \Omega_{\mu_N}^{\circ}\left(\frac{1}{\l}(g - \Phi_r(\{W_j\}))\right),\\
&=& \sup_{\nmF{V} \leq 1, \nmm{\z} \leq 1}\frac{1}{N \l}\sum_{i=1}^{N}\IP{Y_i - \hat{Y}_i}{VX_i\sigma_t(X^T\z)},\\
&=& \sup_{\nmF{V} \leq 1, \nmm{z} \leq 1}\frac{1}{N\l}\nmm{\sum_{i=1}^{N}\IP{(Y_i - \hat{Y}_i)(X_i\sigma_t(X^T\z))^T}{V}},\\
&=& \sup_{\nmF{V} \leq 1, \nmm{z} \leq 1}\frac{1}{N\l}\nmm{\sum_{i=1}^{N}\IP{(Y_i - \hat{Y}_i)(X_i\sigma_t(X^T\z))^T}{V}},
\end{eqnarray*}
\begin{eqnarray*}
\Omega_{\mu_N}^{\circ}\left(-\frac{1}{\l}\nabla_{\hat{Y}}\ell(g, \Phi_r(\{W_j\}))\right)
&\leq& \sup_{\nmm{z} \leq 1}\frac{1}{N\l}\nmm{\sum_{i=1}^{N}(Y_i - \hat{Y}_i)(X_i\sigma_t(X^T\z))^T}_F,\\
&\leq& \frac{1}{N\l}\sum_{i=1}^{N}\sup_{\nmm{z} \leq 1}\nmm{(Y_i - \hat{Y}_i)(X_i\sigma_t(X^T\z))^T}_F,\\
&\leq& \frac{1}{N\l}\sum_{i=1}^{N}\sup_{\nmm{z} \leq 1}\nmm{Y_i - \hat{Y}_i}_F\nmm{(X_i\sigma_t(X^T\z))^T}_F,\\
&\leq& \frac{1}{N\l}\sum_{i=1}^{N}\sup_{\nmm{z} \leq 1}\nmm{Y_i - \hat{Y}_i}_F\nmm{(X_i\sigma_t(X^T\z))^T}_F,\\
&\leq& \frac{1}{N\l}\sum_{i=1}^{N}\nmm{Y_i - \hat{Y}_i}_2\nmm{X_i}_2.
\end{eqnarray*}

\textbf{Choose $\F_{\theta}$}: From Assumptions
\ref{ass:a3} suppose that
\be
\F_{\theta} 
:= \{(V, \z): \nmm{V}_2 \leq 1, \nmm{\z}_2 \leq 1\};
\ee
\textbf{Computing $L_{\phi}$}:
The Lipschtiz constant $L_{\phi}$ in the function $\F_{\theta}$ is $L_{\phi} := \sup_{\nmm{V} \leq 1,
\nmm{\z} \leq 1}\nmL{V(\cdot)\sigma((\cdot)^T\z)} \leq \sup_{\nmm{V} \leq 1}\nmm{V} = 1$.

\textbf{Computing $r_{\phi}$}: Clearly, when $r_{\theta} = \sqrt{2}$ we have that $\F_{\theta} \subseteq \mathbb{B}(\sqrt{2})$.

\textbf{Choose $\F_{\mathcal{W}}$}: 
From the corollary's assumptions we have that, $\B_R := \{(V, \z): \nmm{V}_2 \leq B_V, \nmm{\z}_2 \leq 1\}$;
our hypothesis class is defined as

\be
\F_{\mathcal{W}} 
:= \{\{(V_j, \z_j)\}: \nmL{\sum_{j=1}^{r}V_j\cdot\sigma((\cdot)^T\z_j)} \leq \sum_{j=1}^{r}\nmF{V_j} \leq \gamma, \nmF{V_j} \leq B_V,
\nmm{\z} \leq 1\}.
\ee

From Proposition \ref{prop:regular} we have
$\Omega(f_{\mu}^*) \leq \nmF{A^*}$. We have
$\gamma \geq \Omega(f_{\mu}^*) L_{\phi} = \nmF{A^*}$,  then we set $\gamma = C_{V}\nmF{A^*}$.
We have that
\be
\begin{split}
\F_{\mathcal{W}} 
:= \Big\{\{(V_j, \z_j)\}&: \nmL{\sum_{j=1}^{r}V_j\cdot\sigma((\cdot)^T\z_j)} \leq \sum_{j=1}^{r}\nmF{V_j} \leq C_{V}\nmF{A^*}, \\
&\nmF{V_j} \leq B_V,
\nmm{\z} \leq 1\Big\}.
\end{split}
\ee

\textbf{Estimating $\epsilon_0$}:
From the data generating mechanism we have $\nmL{g} \leq \nmF{A^*}$, $\sigma_X = 1$, $\sigma_{Y|X} = \sigma$,
then we have the following constants from Theorem \ref{thm:gen_master}:
\be
\epsilon_0 = 16\gamma^2\sigma_X^2\min\left\{1,
\frac{L}{4}\left[1 + \frac{\nmL{g}^2}{\gamma^2}\left(1 + \frac{\sigma_{Y|X}^2}{\sigma_X^2}\right)\right]\right\},
\ee
this evaluates to,
\be
\epsilon_0 = 16(\sigma^2 + 1)\nmF{A^*}^2,
\ee
when $C_V \leq \sqrt{1+\sigma^2}$.

\textbf{Estimating $\epsilon_1$}:
Similarly, we evaluate
\be
\epsilon_1 = 16\gamma^2\sigma_X^2\max\left\{1,
\frac{L}{4}\left[1 + \frac{\nmL{g}^2}{\gamma^2}\left(1 + \frac{\sigma_{Y|X}^2}{\sigma_X^2}\right)\right]\right\},
\ee
obtaining,
\be
\epsilon_1 = 16C_V^2\nmF{A^*}^2,
\ee
when $C_V \leq \sqrt{1+\sigma^2}$.

\textbf{Choosing the convex set $\C$}:
Consider a convex set $\C = \{X = (\x_1, \dots, \x_T): \nmm{\x_j}_2 \leq g/\sqrt{T}\}$. 

First and foremost we need to estimate $\delta_{\C}$ for the inequality to hold:
\be
P(\cap_{i=1}^{N}\x_i \in \C) \geq 1 - \delta_{\C}.
\ee

The probability of $\x \in \C = \mathbb{B}(g)$ is equivalent to saying the
probability of the event when $\nmm{\x}_2 \leq g$. Since, $x_{i} \sim \N(0, 1/n)$ 
as a consequence of Bernstein's Inequality \citep[Corollary 2.8.3]{vershynin_high-dimensional_2018} we have that,
for any $t \geq 0$.

\be
P(\left|\nmm{\x}_2 -1 \right|\leq t) \geq 1 - 2\exp\left(-cnTt^2\right),
\ee
for some constant $c \geq 0$. Now we have
\be
P(\nmm{\x}_2 \leq g/\sqrt{T}) \begin{cases}
\geq 1 - 2\exp\left(-cn(g-1)^2\right) & \text{ if }g \geq 1\\
\leq 2\exp\left(-cn(g-1)^2\right) & \text{ otherwise }
\end{cases}.
\ee

We consider the case where $g \geq 1$. Then we have that
\be
P(\cap_{i=1}^{N}\x_i \in \C) = P(\cap_{i=1}^{N}\nmm{\x}_2 \leq g/\sqrt{T}) \geq 1 - \underbrace{2N\exp\left(-cn(g-1)^2\right)}_{=\delta_{\C}}.
\ee

We have that $\delta_{\C} = 2N\exp\left(-cn(g-1)^2\right)$.

\textbf{Estimating $B_{\Phi}$}:
\begin{eqnarray}
B_{\Phi} &=& \sup_{X \in \C, \{(V_j, \z_j)\} \in \F_{\mathcal{W}}}\nmm{\sum_{j=1}^{r}V_jX\sigma_t(X^T\z_j)}_2,\\
&\leq& R\sup_{X \in \C, \{(V_j, \z_j)\} \in \F_{\mathcal{W}}}\nmm{V_j}_F\nmm{X\sigma_t(X^T\z)}_2,\\
&\leq& R\sup_{X \in \C, \{(V_j, \z_j)\} \in \F_{\mathcal{W}}}\nmm{V_j}_F\nmm{X\sigma_t(X^T\z)}_2,\\
&\leq& gRB_V/\sqrt{T}.
\end{eqnarray}

\textbf{Estimating $B_{\ell}$}: we have

\begin{eqnarray}
B_{\ell} &=& \sup_{X \in \C, \{(V_j, \z_j)\} \in \F_{\mathcal{W}}}\nmm{\sum_{j=1}^{r}V_jX\sigma_t(X^T\z_j)-A^*XB^*},\\
&=& g\left[RB_V + \nmF{A^*}\right]/\sqrt{T}.
\end{eqnarray}

\textbf{Estimating $\tilde{L}_{\Phi}$}: We have 
\be
\tilde{L}_{\Phi} = gR\sqrt{B_V^2 + 1}/\sqrt{T}.
\ee

\textbf{Estimating $\tilde{L}_{\phi}$}:
Similarly we get $\tilde{L}_{\phi} = g\sqrt{B_V^2 + 1}/\sqrt{T}$ as we have only one slice of factor.

\textbf{Estimating $\epsilon_2$}:
Recall that,
\be
\epsilon_2 = \max\{8B_{\ell}\tilde{L}_{\Phi}, 8\tilde{L}_{\Phi}[B_{\ell} + B_{\Phi}L], 32\Omega(f_{\mu}^*)\tilde{L}_{\phi}\max\{B_{\ell}, LB_{\Phi}\}, 4\tilde{L}_{\Phi}B_{\Phi}\}.
\ee

From all the constants computed earlier, we have that,

\be
\epsilon_2 = k_1g^2R^2 B_V^2/T,
\ee
for some constant $k_1 \geq 0$.

Next, we move on estimating $B(\C)$ we need to analyze three terms

\textbf{The First Term} is defined via 

\be
\begin{split}
T_1 := \sup_{\{W_j\} \in \F_{\mathcal{W}}}\Big|&\nmm{f_{\mu}^* \circ \P_{\C} - \Phi_r(\{W_j\}) \circ \P_{\C}}_{\mu}^2\\
&-\nmm{f_{\mu}^* - \Phi_r(\{W_j\})}_{\mu}^2\Big|.
\end{split}
\ee

From Lemma \ref{lemma:prog_sftmax} we obtain that, taking $g \geq 1$
and further simplifying, we get,

\begin{eqnarray}\label{eq:b1_tf}
T_1
\leq c_1T^2ge^{-c_2Tg^2}\left[R^2\gamma^2 + \nmm{A^*}_2\right],
\end{eqnarray}

\textbf{The Second Term} is defined via

\be
\begin{split}
T_2 := \sup_{\{W_j\} \in \F_{\mathcal{W}}, W' \in \F_{\theta}}\Big|&\IP{\nabla_{\hat{Y}}\ell\left(g \circ \P_{\C}, \Phi_r(\{W_j\}) \circ \P_{\C}\right)}{\phi(W') \circ \P_{\C}}_{\mu}\\
&- \IP{\nabla_{\hat{Y}}\ell\left(g, \Phi_r(\{W_j\})\right)}{\phi(W')}_{\mu}\Big|.
\end{split}
\ee

As a consequence of Lemma \ref{lemma:prog_sftmax} we have

\be\label{eq:b2_tf}
T_2 \leq c_1T^2ge^{-c_2Tg^2}\left[\nmm{A^*}_2 + R\gamma\right].
\ee

\textbf{The Third Term} is defined via 

\be
\begin{split}
T_3 := \sup_{\{W_j\} \in \F_{\mathcal{W}}}\Big|&\IP{\nabla_{\hat{Y}}\ell\left(g \circ \P_{\C}, \Phi_r(\{W_j\}) \circ \P_{\C}\right)}{\Phi_r(\{W_j\}) \circ \P_{\C}}_{\mu} \\ 
&- \IP{\nabla_{\hat{Y}}\ell\left(g, \Phi_r(\{W_j\})\right)}{\Phi_r(\{W_j\})}_{\mu}\Big|.
\end{split}
\ee

As a consequence of Lemma \ref{lemma:prog_sftmax} we have
\be\label{eq:b3_tf}
T_3 \leq c_1T^2ge^{-c_2Tg^2}R\gamma\left[\nmm{A^*}_2 + R\gamma\right],
\ee
for some positive constant, $c_1, c_2$.

Now combining $T_1, T_2$, and $T_3$ from equations
\eqref{eq:b1_tf}, \eqref{eq:b2_tf}, \eqref{eq:b3_tf} we obtain that
\be
B(\C)
\leq c_1T^2ge^{-c_2Tg^2}\left[\alpha(\nmm{A^*}_2 + R^2\gamma^2) + \nmm{A^*}_2 + R\gamma + R\gamma\left[\nmm{A^*}_2 + R\gamma\right]\right].
\ee
We further upper bound for simplicity as
\be\label{eq:b_tf}
B(\C)
\leq 4Rc_1\sqrt{\frac{\log(T)}{T^5}}ge^{-c_2g^2}\gamma\left[\nmm{A^*}_2 + \gamma\right].
\ee

From Theorem \ref{thm:gen_master} we have that 
\bd
\frac{1}{m}\left|\NC_{\mu}(\{W_j\}) - \NC_{\mu_N}(\{W_j\})\right| \lesssim
\frac{1}{2m}\nmF{A^*}\left[\sup_{\nmm{z} \leq 1}\frac{1}{N}\nmm{\sum_{i=1}^{N}(Y_i-\hat{Y}_i)^T(X_i\sigma_t(X^T\z))}_F-\l\right] 
\ed
\bd
+  \frac{2}{m}Rc_1T^2ge^{-c_2Tg^2}C_V\left[\nmF{A^*}\right]\left[\nmm{A^*}_2 + \gamma\right]
+ C_V^2\left[\nmF{A^*}\right] 
\ed
\be
\times \sqrt{\frac{R(m+n)
\log\left({C_V^2\left[\nmF{U^*}^2 + \nmF{V^*}^2\right]^2 k_1g^2RB_V/T}\right)\log(N)+ \log(1/\delta)}{N}}.
\ee
holds true w.p at least $1- \delta - 2N\exp\left(-cn_X(g-1)^2\right)$.

Now choose
\be
g = 1 + \tilde{\mathcal{O}}\left(\frac{1}{T}\log\left(\frac{N}{R(m+n) + log(1/\delta)}\right)\right).
\ee
Then we have
\be
\frac{2}{m}Rc_1T^2ge^{-c_2Tg^2}C_V\left[\nmF{A^*}\right]\left[\nmm{A^*}_2 + \gamma\right]
\lesssim \sqrt{\frac{R(m+n)
\log\left({C_V^2\left[\nmF{U^*}^2 + \nmF{V^*}^2\right]^2 k_1g^2RB_V/T}\right)\log(N)+ \log(1/\delta)}{N}}.
\ee

Therefore we can upper bound the middle term to the right most leaving us behind
\bd
\frac{1}{m}\left|\NC_{\mu}(\{W_j\}) - \NC_{\mu_N}(\{W_j\})\right| \lesssim
\frac{1}{2m}\nmF{A^*}\left[\frac{1}{N}\sum_{i=1}^{N}\nmm{Y_i - \hat{Y}_i}_2\nmm{X_i}_2-\l\right] 
\ed
\be
+ C_V^2\left[\nmF{A^*}\right]
\sqrt{\frac{R(m+n)
\log\left(R(m+n)(C_V + B_V)\right)\log(N)+ \log(1/\delta)}{N}},
\ee
holds true w.p at least $1- \delta$.
\end{proof}

\newpage
\section{GOODS EVENTS}\label{sec:good_events}

In this section, we provide compute the probabilities of events defined
in the proof of Theorem \ref{thm:gen_master}. Recall the definition of our function classes:
\begin{align}
\F_{\theta}
&:= \left\{\{W_j\}: \nmL{\Phi_{R}(\{W_j\})} \leq \gamma,
\Theta_R(\{W_j\}) \leq \gamma/L_{\phi}\right\}; \\
\F_{\mathcal{W}}
&:= \left\{\{W_j\}: \nmL{\Phi_{R}(\{W_j\})} \leq \gamma,
\Theta_R(\{W_j\}) \leq \gamma/L_{\phi}\right\}; \\
\F_{\Phi} &:= \left\{\Phi_R(\zeta): \forall \zeta \in \F_{\mathcal{W}}\right\}
\end{align}

We define the below events:
\begin{align}
\E_{cvx}(\epsilon) &:= \{\forall \zeta \in \F_{\mathcal{W}}: \left|\CP_{\mu_N}(f_{\zeta}) - \CP_{\mu}(f_{\zeta})\right| \leq \epsilon
+ B_{nrm}(\C)\}; \\
\E_{eql}(\epsilon) &:= \{\forall \zeta \in \F_{\mathcal{W}}: \left|\IP{\nabla_{\hat{Y}}\ell\left(g, f_{\zeta}\right)}{f_{\zeta}}_{\mu_N}
- \IP{\nabla_{\hat{Y}}\ell\left(g, f_{\zeta}\right)}{f_{\zeta}}_{\mu}\right| \leq \epsilon + B_{eql}(\C)\};\\
\E_{plr}(\epsilon) &:= \{\forall \zeta \in \F_{\mathcal{W}}: \left|\Omega_{\mu_N}^{\circ}\left(\nabla_{\hat{Y}}\ell\left(g, f_{\zeta}\right)\right)
- \Omega_{\mu}^{\circ}\left(\nabla_{\hat{Y}}\ell\left(g, f_{\zeta}\right)\right)\right| \leq \epsilon
+ B_{plr}(\C)\}; \\
\E_{nrm}(\epsilon) &:= \{\forall \zeta \in \F_{\mathcal{W}}: \left|\nmm{f_{\mu}^*-f_{\zeta}}_{\mu_N}^2 - \nmm{f_{\mu}^*-f_{\zeta}}_{\mu}^2\right| \leq \epsilon
+ B_{nrm}(\C)\}.
\end{align}

In each of the sections below, we discuss the technical analysis to estimate 
the probability of the events, $\E_{cvx}(\epsilon)$, $\E_{eql}(\epsilon)$, 
$\E_{plr}(\epsilon)$ and $\E_{nrm}(\epsilon)$.

\subsection{Concentration of Norms}

In this section, we upper bound the probability of the event, $\E_{nrm}(\epsilon)$
through Lemma \ref{lemma:cnc_nrms}.

\begin{lemma}[Concentration of Norms]\label{lemma:cnc_nrms}
Consider an $n_X$-dimensional sub-Gaussian vector $X \sim SG(0, (\sigma_X^2/n_X)I)$,
and set of functions $f_{\zeta}: \mathbb{R}^{n_X} \to \R$ as parameterized by $\zeta \in \F_{\mathcal{W}}$.
Let $\C$ be some convex obeying ${P}(\bigcap_{i=1}^{N}X_i \in \C) \geq 1 - \delta_{\C}$
for i.i.d samples $\{X_i\}_{i=1}^{N}$. Assume that for any fixed, $\zeta, \zeta' \in \F_{\mathcal{W}}$,
and fixed $Z \in \C$, we have
\be
\nmm{f_{\zeta}(Z) - f_{\zeta'}(Z)} \leq \tilde{L}_{\Phi}d(\zeta_1, \zeta_2)\text{ and }\nmm{f_{\zeta}(Z)} \leq B_{\Phi}.
\ee
Denote,
\be
B_{nrm}(\C) := \sup_{\zeta \in \F_{\mathcal{W}}}\left|\nmm{f_{\mu}^* \circ \P_{\C} - f_{\zeta} \circ \P_{\C}}_{\mu}^2
-\nmm{f_{\mu}^* - f_{\zeta}}_{\mu}^2\right|,
\ee
where $\P_{\C}(\cdot)$ denotes the Euclidean projection onto the set $\C$. Define,
\be
K := 64n_Y\gamma^2\sigma_X^2.
\ee
Then for any 
$\epsilon \in [0, K]$,
\bd
\mathbb{P}\left(\sup_{\zeta \in \F_{\mathcal{W}}}\left|\nmm{f_{\mu}^*-f_{\zeta}}_{\mu_N}^2 - \nmm{f_{\mu}^*-f_{\zeta}}_{\mu}^2\right| \geq \epsilon + B_{nrm}(\C)\right)
\ed
\be
\leq \delta_{\C} + c\exp\left(\log(C_{\F_{\mathcal{W}}}\left(\frac{\epsilon}{4\tilde{L}_{\Phi}B_{\Phi}}\right))-N\frac{\epsilon^2}{K^2}\right).
\ee
for some positive constant, $c$
and $C_{\F_{\mathcal{W}}}(\nu)$ is the $\nu$-net
covering number of the set $\F_{\mathcal{W}}$.
\end{lemma}

\begin{proof}
If $X \in \R^{n_X} \sim SG\left(\frac{\sigma_X^2}{n_X}I_{n_X \times n_X}\right)$, The function map,
$\nmm{f_{\mu}^* - f_{\zeta}}$ has Lipschitz constant of
$\nmL{f_{\mu}^*} + \nmL{f_{\zeta}} \leq 2\gamma$; as $f_{\zeta}, f_{\mu}^* \in \F_{\Phi}$. Therefore 
from Theorem 5.1.4 in \cite{vershynin_high-dimensional_2018} we have that,
$f_{\mu}^*(X) - f_{\zeta}(X) \sim SG\left(4\gamma^2\sigma_X^2I_{n_Y \times n_Y}\right)$.
Thus, $\nmm{f_{\mu}^*(X) - f_{\zeta}(X)}^2 \sim SE\left(4n_Y\gamma^2\sigma_X^2\right)$

Now, applying the concentration inequality for sub-exponential from Theorem 2.8.1 \cite{vershynin_high-dimensional_2018} 
for a fixed $\zeta \in \F_{\mathcal{W}}$, we have that
\begin{align}
\mathbb{P}\left(\left|\nmm{f_{\mu}^*-f_{\zeta}}_{\mu_N}^2 - \nmm{f_{\mu}^*-f_{\zeta}}_{\mu}^2\right| \geq \epsilon\right) \leq C\exp\left(-N\min\left\{\frac{\epsilon^2}{16n_Y^2\gamma^4\sigma_X^4}, \frac{\epsilon}{4n_Y\gamma^2 \sigma_X^2}\right\}\right),
\end{align}
for some positive constant, $C \geq 0$. We use Lemma \ref{lemma:conc_unf} for applying the
concentration bounds.
Now set
\begin{align}
g_{\theta} = \nmm{f_{\mu}^* - f_{\zeta}}^2.
\end{align}
We need to check if the function, $g$, is Lipschitz on some metric and convex set $\C \subseteq \R^{n_X}$,
choose for any $Z \in \C$. We have $P(\bigcap_{i=1}^{N}X_i \in \C) \geq 1 - \delta_{\C}$. Recall that
\begin{enumerate}
\item $\forall \zeta_1, \zeta_2 \in \F_{\mathcal{W}}: \nmm{f_{\zeta_1}(Z) - f_{\zeta_2}(Z)}
\leq \tilde{L}_{\Phi} d(\zeta_1, \zeta_2)$, for all $Z \in \C$.
\item $\forall \zeta \in \F_{\mathcal{W}}: \nmm{f_{\zeta}(Z)} \leq B_{\Phi}$, for all $Z \in \C$.
\item For a fixed $\zeta \in \F_{\mathcal{W}}$,
\be
\left|\mathbb{E}\left[\nmm{f_{\mu}^*(\P_{\C}(X)) - f_{\zeta}(\P_{\C}(X))}^2
-\nmm{f_{\mu}^*(X) - f_{\zeta}(X)}^2\right]\right| \leq B_{nrm}(\C).
\ee
\end{enumerate}

By exploiting the above items we have
\begin{eqnarray*}
\left|g_{\theta_1} - g_{\theta_2}\right|
&=& \left|\nmm{f_{\mu}^* - f_{\zeta_1}}^2 - \nmm{f_{\mu}^* - f_{\zeta_2}}^2\right|,
= \left|\IP{2f_{\mu}^* - (f_{\zeta_1} + f_{\zeta_2})}{f_{\zeta_1} - f_{\zeta_2}}\right|,\\
&\leq& \nmm{2f_{\mu}^* - (f_{\zeta_1} + f_{\zeta_2})}^{\circ}\nmm{f_{\zeta_1} - f_{\zeta_2}},\\
&\leq& 4\tilde{L}_{\Phi}B_{\Phi}d(\zeta_1, \zeta_2).
\end{eqnarray*}

Then we have that for covering number $C_{\F_{\mathcal{W}}}(\nu) = \N(\F_{\mathcal{W}}, d(., .), \nu)$,
and $K = 4\tilde{L}_{\Phi}B_{\Phi}$
\bd
\mathbb{P}\left(\sup_{\zeta \in \F_{\mathcal{W}}}\left|\nmm{f_{\mu}^*-f_{\zeta}}_{\mu_N}^2 - \nmm{f_{\mu}^*-f_{\zeta}}_{\mu}^2\right| \geq \epsilon + B_{nrm}\right)
\ed
\be
\leq \delta_{\C} + C\exp\left(\log(C_{\F_{\mathcal{W}}}\left(\frac{\epsilon}{4\tilde{L}_{\Phi}B_{\Phi}}\right))-N\min\left\{\frac{\epsilon^2}{256n_Y^2\gamma^4\sigma_X^4}, \frac{\epsilon}{64n_Y\gamma^2 \sigma_X^2}\right\}\right).
\ee
We conclude the result by choosing $\epsilon \in [0, 64n_Y\gamma^2\sigma_X^2]$.
\end{proof}

\subsection{Concentration of Convex functions}

In this section, we upper bound the probability of the event, $\E_{cvx}(\epsilon)$
through Lemma \ref{lemma:cnc_cvx_func}. In this, we consider 
strongly and smooth convex function (see assumption \ref{ass:a4}) through
Taylor expansion of the function is always bounded quadratically. Lemma \ref{lemma:cnc_nrms}
plays an important role in establishing Lemma \ref{lemma:cnc_cvx_func}.

\begin{lemma}[Concentration of Convex functions]\label{lemma:cnc_cvx_func}
Consider an $n_X$-dimensional sub-Gaussian vector $X \sim SG(0, (\sigma_X^2/n_X)I)$,
and set of functions $f_{\zeta}: \mathbb{R}^{n_X} \to \R$ as parameterized by $\zeta \in \F_{\mathcal{W}}$.
Let $\C$ be some convex obeying ${P}(\bigcap_{i=1}^{N}X_i \in \C) \geq 1 - \delta_{\C}$
for i.i.d samples $\{X_i\}_{i=1}^{N}$. Assume that for any fixed, $\zeta, \zeta' \in \F_{\mathcal{W}}$,
and fixed $Z \in \C$, we have
\be
\nmm{f_{\zeta}(Z) - f_{\zeta'}(Z)} \leq \tilde{L}_{\Phi}d(\zeta_1, \zeta_2)\text{ and }\nmm{f_{\zeta}(Z)} \leq B_{\Phi}.
\ee
Denote
\be
B_{nrm}(\C) := \sup_{\zeta \in \F_{\mathcal{W}}}\left|\nmm{f_{\mu}^* \circ \P_{\C} - f_{\zeta} \circ \P_{\C}}_{\mu}^2
-\nmm{f_{\mu}^* - f_{\zeta}}_{\mu}^2\right|.
\ee
where $\P_{\C}(\cdot)$ denotes the Euclidean projection onto the set $\C$. Define
\be
K := n_YL\left[(\gamma^2 + \nmL{g}^2)\sigma_X^2 + \nmL{g}^2\sigma_{Y|X}^2\right].
\ee
Then for any 
$\epsilon \in [0, K]$,
\begin{align}
\mathbb{P}\left(\sup_{{\zeta} \in \F_{\mathcal{W}}} \left|C_{\mu_N}(f_{\zeta}) - C_{\mu}(f_{\zeta})\right| \geq \epsilon + B_{nrm}(\C)\right)
\end{align}
\bd
\leq 
\delta_{\C} + 2\exp\left(\log\left(C_{\F_{\mathcal{W}}}\left(\frac{\epsilon}{2B_{\ell}\tilde{L}_{\Phi}}\right)\right)-cN\left(\frac{\epsilon}{K}\right)^2 \right),
\ed
for some positive constant, $c$
and $C_{\F_{\mathcal{W}}}(\nu)$ is the $\nu$-net
covering number of the set $\F_{\mathcal{W}}$.
\end{lemma}

\begin{proof}
Recall the definitions of the convex functions:
\begin{align}
C_{\mu_N}(f) := \ell(g, f)_{\mu_N} + \l \Omega(f)\text{, and }
C_{\mu}(f) := \ell(g, f)_{\mu} + \l \Omega(f).
\end{align}
The difference between these two terms is
\begin{align}
\left|C_{\mu_N}(f) - C_{\mu}(f)\right| = \left|\ell(g, f)_{\mu_N} - \ell(g, f)_{\mu}\right|.
\end{align}

From assumption \ref{ass:a4}, $\ell(., .)$
is second-order differentiable in the second argument. By 2nd-order Taylor's theorem, we have
\be
\ell(Y, \hat{Y}) = \ell(Y, \hat{Y}_0) + \IP{\nabla_{\hat{Y}}\ell(Y, \hat{Y}_0)}{\hat{Y} - \hat{Y}_0}
+ \IP{\int_{0}^{1}t\nabla_{\hat{Y}}^2\ell(Y, \hat{Y}_0 + t(\hat{Y} - \hat{Y}_0))dt}{(\hat{Y} - \hat{Y}_0)(\hat{Y} - \hat{Y}_0)^T}.
\ee

Now choose $Y = \hat{Y}_0 = g(X(\omega), E(\omega))$, and $\hat{Y} = f_{\zeta}(X(\omega))$. As 
$\ell(Y, Y) = 0$, and $\nabla_{\hat{Y}}\ell(Y, Y) = {\bf{0}}$. Plugging these parameters in the Taylor expansion
we have that (ignoring the inputs, $(X(\omega), E(\omega))$ for simplicity),
\begin{align}
\ell(g, f) 
= \IP{\int_{0}^{1}t\nabla_{\hat{Y}}^2\ell(g, g + t(f_{\zeta} - g))dt}{(f_{\zeta}-g)(f_{\zeta}-g)^T}.
\end{align}

Now we apply expectation over the measure $\mu_N$, and $\mu$ respectively on the above
equality. Then we have that
\begin{align}
\ell(g, f_{\zeta})_{\mu} &= \IP{\int_{0}^{1}t\nabla_{\hat{Y}}^2\ell(g, g + t(f_{\zeta}-g))dt}{(f_{\zeta} - g)(f_{\zeta}-g)^T}_{\mu};\\
\ell(g, f_{\zeta})_{\mu_N} &= \IP{\int_{0}^{1}t\nabla_{\hat{Y}}^2\ell(g, g + t(f_{\zeta}-g))dt}{(f_{\zeta} - g)(f_{\zeta}-g)^T}_{\mu_N}.
\end{align}
Since $f_{\zeta}$ and $g$ are Lipschitz functions and the inputs are sub-Gaussian, we have that $f_{\zeta}-g$ is a 
sub-Gaussian vector. As a consequence of Lemma 2.7.6 from \cite{vershynin_high-dimensional_2018} we obtain that
$(f_{\zeta}-g)(f_{\zeta}-g)^T$ follows a sub-exponential distribution, whose concentration is well-studied.

As a consequence of assumption \ref{ass:a4} the hessian is bounded, i.e,
$\alpha I \preceq \nabla_{\hat{Y}}^2\ell(., .) \preceq L I$.
We can argue that the product of a bounded RV and sub-exponential RV is sub-exponential.
Recall Item (iii) from Proposition 2.7.1 of \cite{vershynin_high-dimensional_2018}.
The random variable $Z$ is sub-exponential iff
\begin{align}
\mathbb{E}_{Z}\left[e^{\l |Z|}\right] \leq e^{\l K}; \forall \l \in [0, 1/K],
\end{align}
for some positive constant, $ K \geq 0$.

We now verify if $\IP{H(\x, \z)}{\x\x^T}$ is sub-exponential. 
Given that $\x \sim SG(\sigma_X^2/n_x I_{n_x \times n_x})$, $\z$ is a R.V.
Suppose $A \preceq H(\x, \z) \preceq B$ \textit{a.s}. Then we have that
\begin{eqnarray}
\mathbb{E}_{\x, \z}\left[e^{\l\left|\IP{H(\x, \z)}{\x\x^T}\right|}\right]
&\leq& \mathbb{E}_{\x, \z}\left[e^{\l\nmm{H(\x, \z)}_2\nmm{\x\x^T}_2}\right],\\
&\leq&  \mathbb{E}_{\x, \z}\left[e^{\l\max\{\rho(A), \rho(B)\}\nmm{\x}_2^2}\right],\\
\end{eqnarray}
where, $\rho(A)$ is the spectral radius of the matrix, $A$. Since, $\x \sim SG(\sigma_X^2/n_x I_{n_x \times n_x})$,
we have $\nmm{\x}^2 \sim SE(\sigma_X^2)$. Then,
\bd
\mathbb{E}_{\x, \z}\left[e^{\l\left|\IP{H(\x, \z)}{\x\x^T}\right|}\right]
\leq e^{\l\max\{\rho(A), \rho(B)\} \sigma_X^2},
\ed
implies that $\IP{H(\x, \z)}{\x\x^T} \sim SE(\max\{\rho(A), \rho(B)\}\sigma_X^2)$. From this analysis
we have
\bd
\IP{\int_{0}^{1}t\nabla_{\hat{Y}}^2\ell(g, g + t(f_{\zeta}-g))dt}{\underbrace{(f_{\zeta} - g)(f_{\zeta}-g)^T}_{\sim SE\left(\left[(\nmL{f_{\zeta}}^2 + \nmL{g}^2)\sigma_X^2 + \nmL{g}^2\sigma_{Y|X}^2\right]I_{n_y \times n_y}\right)}}
\ed
\be
\sim SE\left(n_Y\frac{L}{2}\left[(\nmL{f_{\zeta}}^2 + \nmL{g}^2)\sigma_X^2 + \nmL{g}^2\sigma_{Y|X}^2\right]\right).
\ee

For convex functions, we know that $ 0 \leq \alpha \leq L $.
As a consequence, we have 
$\frac{\alpha}{2}I \preceq \int_{0}^{1}t\nabla_{\hat{Y}}^2\ell(g, g + t(f_{\zeta}-g))dt \preceq \frac{L}{2}I$.
Now we apply sub-exponential concentration for a fixed ${\zeta} \in \F_{\mathcal{W}}$, yielding
\bd
\mathbb{P}\left(\left|C_{\mu_N}(f_{\zeta}) - C_{\mu}(f_{\zeta})\right| \geq \epsilon\right)
\ed
\bd
\leq 
2\exp\Big(-cN\min\Big\{\left(\frac{2\epsilon}{n_YL\left[(\gamma^2 + \nmL{g}^2)\sigma_X^2 + \nmL{g}^2\sigma_{Y|X}^2\right]}\right)^2, 
\ed
\be
\frac{2\epsilon}{n_YL\left[(\gamma^2 + \nmL{g}^2)\sigma_X^2 + \nmL{g}^2\sigma_{Y|X}^2\right]}\Big\}\Big),
\ee
for some positive constant, $c \geq 0$.

Next, we move on to obtain a uniform concentration for all $\zeta \in \F_{\mathcal{W}}$.
Now we apply covering argument from Lemma \ref{lemma:conc_unf}, and set
\begin{align}
g_{\theta} = \ell(g, f_{\zeta}).
\end{align}

We need to check if the function, $g$, is Lipschitz on some metric and convex set $\C \subseteq \R^{n_X}$,
choose for any $Z \in \C$. We have $P(\bigcap_{i=1}^{N}X_i \in \C) \geq 1 - \delta_{\C}$. Recall that
\begin{enumerate}
\item $\forall \zeta_1, \zeta_2 \in \F_{\mathcal{W}}: \nmm{f_{\zeta_1}(Z) - f_{\zeta_2}(Z)}
\leq \tilde{L}_{\Phi} d(\zeta_1, \zeta_2)$, for all $Z \in \C$.
\item $\forall \zeta \in \F_{\mathcal{W}}: \nmm{\nabla_{\hat{Y}}\ell(g(Z), f_{\zeta}(Z))} \leq B_{\ell}$, for all $Z \in \C$.
\item For a fixed $\zeta \in \F_{\mathcal{W}}$,
\be
\left|\mathbb{E}\left[\nmm{f_{\mu}^*(\P_{\C}(X)) - f_{\zeta}(\P_{\C}(X))}^2
-\nmm{f_{\mu}^*(X) - f_{\zeta}(X)}^2\right]\right| \leq B_{nrm}(\C).
\ee
\end{enumerate}

From Taylor expansion we have that,
\begin{eqnarray*}
|g_{\theta_1} - g_{\theta_2}|
&=& |\IP{\int_{t}\nabla_{\hat{Y}}\ell(g, f_{\zeta_1} + t (f_{\zeta_2} - f_{\zeta_1}))dt}{f_{\zeta_1}-f_{\zeta_2}}|,\\
&\leq& \nmm{\int_{t}\nabla_{\hat{Y}}\ell(g, f_{\zeta_1} + t (f_{\zeta_2} - f_{\zeta_1}))dt}\nmm{f_{\zeta_2} - f_{\zeta_1}},\\
&\leq& B_{\ell}\nmm{f_{\zeta_2} - f_{\zeta_1}},\\
&\leq& B_{\ell}\tilde{L}_{\Phi}d(\zeta_1, \zeta_2).
\end{eqnarray*}

From Lemma \ref{lemma:conc_unf} we have
\begin{align}
\mathbb{P}\left(\sup_{{\zeta} \in \F_{\mathcal{W}}} \left|C_{\mu_N}(f_{\zeta}) - C_{\mu}(f_{\zeta})\right| \geq \epsilon\right)
\end{align}
\bd
\leq 
\delta_{\C} + 2\exp\left(\log\left(C_{\F_{\mathcal{W}}}\left(\frac{\epsilon}{2B_{\ell}\tilde{L}_{\Phi}}\right)\right)-cN\min\left\{\left(\frac{\epsilon}{n_YL\left[(\gamma^2 + \nmL{g}^2)\sigma_X^2 + \nmL{g}^2\sigma_{Y|X}^2\right]}\right)^2, \right. \right.
\ed
\be
\left. \left. \frac{\epsilon}{n_YL\left[(\gamma^2 + \nmL{g}^2)\sigma_X^2 + \nmL{g}^2\sigma_{Y|X}^2\right]}\right\}\right),
\ee
for some positive constant, $c$. 
Now restrict $\epsilon \in \left[0, n_YL\left[(\gamma^2 + \nmL{g}^2)\sigma_X^2 + \nmL{g}^2\sigma_{Y|X}^2\right]\right]$.
This completes our proof.
\end{proof}

\subsection{Concentration of Equilibria}

In this section, we upper bound the probability of the event, $\E_{eql}(\epsilon)$
through Lemma \ref{lemma:cnc_eql}. We first present the concentration of bi-Lipschtiz functions
for sub-Gaussian inputs.

\begin{proposition}\label{prop:bisubg}
Let $X \sim SG(\frac{\sigma_X^2}{n_x}I_{n_x \times n_x})$ and $Y|X \sim SG(\frac{\sigma_{Y|X}^2}{n_y}I_{n_y \times n_y})$, where $\sigma_{Y|X}$ is independent of $X = x$, 
then for any Lipschitz function, $\phi: \mathcal{Z} \to \R$ and a function, $f: \X \times \Y \to \mathcal{Z}$,
satisfies the following. Then $\phi(f(X, Y)) \sim SG\left(4\nmL{\phi}^2\nmL{f}^2\left[\sigma_X^2 + \sigma_{Y|X}^2\right]\right)$.
\begin{align}
\nmm{f(X_2, Y_2) - f(X_1, Y_1)} \leq \nmL{f}\left[\nmm{X_2 - X_1} + \nmm{Y_2 - Y_1}\right].
\end{align}
\end{proposition}

\begin{proof}
Let us compute the moments of the random variable $\phi(f(X, Y))$ for any Lipschitz function, $\phi: \mathcal{Z} \to \R$.

By symmetrization, we have
\bd
\mathbb{E}_{X, Y}\left[\exp(\l[\phi(f(X, Y)) - \mathbb{E}_{X, Y}[\phi(f(X, Y))]])\right]
\ed
\be
= \mathbb{E}_{X, Y}\left[\exp(\l[\phi(f(X, Y)) - \mathbb{E}_{X', Y'}[\phi(f(X', Y'))])\right]
\ee
\be
=\mathbb{E}_{X, Y}\left[\exp(\l\mathbb{E}_{X', Y'}[\phi(f(X, Y)) - \phi(f(X', Y'))])\right].
\ee
By Jensen's inequality we have
\be
\leq\mathbb{E}_{X, Y, X', Y'}\left[\exp(\l[\phi(f(X, Y)) - \phi(f(X', Y'))])\right],
\ee
By Lipschitz continuity we have
\be
\leq\mathbb{E}_{X, Y, X', Y'}\left[\exp(\nmL{\phi}\l[\nmm{f(X, Y) - f(X', Y')}]\right].
\ee
By construction, we have.
\be
\leq\mathbb{E}_{X, Y, X', Y'}\left[\exp(\nmL{\phi}\nmL{f}\l[\nmm{X-X'} + \nmm{Y-Y'}]\right].
\ee
By Cauchy-Schwartz's inequality we have
\be
\leq\mathbb{E}_{X, Y, X', Y'}\left[\exp(\nmL{\phi}\nmL{f}\l[\nmm{X}+\nmm{X'} + \nmm{Y} +\nmm{Y'}]\right].
\ee
As the symmetrized random variables are independent,
\be
\leq\mathbb{E}_{X, Y}\left[\exp(2\nmL{\phi}\nmL{f}\l[\nmm{X}+\nmm{Y'}]\right].
\ee
Now perform conditional expectation,
\be
\leq\mathbb{E}_{X}\mathbb{E}_{Y|X}\left[\exp(2\nmL{\phi}\nmL{f}\l[\nmm{X}+\nmm{Y}]\right],
\ee
\be
\leq\mathbb{E}_{X}\exp(2\nmL{\phi}\nmL{f}\l[\nmm{X}])\mathbb{E}_{Y|X}\left[\exp(2\nmL{\phi}\nmL{f}\l[\nmm{Y}]\right],
\ee
\be
\leq \exp\left(\frac{\l^2}{2}4\nmL{\phi}^2\nmL{f}^2\left[\sigma_X^2 + \sigma_{Y|X}^2\right] + \l[\mathbb{E}_{X}[\nmm{X}] + \mathbb{E}_{Y|X}[\nmm{Y}]]\right),
\ee
\be
\leq K \exp\left(\frac{\l^2}{2}4\nmL{\phi}^2\nmL{f}^2\left[\sigma_X^2 + \sigma_{Y|X}^2\right]\right).
\ee
for some constant, $K \geq 0$.

This implies that, $\phi(f(X, Y)) \sim SG\left(4\nmL{\phi}^2\nmL{f}^2\left[\sigma_X^2 + \sigma_{Y|X}^2\right]\right)$.
\end{proof}

With the above result, we now state and prove the probability of the event, $\E_{eql}(\epsilon)$.

\begin{lemma}[Concentration of Equilibria]\label{lemma:cnc_eql}
Consider an $n_X$-dimensional sub-Gaussian vector $X \sim SG(0, (\sigma_X^2/n_X)I)$,
and set of functions $f_{\zeta}: \mathbb{R}^{n_X} \to \R$ as parameterized by $\zeta \in \F_{\mathcal{W}}$.
Let $\C$ be some convex obeying ${P}(\bigcap_{i=1}^{N}X_i \in \C) \geq 1 - \delta_{\C}$
for i.i.d samples $\{X_i\}_{i=1}^{N}$. Assume that for any fixed, $\zeta_1, \zeta_2 \in \F_{\mathcal{W}}$,
and fixed $Z \in \C$, we have
\be
\nmm{\nabla_{\hat{Y}}\ell(g(Z), f_{\zeta}(Z))} \leq B_{\ell}\text{,  }
\nmm{f_{\zeta}(Z)} \leq B_{\Phi}\text{, and}
\ee
\be
\nmm{f_{\zeta}(Z)-f_{\zeta'}(Z)} \leq \tilde{L}_{\Phi}d(\zeta, \zeta').
\ee
In addition, we have that,
\be
\sup_{\zeta \in \F_{\mathcal{W}}}\left|\mathbb{E}\left[\IP{\nabla_{\hat{Y}}\ell\left(g \circ \P_{\C}, f_{\zeta} \circ \P_{\C}\right)}{f_{\zeta'} \circ \P_{\C}}_{\mu}
- \IP{\nabla_{\hat{Y}}\ell\left(g, f_{\zeta}\right)}{f_{\zeta'}}_{\mu}\right]\right| = B_{eql}(\C).
\ee
Define 
\be
K := 4n_y\gamma\nmL{\nabla_{\hat{Y}}\ell}\sigma_X\sqrt{(\gamma^2 + \nmL{g}^2)\sigma_X^2 + \nmL{g}^2\sigma_{E|X}^2}.
\ee
Then for any $\epsilon \in \left[0, K\right]$,
\bd
\mathbb{P}\left(\sup_{\zeta \in \F_{\mathcal{W}}}\left|\IP{\nabla_{\hat{Y}}\ell(g, f_{\zeta})}{f_{\zeta}}_{\mu_N} - \IP{\nabla_{\hat{Y}}\ell(g, f_{\zeta})}{f_{\zeta}}_{\mu}\right| \geq \epsilon + B_{eql}(\C)\right) \leq
\ed
\begin{align}
\delta_{\C} + c\exp\left(\log\left(\C_{\F_{\mathcal{W}}}\left(\frac{\epsilon}{2\tilde{L}_{\Phi}\left[B_{\ell} + B_{\Phi}L\right]}\right)\right)-N\frac{\epsilon^2}{K^2}\right),
\end{align}
for some positive constant, $c$ and $C_{\F_{\mathcal{W}}}(\nu)$ is the $\nu$-net
covering number of the set $\F_{\mathcal{W}}$..
\end{lemma}

\begin{proof}
From Assumptions \ref{ass:a1}-\ref{ass:a5}, we have that for any $g_1, f_1, g_2, f_2 \in L^2(\mu)$
\begin{align}
\nmm{\nabla_{\hat{Y}}\ell(g_2(X(\omega), &E(\omega)), f_2(X(\omega))) - \nabla_{\hat{Y}}\ell(g_1(X(\omega), E(\omega)), f_1(X(\omega)))} \\
&\leq \nmL{\nabla_{\hat{Y}}\ell}\left[\nmm{g_2(X(\omega), E(\omega)) - g_1(X(\omega), E(\omega))} + \nmm{f_2(X(\omega)) - f_1(X(\omega))}\right].
\end{align}

Since $X(\omega)$ and $E(\omega)$ are Lipschitz concentrated R.Vs, it holds that
\be
g(X, E)|E \sim SG(\nmL{g}^2\sigma_X^2I_{n_y \times n_y}),
g(X, E)|X \sim SG(\nmL{g}^2\sigma_{E|X}^2I_{n_y \times n_y}),
\ee
\be
\text{and }f_{\zeta}(X) \sim SG(\nmL{f_{\zeta}}^2\sigma_X^2I_{n_y \times n_y}).
\ee
From Proposition \ref{prop:bisubg} we have
\bd
\nabla_{\hat{Y}}\ell(g(X(\omega), E(\omega)), f_{\zeta}(X(\omega))) \sim 
\ed
\be
SG\left(4\nmL{\nabla_{\hat{Y}}\ell}^2\left[(\nmL{f_{\zeta}}^2 + \nmL{g}^2)\sigma_X^2 + \nmL{g}^2\sigma_{E|X}^2\right]I_{n_y}\right).
\ee

Now we have the inner product between two sub-Gaussian random variables from Proposition \ref{prop:sg_se}
we have that the result is sub-exponential, i.e.,
\bd
\IP{\underbrace{\nabla_{\hat{Y}}\ell(g(X(\omega), E(\omega)), f_{\zeta}(X(\omega)))}_{\sim SG\left(4\nmL{\nabla_{\hat{Y}}\ell}^2\left[(\nmL{f_{\zeta}}^2 + \nmL{g}^2)\sigma_X^2 + \nmL{g}^2\sigma_{E|X}^2\right]I_{n_y \times n_y}\right)}}{\underbrace{f_{\zeta}(X(\omega))}_{\sim SG(\nmL{f_{\zeta}}^2\sigma_X^2I_{n_y})}}
\ed
\begin{align}
\sim SE\left(2n_y\nmL{\nabla_{\hat{Y}}\ell}\nmL{f_{\zeta}}\sigma_X\sqrt{(\nmL{f_{\zeta}}^2 + \nmL{g}^2)\sigma_X^2 + \nmL{g}^2\sigma_{E|X}^2}\right).
\end{align}

The class of functions, $f_{\zeta}$ for $\zeta \in \mathcal{F}_{\mathcal{W}}$, 
has bounded Lipschitz constant $\gamma$.
As a consequence of the sub-exponential concentration bound from Theorem 2.8.1 in \cite{vershynin_high-dimensional_2018}, we have that
for a fixed $\zeta \in \F_{\mathcal{W}}$, 
\begin{align}\label{eq:eqm_f}
\mathbb{P}\left(\left|\IP{\nabla_{\hat{Y}}\ell(g, f_{\zeta})}{f_{\zeta}}_{\mu_N} - \IP{\nabla_{\hat{Y}}\ell(g, f_{\zeta}}{f_{\zeta}}_{\mu}\right| \geq \epsilon\right) \leq C\exp\left(-N\min\left\{\frac{\epsilon^2}{K^2}, \frac{\epsilon}{K}\right\}\right).
\end{align}
where $K := 2n_y\gamma\nmL{\nabla_{\hat{Y}}\ell}\sigma_X\sqrt{(\gamma^2 + \nmL{g}^2)\sigma_X^2 + \nmL{g}^2\sigma_{E|X}^2}$
and some positive constant, $C$.

Now we move on to providing a uniform concentration in the inequality \eqref{eq:eqm_f}.
We will apply uniform concentration result from Lemma \ref{lemma:conc_unf}, for this set:
\begin{align}
g_{\theta} = \IP{\nabla_{\hat{Y}}\ell(g, f_{\zeta})}{f_{\zeta}}.
\end{align}

Recall the below items:
\begin{enumerate}
\item For a fixed $Z \in \C$ we have $\forall \zeta \in \F_{\mathcal{W}}:
\nmm{\nabla_{\hat{Y}}\ell(g(Z), f_{\zeta}(Z))} \leq B_{\ell}$.
\item For a fixed $Z \in \C$ we have $\forall \zeta \in \F_{\mathcal{W}}:
\nmm{f_{\zeta}(Z)} \leq B_{\Phi}$.
\item For a fixed $Z \in \C$ we have $\forall \zeta, \zeta' \in \F_{\mathcal{W}}:
\nmm{f_{\zeta}(Z)-f_{\zeta'}(Z)} \leq \tilde{L}_{\Phi}d(\zeta, \zeta')$.
\item For a any $\hat{Y}_1, \hat{Y}_2 \in \R^{n_Y}$ we have  $\nmm{\nabla_{\hat{Y}}\ell(Y, \hat{Y}_1)
-\nabla_{\hat{Y}}\ell(Y, \hat{Y}_2)} \leq L\nmm{\hat{Y}_1 - \hat{Y}_2}$.
\item For a fixed $\zeta \in \F_{\mathcal{W}}$,
\be
\sup_{\zeta \in \F_{\mathcal{W}}}\left|\mathbb{E}\left[\IP{\nabla_{\hat{Y}}\ell\left(g \circ \P_{\C}, f_{\zeta} \circ \P_{\C}\right)}{f_{\zeta} \circ \P_{\C}}_{\mu}
- \IP{\nabla_{\hat{Y}}\ell\left(g, f_{\zeta}\right)}{f_{\zeta}}_{\mu}\right]\right| = B_{eql}(\C).
\ee
\end{enumerate}

Now we check the Lipschitz continuity of the function 
$g_{\theta}$:
\begin{eqnarray*}
|g_{\theta_1} - g_{\theta_2}|
&=& \left|\IP{\nabla_{\hat{Y}}\ell(g, f_{\zeta_1})}{f_{\zeta_1}} - \IP{\nabla_{\hat{Y}}\ell(g, f_{\zeta_2})}{f_{\zeta_2}}\right|,\\
&=&  \left|\IP{\nabla_{\hat{Y}}\ell(g, f_{\zeta_1})}{f_{\zeta_1} - f_{\zeta_2}} - \IP{\nabla_{\hat{Y}}\ell(g, f_{\zeta_2})-\nabla_{\hat{Y}}\ell(g, f_{\zeta_1})}{f_{\zeta_2}}\right|,\\
&\leq&
\left|\IP{\nabla_{\hat{Y}}\ell(g, f_{\zeta_1})}{f_{\zeta_1} - f_{\zeta_2}}\right| + 
\left|\IP{\nabla_{\hat{Y}}\ell(g, f_{\zeta_2})-\nabla_{\hat{Y}}\ell(g, f_{\zeta_1})}{f_{\zeta_2}}\right|,\\
&\leq& \nmm{\nabla_{\hat{Y}}\ell(g, f_{\zeta_1})}
\nmm{f_{\zeta_1} - f_{\zeta_2}}
+ \nmm{\nabla_{\hat{Y}}\ell(g, f_{\zeta_2})-\nabla_{\hat{Y}}\ell(g, f_{\zeta_1})}\nmm{f_{\zeta_2}},\\
&\leq& B_{\ell}\nmm{f_{\zeta_1} - f_{\zeta_2})}
+ B_{\Phi}L\nmm{f_{\zeta_1} - f_{\zeta_2}},\\
&\leq& \tilde{L}_{\Phi}\left[B_{\ell} + B_{\Phi}L\right]d(\zeta_1, \zeta_2).
\end{eqnarray*}
Then from Lemma \ref{lemma:conc_unf} we have that
\begin{align}
\mathbb{P}&\left(\sup_{\zeta \in \F_{\mathcal{W}}}\left|\IP{\nabla_{\hat{Y}}\ell(g, f_{\zeta})}{f_{\zeta}}_{\mu_N} - \IP{\nabla_{\hat{Y}}\ell(g, f_{\zeta})}{f_{\zeta}}_{\mu}\right| \geq \epsilon + B_{eql}(\C)\right)  \\ &\leq
\delta_{\C} + C\exp\left(\log\left(\C_{\F_{\mathcal{W}}}\left(\frac{\epsilon}{2\tilde{L}_{\Phi}\left[B_{\ell} + B_{\Phi}L\right]}\right)\right)-N\min\left\{\frac{\epsilon^2}{4K^2}, \frac{\epsilon}{2K}\right\}\right).
\end{align}

\end{proof}

\subsection{Concentration of Polar}
In this section, we compute the probability of the occurrence of the event, $\E_{plr}(\epsilon)$ through Lemma \ref{lemma:cnc_plr}.
The analysis of $\E_{plr}(\epsilon)$ resembles to that of $\E_{eql}(\epsilon)$ following similar arguments.

\begin{lemma}[Concentration of Polar]\label{lemma:cnc_plr}
Consider an $n_X$-dimensional sub-Gaussian vector $X \sim SG(0, (\sigma_X^2/n_X)I)$,
and set of functions $f_{\zeta}: \mathbb{R^{n_X}} \to \R$ as parameterized by $\zeta \in \F_{\mathcal{W}}$.
Let $\C$ be some convex obeying ${P}(\bigcap_{i=1}^{N}X_i \in \C) \geq 1 - \delta_{\C}$
for i.i.d samples $\{X_i\}_{i=1}^{N}$. Assume that for any fixed, $\zeta_1, \zeta_2 \in \F_{\mathcal{W}}$,
$\zeta'_1, \zeta'_2 \in \F_{\theta}$, and fixed $Z \in \C$, we have
\be
\nmm{\nabla_{\hat{Y}}\ell(g(Z), f_{\zeta}(Z))} \leq B_{\ell}\text{,  }
\nmm{f_{\zeta}(Z)} \leq B_{\Phi}\text{, }
\ee
\be
\nmm{f_{\zeta_1}(Z)-f_{\zeta_2}(Z)} \leq \tilde{L}_{\Phi}d(\zeta_1, \zeta_2)\text{, and }
\nmm{f_{\zeta_1'}(Z)-f_{\zeta_2'}(Z)} \leq \tilde{L}_{\phi}d(\zeta_1', \zeta_2').
\ee
In addition, we have that,
\be
\sup_{\zeta \in \F_{\mathcal{W}}, \zeta' \in \F_{\theta}}\left|\mathbb{E}\left[\IP{\nabla_{\hat{Y}}\ell\left(g \circ \P_{\C}, f_{\zeta} \circ \P_{\C}\right)}{f_{\zeta'} \circ \P_{\C}}_{\mu}
- \IP{\nabla_{\hat{Y}}\ell\left(g, f_{\zeta}\right)}{f_{\zeta'}}_{\mu}\right]\right| = B_{plr}(\C).
\ee
Define 
\be
K := 4n_Y\nmL{\nabla_{\hat{Y}}\ell}L_{\phi}\sigma_X\sqrt{(\gamma^2 + \nmL{g}^2)\sigma_X^2 + \nmL{g}^2\sigma_{E|X}^2}.
\ee
Then for any $\epsilon \in \left[0, K\right]$,
\bd
\mathbb{P}\left(\sup_{\zeta \in \F_{\mathcal{W}}}: \left|\Omega_{\mu_N}^{\circ}\left(\nabla_{\hat{Y}}\ell\left(g, f_{\zeta}\right)\right)
- \Omega_{\mu}^{\circ}\left(\nabla_{\hat{Y}}\ell\left(g, f_{\zeta}\right)\right)\right| \geq \epsilon + B_{plr}(\C)\right) 
\ed
\be
\begin{split}
\leq 
\delta_{\C} + c\exp\Bigg(&\log\left(\C_{\F_{\mathcal{W}}}\left(\frac{\epsilon}{8\max\{\tilde{L}_{\phi}B_{\ell}, L\tilde{L}_{\Phi}B_{\Phi}\}}\right)\right) \\
& + \log\left(\C_{\F_{\theta}}\left(\frac{\epsilon}{8\max\{\tilde{L}_{\phi}B_{\ell}, L\tilde{L}_{\Phi}B_{\Phi}\}}\right)\right) -N\frac{\epsilon^2}{K^2}\Bigg),
\end{split}
\ee
for some positive constant, $c$ and $C_{\F_{\mathcal{W}}}(\nu)$ (and $C_{\F_{\theta}}(\nu)$) is the $\nu$-net
covering number of the set $\F_{\mathcal{W}} $(and $\F_{\theta}$).
\end{lemma}

\begin{proof}
Recall the definition of the polar in Equation \ref{def:polar_func}:
\begin{align}
\Omega_{\mu_N}^{\circ}\left(\nabla_{\hat{Y}}\ell\left(g, f_{\zeta}\right)\right) 
&:= \sup_{\zeta' \in \F_{\theta}}\IP{\nabla_{\hat{Y}}\ell\left(g, f_{\zeta}\right)}{f_{\zeta'}}_{\mu_N},\\
\Omega_{\mu}^{\circ}\left(\nabla_{\hat{Y}}\ell\left(g, f_{\zeta}\right)\right) 
&:= \sup_{\zeta' \in \F_{\theta}}\IP{\nabla_{\hat{Y}}\ell\left(g, f_{\zeta}\right)}{f_{\zeta'}}_{\mu}.
\end{align}
Now, by taking the difference between the above two polars, we have
\be
\begin{split}
\Big|\Omega_{\mu_N}^{\circ}\left(\nabla_{\hat{Y}}\ell\left(g, f_{\zeta}\right)\right)
&- \Omega_{\mu}^{\circ}\left(\nabla_{\hat{Y}}\ell\left(g, f_{\zeta}\right)\right)\Big| \\
&= \left|\sup_{\zeta' \in \F_{\theta}}\IP{\nabla_{\hat{Y}}\ell\left(g, f_{\zeta}\right)}{f_{\zeta'}}_{\mu_N}
-\sup_{\zeta' \in \F_{\theta}}\IP{\nabla_{\hat{Y}}\ell\left(g, f_{\zeta}\right)}{f_{\zeta'}}_{\mu}
\right|.
\end{split}
\ee

Denote, $\zeta_{\mu}'^* = arg\sup_{\zeta' \in \F_{\theta}}\IP{\nabla_{\hat{Y}}\ell\left(g, f_{\zeta}\right)}{f_{\zeta'}}_{\mu}$
and $\zeta_{\mu_N}'^* = arg\sup_{\zeta' \in \F_{\theta}}\IP{\nabla_{\hat{Y}}\ell\left(g, f_{\zeta}\right)}{f_{\zeta'}}_{\mu_N}$,
then, by definition we have that

\bd
-\IP{\nabla_{\hat{Y}}\ell\left(g, f_{\zeta}\right)}{f_{\zeta_{\mu}'^*}}_{\mu_N} + \IP{\nabla_{\hat{Y}}\ell\left(g, f_{\zeta}\right)}{f_{\zeta_{\mu}'^*}}_{\mu} 
\leq \Omega_{\mu_N}^{\circ}\left(\nabla_{\hat{Y}}\ell\left(g, f_{\zeta}\right)\right)
- \Omega_{\mu}^{\circ}\left(\nabla_{\hat{Y}}\ell\left(g, f_{\zeta}\right)\right) \leq
\ed
\be
\IP{\nabla_{\hat{Y}}\ell\left(g, f_{\zeta}\right)}{f_{\zeta_{\mu}'^*}}_{\mu_N} - \IP{\nabla_{\hat{Y}}\ell\left(g, f_{\zeta}\right)}{f_{\zeta_{\mu}'^*}}_{\mu}.
\ee
Applying modulus on both sides we obtain
\be\label{eq:polar_max}
\begin{split}
\Big|\Omega_{\mu_N}^{\circ}\Big(\nabla_{\hat{Y}}\ell\Big(g, f_{\zeta}\Big)\Big)
&- \Omega_{\mu}^{\circ}\Big(\nabla_{\hat{Y}}\ell\Big(g, f_{\zeta}\Big)\Big)\Big|\\
&\leq \max\Big\{\Big|\IP{\nabla_{\hat{Y}}\ell\Big(g, f_{\zeta}\Big)}{f_{\zeta_{\mu}'^*}}_{\mu_N} - \IP{\nabla_{\hat{Y}}\ell\Big(g, f_{\zeta}\Big)}{f_{\zeta_{\mu}'^*}}_{\mu} \Big|, \\ 
&\Big|\IP{\nabla_{\hat{Y}}\ell\Big(g, f_{\zeta}\Big)}{f_{\zeta_{\mu_N}'^*}}_{\mu_N}
- \IP{\nabla_{\hat{Y}}\ell\Big(g, f_{\zeta}\Big)}{f_{\zeta_{\mu_N}'^*}}_{\mu}\Big|\Big\}\\
&\leq \sup_{\zeta' \in \F_{\theta}}\Big|
\IP{\nabla_{\hat{Y}}\ell\Big(g, f_{\zeta}\Big)}{f_{\zeta'}}_{\mu_N} - \IP{\nabla_{\hat{Y}}\ell\Big(g, f_{\zeta}\Big)}{f_{\zeta'}}_{\mu}\Big|.
\end{split}
\ee

Now, we have to compute a lower bound on
\be\label{eq:prob_polar}
\mathbb{P}\left(\sup_{\zeta' \in \F_{\theta}}\left|
\IP{\nabla_{\hat{Y}}\ell\left(g, f_{\zeta}\right)}{f_{\zeta'}}_{\mu_N} - \IP{\nabla_{\hat{Y}}\ell\left(g, f_{\zeta}\right)}{f_{\zeta'}}_{\mu}\right| \leq \epsilon\right).
\ee

The computation of Equation \eqref{eq:prob_polar} is similar to that of Lemma \ref{lemma:cnc_eql}.
We can re-write the concentration of polars and apply the monotonicty of probability
in inequality \eqref{eq:polar_max} by doing this we have
\bd
\mathbb{P}\left(\sup_{ \zeta \in \F_{\mathcal{W}}}: \left|\Omega_{\mu_N}^{\circ}\left(\nabla_{\hat{Y}}\ell\left(g, f_{\zeta}\right)\right)
- \Omega_{\mu}^{\circ}\left(\nabla_{\hat{Y}}\ell\left(g, f_{\zeta}\right)\right)\right| \geq \epsilon\right)
\ed
\be
\leq 
\mathbb{P}\left(\sup_{\zeta \in \F_{\Phi}, \zeta' \in \F_{\theta}}: \left|
\IP{\nabla_{\hat{Y}}\ell\left(g, f_{\zeta}\right)}{f_{\zeta'}}_{\mu_N} - \IP{\nabla_{\hat{Y}}\ell\left(g, f_{\zeta}\right)}{f_{\zeta'}}_{\mu}\right| \leq \epsilon\right).
\ee

As the data follow the sub-Gaussian distribution
from Proposition \ref{lemma:bisubg} and Proposition \ref{prop:sg_se}, we have
\bd
\IP{\underbrace{\nabla_{\hat{Y}}\ell\left(g, f_{\zeta}\right)}_{\sim SG\left(4\nmL{\nabla_{\hat{Y}}\ell}^2\left[(\nmL{f_{\zeta}}^2 + \nmL{g}^2)\sigma_X^2 + \nmL{g}^2\sigma_{E|X}^2\right]I_{n_Y \times n_Y}\right)}}{\underbrace{f_{\zeta'}}_{\sim SG(\nmL{f_{\zeta'}}^2\sigma_X^2I_{n_Y \times n_Y})}}_{\mu_N}
\ed
\begin{align}
\sim SE\left(2n_Y\nmL{\nabla_{\hat{Y}}\ell}\nmL{f_{\zeta'}}\sigma_X\sqrt{(\nmL{f_{\zeta}}^2 + \nmL{g}^2)\sigma_X^2 + \nmL{g}^2\sigma_{E|X}^2}\right).
\end{align}

From Assumption \ref{ass:a7} the class, $\F_{\Phi}$ has Lipschitz constant at most, $\gamma$.
From the assumption \ref{ass:a3} $\F_{\theta}$ has a Lipschitz constant at most
$\gamma_{\theta}$. Therefore, the inner product described above is concentrated as a consequence
of Theorem 2.8.1 from \cite{vershynin_high-dimensional_2018}.
Now for a fixed $\zeta \in \F_{\mathcal{W}}, \zeta' \in \F_{\theta}$, we have that 
\begin{align}
\mathbb{P}\left(\left|
\IP{\nabla_{\hat{Y}}\ell\left(g, f_{\zeta}\right)}{f_{\zeta'}}_{\mu_N} - \IP{\nabla_{\hat{Y}}\ell\left(g, f_{\zeta}\right)}{f_{\zeta'}}_{\mu}\right| \leq \epsilon\right)
\leq C\exp\left(-N\min\left\{\frac{\epsilon^2}{K^2}, \frac{\epsilon}{K}\right\}\right).
\end{align}
where, $K = 2n_Y\nmL{\nabla_{\hat{Y}}\ell}L_{\phi}\sigma_X\sqrt{(\gamma^2 + \nmL{g}^2)\sigma_X^2 + \nmL{g}^2\sigma_{E|X}^2}$.

Now we utilize Lemma \ref{lemma:conc_unf} to have this concentration uniformly for all, $\zeta \in \F_{\mathcal{W}}, 
\zeta' \in \F_{\theta}$. Set
\begin{align}
g_{\theta} = \IP{\nabla_{\hat{Y}}\ell(g, f_{\zeta})}{f_{\zeta'}}.
\end{align}
Recall the below items:
\begin{enumerate}
\item For a fixed $Z \in \C$ we have $\forall \zeta \in \F_{\mathcal{W}}:
\nmm{\nabla_{\hat{Y}}\ell(g(Z), f_{\zeta}(Z))} \leq B_{\ell}$.
\item For a fixed $Z \in \C$ we have $\forall \zeta \in \F_{\mathcal{W}}:
\nmm{f_{\zeta}(Z)} \leq B_{\Phi}$.
\item For a fixed $Z \in \C$ we have $\forall \zeta, \zeta' \in \F_{\mathcal{W}}:
\nmm{f_{\zeta}(Z)-f_{\zeta'}(Z)} \leq \tilde{L}_{\Phi}d(\zeta, \zeta')$.
\item For a fixed $Z \in \C$ we have $\forall \zeta, \zeta' \in \F_{\theta}:
\nmm{f_{\zeta}(Z)-f_{\zeta'}(Z)} \leq \tilde{L}_{\phi}d(\zeta, \zeta')$.
\item For a any $\hat{Y}_1, \hat{Y}_2 \in \R^{n_Y}$ we have  $\nmm{\nabla_{\hat{Y}}\ell(Y, \hat{Y}_1)
-\nabla_{\hat{Y}}\ell(Y, \hat{Y}_2)} \leq L\nmm{\hat{Y}_1 - \hat{Y}_2}$.
\item For a fixed $\zeta \in \F_{\mathcal{W}}, \zeta' \in \F_{\theta}$,
\be
\sup_{\zeta \in \F_{\mathcal{W}}, \zeta' \in \F_{\theta}}\left|\mathbb{E}\left[\IP{\nabla_{\hat{Y}}\ell\left(g \circ \P_{\C}, f_{\zeta} \circ \P_{\C}\right)}{f_{\zeta'} \circ \P_{\C}}_{\mu}
- \IP{\nabla_{\hat{Y}}\ell\left(g, f_{\zeta}\right)}{f_{\zeta'}}_{\mu}\right]\right| = B_{plr}(\C).
\ee
\end{enumerate}

Now we check the Lipschitzness of $g$:
\begin{eqnarray*}
|g_{\theta_1} - g_{\theta_2}| &=& 
\left|\IP{\nabla_{\hat{Y}}\ell(g, f_{\zeta_1})}{f_{\zeta'_1}} - \IP{\nabla_{\hat{Y}}\ell(g, f_{\zeta_2})}{f_{\zeta'_2}}\right|\\
&=& \left|\IP{\nabla_{\hat{Y}}\ell(g, f_{\zeta_1})}{f_{\zeta'_1}-f_{\zeta'_2}} - \IP{\nabla_{\hat{Y}}\ell(g, f_{\zeta_2})-\nabla_{\hat{Y}}\ell(g, f_{\zeta_1})}{f_{\zeta'_2}}\right|,\\
&\leq& \left|\IP{\nabla_{\hat{Y}}\ell(g, f_{\zeta_1})}{f_{\zeta'_1}-f_{\zeta'_2}}\right| +
\left|\IP{\nabla_{\hat{Y}}\ell(g, f_{\zeta_2})-\nabla_{\hat{Y}}\ell(g, f_{\zeta_1})}{f_{\zeta'_2}}\right|,\\
&\leq& \nmm{\nabla_{\hat{Y}}\ell(g, f_{\zeta_1})}
\nmm{f_{\zeta'_1}-f_{\zeta'_2}},
+ \nmm{\nabla_{\hat{Y}}\ell(g, f_{\zeta_2})-\nabla_{\hat{Y}}\ell(g, f_{\zeta_1})}\nmm{f_{\zeta'_2}},\\
&\leq& B_{\ell}\tilde{L}_{\phi}d(\zeta'_1, \zeta'_2) + B_{\Phi}L\tilde{L}_{\Phi}d(\zeta_1, \zeta_2),\\
&\leq& 2\max\{\tilde{L}_{\phi}B_{\ell}, L\tilde{L}_{\Phi}B_{\Phi}\}\max\{d(\zeta'_1, \zeta'_2), d(\zeta_1, \zeta_2)\}.
\end{eqnarray*}

Now, we have a product of two metric spaces whose metric
is a maximum of individual metrics, therefore simply
we can upper bound the covering number by product
of these two metric spaces.,i.e.,
\be
\N(\F_{\mathcal{W}} \times \F_{\theta}, \nmm{\cdot}_{\infty, d(., .)}, \nu)
\leq \N(\F_{\mathcal{W}}, d(., .), \nu)\N(\F_{\theta}, d(., .), \nu).
\ee
From Lemma \ref{lemma:conc_unf} we have that
\begin{align}
\mathbb{P}\left(\sup_{\zeta \in \F_{\mathcal{W}}, \zeta' \in \F_{\theta}}: \left|
\IP{\nabla_{\hat{Y}}\ell\left(g, f_{\zeta}\right)}{f_{\zeta'}}_{\mu_N} - \IP{\nabla_{\hat{Y}}\ell\left(g, f_{\zeta}\right)}{f_{\zeta'}}_{\mu}\right| \geq \epsilon + B_{plr}(\C)\right)
\end{align}
\be
\begin{split}
\leq 
\delta_{\C} + C\exp\Big(&\log\left(\C_{\F_{\mathcal{W}}}\left(\frac{\epsilon}{8\tilde{L}_{\Phi}\max\{B_{\ell}, LB_{\Phi}\}}\right)\right) \\
&+ \log\left(\C_{\F_{\theta}}\left(\frac{\epsilon}{8\tilde{L}_{\Phi}\max\{B_{\ell}, LB_{\Phi}\}}\right)\right) -N\min\left\{\frac{\epsilon^2}{4K^2}, \frac{\epsilon}{2K}\right\}\Big).
\end{split}
\ee
This completes our result.
\end{proof}

\section{NUMERICAL EXPERIMENTS}\label{sec:apdx_ms}

\bfig[ht!] \label{fig:simulations}
% \hspace*{-2cm}
\centering
\includegraphics[width=\textwidth]{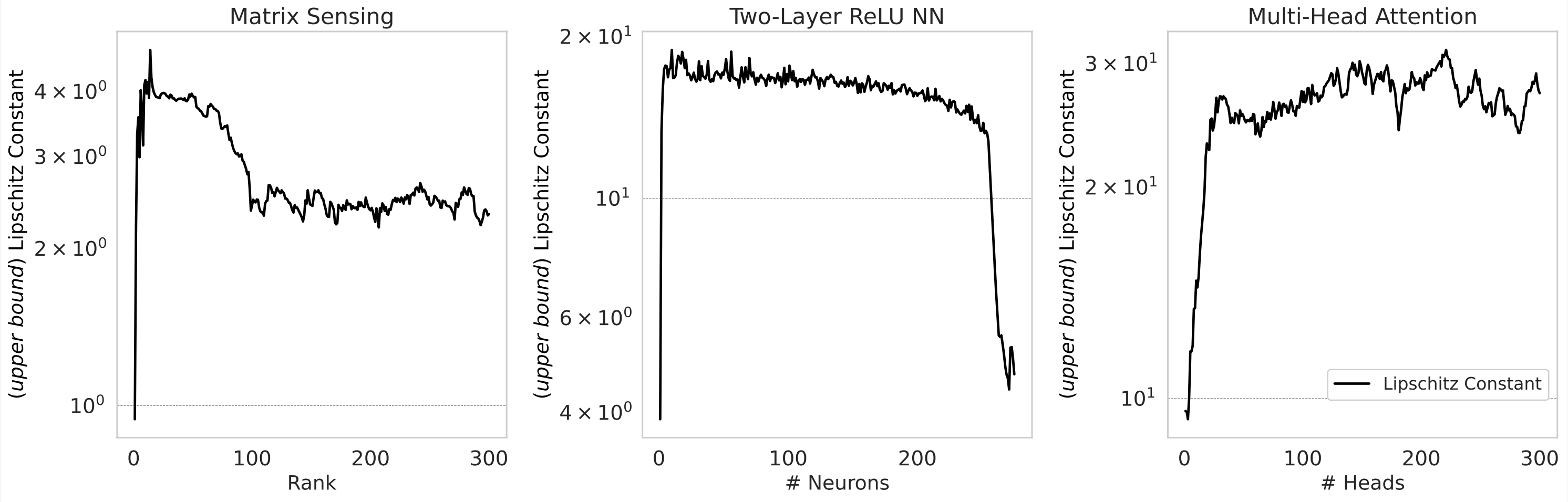}
\caption{Numerical simulations of the Lipschitz constant (or upper bound thereof) obtained for different model widths $(r)$.}
\efig

In this section, we present numerical simulations for the problems of low-rank matrix sensing,
two-layer ReLU neural networks, and multi-head attention. In each simulation shown in 
Figure \ref{fig:simulations}, we generated data using a teacher model with random initialization
of parameters, $Y = \Phi_{r^*}(\{W_j\})(X) + \epsilon$, where $r^* = 64$, 
$X \sim \N(0, (\sigma_X^2/n_X)I)$, and $\epsilon \in \N(0, (\sigma_E^2/n_Y)I)$. We used gradient descent to
reach a stationary point for each \( R \) (rank, number of neurons, or number of heads), starting
from 1 and increasing up to 300. The first factor was initialized with small-scale random values.
For each subsequent factor, we initialized the new factor with the supremum obtained from the 
polar equation \eqref{eq:polar}, following the algorithm in \cite{haeffele-vidal-arxiv15,haeffele-vidal-tpami20}. 

In each problem shown in Figure \ref{fig:simulations}, we plot the upper bounds on the Lipschitz constant for these problems. For matrix sensing, the Lipschitz constant is 
trivially upper-bounded by $\nmm{\mathbf{UV}^T}_2$; for the ReLU neural network, it is upper-bounded 
by $\nmm{\mathbf{U}\|_2 \|\mathbf{V}}_2$; and for multi-head attention, it is upper-bounded 
by $\sum_{j=1}^{r} \nmm{\mathbf{V_j}}_2$. We can observe from Figure \ref{fig:simulations}
upper bounds on the Lipschitz constants are uniformly bounded, indicating that our 
Assumption \ref{ass:a7} is realistic and holds empirically. 

We conjecture that it is possible to show that the Lipschitz constants are uniformly bounded
for any stationary. However, the analysis of this is beyond the scope of this work. 
Similar analyses based on gradient descent can be found in \cite{oymak-mahdi-icml19}.

\section{OTHER RELATED WORKS}

In this section, we provide a comprehensive study of the related works of the applications
that are of the concern in this work.

\textbf{Statistical Learning Theory (SLT):}  SLT provides a theoretical framework for analyzing
generalization error, often producing results of the form
\eqref{eq:slt}. The seminal work by \cite{vapnik-nature00} established a systematic approach to deriving bounds
of this nature. Over time, various approaches in SLT have attempted to estimate $\epsilon(\F, N, \delta)$, as
summarized in Table \ref{tab:slt_framework}. A recurring challenge in these bounds is the need to quantify the
“capacity” of the model’s hypothesis class, which is particularly difficult for DNNs.

While there have been attempts to estimate the VC-dimension, such as those based on the norm of the parameters
\citep{neyshabur-et-al-nips17}, the resulting bounds heavily depend on the norm of the parameters. Consequently,
it remains unclear how to accurately estimate the sample complexity of models when varying the depth or width of
DNNs. More recent work, such as \cite{imaizumi-schmidt-hieber-tit23}, presents bounds that are tight but still 
dependent on the norm of the weights, assuming that the SGD iterates converge to a specific class of parameters.

Another line of research by \cite{muthukumar-sulam-colt23} explores bounds that leverage the sparsity of feed-
forward neural networks. However, there is still a lack of data-dependent bounds that do not rely on capacity
estimates for models trained on random labels.

\be\label{eq:slt}
\mathbb{P}\left(\sup_{f \in \F}\left|\mathbb{E}_{X, Y}\left[\ell(Y, f(X))\right] - \frac{1}{N}\sum_{i=1}^{N}\ell(Y_i, f(X_i))\right|
\leq \epsilon(\F, \delta, N)\right) \geq 1 - \delta
\ee

\begin{table}[h!]
\centering
\begin{tabular}{|c|c|c|}
\hline
Description & $\epsilon(\F, N, \delta)$\\
\hline
\hline
Vapnik-Chernoviks Dimension, \citep{vapnik-nature00} & $\sqrt{\frac{\textrm{VCdim}(\F) - \log(\delta)}{N}}$\\
\hline
Rademacher Complexity, \citep{bartlett-medelson-jmlr01} & $R_N(\F) + \sqrt{\frac{-\log(\delta)}{N}}$\\
\hline
PAC-Bayes Bounds, \citep{mcallester-colt99} & $\frac{KL(Q||P) - \log(\delta)}{N}$\\
\hline
Gaussian Complexity, \citep{bartlett-medelson-jmlr01} & $G_N(\F) + \sqrt{\frac{-\log(\delta)}{N}}$\\
\hline
Information-theoretic Bounds, \citep{he-et-al-arxi24} & $\frac{1}{N}\sum_{i=1}^{N}\sqrt{I(W;(X_i, Y_i))}$\\
\hline
Algorithmic Stability, \citep{feldman-vonrdrak-arxiv19} & $\beta + \sqrt{\frac{-\log(\delta)}{N}}$\\
\hline
\end{tabular}
\caption{SLT frameworks (in chronological order)}
\label{tab:slt_framework}
\end{table}

\textbf{Matrix recovery:} This is a fundamental problem in signal processing, where we seek to recover 
a matrix by indirect measurements, like random measurements, and random entry access. We typically have limited measurements;
the problem itself is ill-posed when reconstructing the matrix. However, if the underlying matrix has certain special
structures like low-rankedness, or sparsity in entries, the problem becomes tractable so as to reconstruct the true matrix.
In practice, the problem tends to have low-rankedness, therefore having immense literature in this area, our work
also presents such results, considering the optimization. 

Let, $Y_i = \IP{M^*}{X_i} + \epsilon \in \R$, where, $X_i \in \R^{m \times n} (m \geq n)$ is Gaussian
entried matrix, $\epsilon \sim \N(0, \sigma^2)$ and $M^* \in \R^{m \times n}$ is a $r^*$-rank matrix. Consider the below
problem

\begin{table*}[h]
\centering
\renewcommand{\arraystretch}{1.5} % Increase row height
\setlength{\tabcolsep}{8pt} % Increase column spacing
\begin{tabular}{c|c}
% \hline
\begin{minipage}{0.45\textwidth}
\centering
\small % Adjust the font size
\begin{equation}\label{eq:opt_MS_r}
\begin{aligned}
\min_{M \in \R^{m \times n}} \quad \text{rank}(M)\\
\text{s.t. } \quad \|Y_i - \IP{M}{X_i} \| \leq \delta
\end{aligned}
\end{equation}
\end{minipage}
&
\begin{minipage}{0.45\textwidth}
\centering
\small % Adjust the font size
\begin{equation}\label{eq:opt_MS_nuc}
\begin{aligned}
\min_{M \in \R^{m \times n}} \quad \|M\|_*\\
\text{s.t. } \quad \|Y_i - \IP{M}{X_i} \| \leq \delta
\end{aligned}
\end{equation}
\end{minipage}
\\
\hline
\begin{minipage}{0.45\textwidth}
\centering
\small % Adjust the font size
\begin{equation}\label{eq:opt_MS_bm}
\begin{aligned}
\min_{r \in \mathbb{N}, U \in \R^{m \times r}, V \in \R^{n \times r}} \quad \|UV^T\|_*\\
\text{s.t. } \quad \|Y_i - \IP{UV^T}{X_i} \| \leq \delta
\end{aligned}
\end{equation}
\end{minipage}
&
\begin{minipage}{0.45\textwidth}
\centering
\small % Adjust the font size
\begin{equation}\label{eq:opt_MS_bm_eq}
\begin{aligned}
\min_{r \in \mathbb{N}, U \in \R^{m \times r}, V \in \R^{n \times r}} \quad \frac{1}{2}\left(\|U\|_F^2 + \|V\|_F^2\right)\\
\text{s.t. } \quad \|Y_i - \IP{UV^T}{X_i} \| \leq \delta
\end{aligned}
\end{equation}
\end{minipage}
\\
% \hline
\end{tabular}
\caption{Optimization problems for matrix sensing}
\label{tab:ms_opt}
\end{table*}

\be\label{eq:opt_MS}
\min_{r \in \mathbb{N}, U \in \R^{m \times r}, V \in \R^{n \times r}} \nmm{Y_i - \IP{UV^T}{X_i}}^2 + \frac{\l}{2}\left[\nmF{U}^2 + \nmF{V}^2\right]
\ee

The optimization problem in \eqref{eq:opt_MS_r} is non-convex due to its rank-minimization nature, an NP-HARD problem. 
However, under certain specific conditions on the measurement matrices $X_i$, the convex relaxation \eqref{eq:opt_MS_nuc}
can recover solutions to \eqref{eq:opt_MS_r}, as demonstrated in \cite{recht-et-al-cdc08}. Solving the convex program 
\eqref{eq:opt_MS_nuc} requires computing the Singular Value Decomposition (SVD), which has a computational complexity of
$\mathcal{O}(mn^2)$.

To mitigate this computational burden, the Burer-Monteiro (BM) factorization \citep{burer-monterio-mp03} is employed, yielding
the bilinear factorization in the non-convex program \eqref{eq:opt_MS_bm}. This approach is more efficient than \eqref{eq:opt_MS_nuc} 
because it introduces an implicit rank constraint, $rank(UV^T) \leq \min(n, r)$, which reduces the runtime of SVD to 
$\mathcal{O}((m+n)r^2)$. Additionally, the equivalence between the nuclear norm and the sum of Frobenius norms, as shown 
by \cite{paris-et-al-nips20}, further accelerates the optimization process, reducing the complexity to $\mathcal{O}((m+n)r)$.

While the BM factorization program \eqref{eq:opt_MS_bm} is non-convex, in contrast to the convex program \eqref{eq:opt_MS_nuc},
gradient descent (GD) algorithms typically guarantee only local minima for non-convex optimization problems \citep{reddy-arxiv23}.
However, \cite{ge-et-al-nips17} has proven that the program \eqref{eq:opt_MS_bm} has no spurious local minima, and any local 
minimum is indeed a global minimum. Numerous studies \citep{jia-et-al-nips23} have explored the optimization landscapes and the
convergence to global minima.

Our work primarily focuses on the generalization capabilities of the BM factorization program \eqref{eq:opt_MS}, which 
represents the Lagrangian form of the program \eqref{eq:opt_MS_bm}.  Table \ref{tab:lr_ms} summarizes
the results from the literature that provide matrix recovery guarantees; from this we 
can suggest there are no bounds in the literature for low-rank matrix recovery with nuclear norm regularization
under noisy settings with generic parameterization. Our work presents results first of its kind.

\begin{table*}[h]
\label{tab:lr_ms}
\centering
\renewcommand{\arraystretch}{1.5} % Increase row height
\setlength{\tabcolsep}{6pt} % Adjust column spacing for better fit
% \hspace*{-3cm}
% \begin{tabular}{|p{2cm}|c|c|p{8cm}|}
\resizebox{\textwidth}{!}{
\begin{tabular}{|c|c|c|c|}
\hline
\textbf{Measurement Type} & \textbf{Scenario} & \textbf{Reference} & \textbf{Result} \\
\hline
%\multirow{\textbf{Exact}} 
\textbf{Exact}
& Under-Parameterized ($r < r^*$) & N/A & N/A \\
\cline{2-4}
& Exactly-Parameterized ($r = r^*$) & N/A & Not directly available. \\
\cline{2-4}
& Over-Parameterized ($r > r^*$) & \citep{stoger-soltanolkotabi-arxiv22} & $\nmm{UU^T - M^*}_F \lesssim {r^*}^{1/8}(r-r^*)^{3/8}$ when $r \in (r^*, 2r^*)$. \\
\cline{2-4}
& Generic Parameterization ($r \geq 1$) & \citep{jin-et-al-icml23} & GD learns rank incrementally, $\nmm{M^* - UU^T}_F \lesssim \alpha^{\frac{1}{C_2\k_*^2}}$, but analysis is algorithmic. \\
\cline{2-4}
& SDP Relaxation (\textit{Full SDP Matrix}) & N/A & Not directly available. \\
\hline
%\multirow{\textbf{Noisy}} 
\textbf{Noisy}
& Under-Parameterized ($r < r^*$) & N/A & N/A \\
\cline{2-4}
& Exactly-Parameterized ($r = r^*$) & \citep{ma-et-al-fcm20} & $\nmm{M^* - UU^T}_F \lesssim \sqrt{\frac{\log(m)}{N}}$ under RIP assumptions $\delta_{4r^*} \leq 0.1$. \\
\cline{3-4}
& & \citep{negahban-wainwright-arxiv09} & $\nmm{\hat{M} - M}_F \lesssim \sqrt{r^*\frac{m + n}{N}}$. \\
\cline{2-4}
& Over-Parameterized ($r > r^*$) & \citep{ma-et-al-fcm20} & $\nmm{M^* - UU^T}_F \lesssim \sqrt{\delta_{r+r^*}\nmm{M^*}_2}$. \\
\cline{2-4}
& Generic Parameterization ($r \geq 1$) & N/A & N/A \\
\cline{2-4}
& SDP Relaxation (\textit{Full SDP Matrix}) & \citep{candes-yaniv-arxiv10} & $\nmm{\hat{M} - M^*}_F \lesssim \sqrt{nr^*/N}$ under RIP assumptions. \\
\cline{3-4}
& & \citep{koltchinskii-et-al-annals11} & $\nmm{\hat{M} - M}_F \lesssim \frac{mnr^*\log(N)}{N}$ under uniform noisy measurements. \\
\hline
\end{tabular}
}
% \hspace*{-3cm}
\caption{Summary of Related Works on Matrix Recovery. N/A is an acronym for "Not Available".}
\label{tab:related_works}
\end{table*}

\textbf{Transformers:} The remarkable success of Large Language Models (LLMs) \citep{gemini} can 
largely be attributed
to their foundational architecture—Transformers \citep{vaswani-et-al-arxiv23}. The optimization dynamics of 
Transformers have been a subject of extensive recent research \citep{bordlen-arxiv24}, \citep{singh-arxiv23},
\citep{yang-arxiv22}, \citep{tian-arxiv23}, \citep{nichani-arxiv24}. Although Transformers exhibit impressive 
generalization capabilities in practical applications \citep{zhou-arxiv2024}, there is still a significant 
gap in the theoretical analysis of their generalization error.

To apply classical SLT bounds, one must determine the capacities of the 
function classes induced by Transformers. Previous attempts, such as in \citep{edelman-arxiv2022}, have 
made progress but were limited to scenarios where input data is bounded. In contrast, our work extends 
these results to settings where the inputs are not necessarily bounded.

Another line of research \citep{li-arxiv23b}, \citep{deora-arxiv2023} has provided bounds that depend on
step sizes and initialization choices for Gradient Descent (GD). For instance, \cite{li-arxiv23b} offered 
bounds within the context of in-context learning \citep{zhang-arxiv23}, yet without evaluating the
capacities of the stable algorithms used to train these Transformers.

In the broader literature, existing studies on generalization bounds often rely on strong assumptions, 
such as (i) bounded input data, (ii) algorithmic stability in some defined sense, and (iii) Lipschitz 
continuity of the loss function (which does not hold globally for mean squared error). Our results address
these limitations by providing near-tight sample complexity bounds, offering a more comprehensive 
understanding of generalization in Transformer models.

\newpage
\section{PRELIMINARIES}

This section provides preliminaries of convex analysis and concentration of measure.

\subsection{Convex Functions}

\begin{definition}[$L^2$ functions]\label{def:l2}
A function $f : \X \to \Y$ is said to be square integrable on measure $\mu$, i.e., $L^2(\mu)$ if and only if,
\begin{align}
\IP{f}{f}_{\mu} = \int_{x \in \X}\IP{f(x)}{f(x)}_{\Y}d\mu(x) < \infty.
\end{align}
\end{definition}

\begin{definition}[Convex Set, \citep{rockafellar_convex_1970}]\label{def:cvx_set}
A set $\C$ is said to be convex if and only if $\forall f, g \in \C$, $\alpha  f + (1-\alpha) g \in \C; \forall \alpha \in [0, 1]$.
\end{definition}

\begin{definition}[Convex functions, \citep{rockafellar_convex_1970}]\label{def:cvx_func}
A function, $\Omega$ is said to be convex if and only if $dom(\Omega)$ is convex and $\forall f, g \in dom(\Omega)$ and 
any $\alpha \in [0, 1]$.
\begin{align}
\Omega(\alpha f + (1-\alpha) g) \leq \alpha \Omega (f) + (1-\alpha) \Omega(g).
\end{align}
\end{definition}

\begin{definition}[Gauge function, \citep{rockafellar_convex_1970}]\label{def:guage}
The gauge function or the Minkowski functional is defined in a set $\C \in L^2(\mu)$ for a point $f$ as follows,
\begin{align}
\sigma_{\C}(f) := \inf\left\{t \geq 0;\text{ such that }f \in t\textrm{conv}(\C)\right\}.
\end{align}
\end{definition}

\begin{definition}[Polar Set, \citep{rockafellar_convex_1970}]\label{def:polar_set}
The polar set of any set $\C \subseteq L^2(\mu)$ is given be
\begin{align}
\C^{\circ} := \left\{g \in L^2(\mu):\text{ such that }\IP{g}{f}_{\mu} \leq 1; \forall f \in \C\right\}.
\end{align}
\end{definition}

\begin{proposition}[Polar Properties]\label{prop:polar_prop}
\end{proposition}

\begin{definition}[Polar function, \citep{rockafellar_convex_1970}]\label{def:polar_func}
The polar function of any gauge function, $\sigma$ defined in the set $\C \subseteq L^2(\mu)$ is given be
\begin{align}
\sigma^{\circ}_{\C}(g) := \sigma_{\C^{\circ}}(g).
\end{align}
\end{definition}

\begin{definition}[Fenchal dual, \citep{rockafellar_convex_1970}]\label{def:dual}
The fenchal-dual for any $\mu$-measurable function, $\Omega$ evaluated at $g \in L^2(\mu)$ is defined by,
\begin{align}
\Omega^*(g) := \sup_{f \in L^2(\mu)} \IP{g}{f}_{\mu} - \Omega(f).
\end{align}
\end{definition}

\begin{lemma}[First Convexity, \citep{rockafellar_convex_1970}]\label{lemma:fo_cvx}
Any function $\Omega$ that is first-order differentiable, $\Omega \in \C^1$ is convex if and only if for any $f, g \in dom(\Omega)$
\begin{align}
\Omega(f) \geq \Omega(g) + \IP{\nabla \Omega(g)}{f-g}_{\mu}.
\end{align}
\end{lemma}

\begin{lemma}[Strongly Convex, \citep{rockafellar_convex_1970}]\label{lemma:strong_cvx}
Any function $\Omega$ that is first-order differentiable, $\Omega \in \C^1$ is said to be $\l(\geq 0)$-strongly convex 
if and only if for any $f, g \in dom(\Omega)$
\begin{align}
\Omega(f) \geq \Omega(g) + \IP{\nabla \Omega(g)}{f-g}_{\mu} + \frac{\l}{2}\nmm{f-g}_{\mu}^2.
\end{align}
\end{lemma}

\begin{definition}[Lipschitz Continuous]
A function $f: \X \to \Y$ is said to be Lipschitz continuous with Lipschitz constant $\nmL{f}$ if for any $x_2, x_2 \in \X$
\begin{align}
\nmm{f(x_1) - f(x_2)}_{\Y} \leq \nmL{f}\nmm{x_1-x_2}_{\X}.
\end{align}
\end{definition}
\textbf{Remark}: Lipschitz constant, $\nmL{f}$ is not a norm but only a semi-norm.
Because $\nmL{f} = 0$, it implies that $f$ can be any constant function.
\begin{definition}[Lipschitz Smooth]
A first-order differentiable function $f: \X \to \Y \in \C^1$ is said to be lipchtiz smooth if $\nabla f$
is Lipschitz continous.
\end{definition}

\begin{definition}[$(L, \l)$ convex function]\label{def:strosmot_cvx}
A first-order differentiable function $f: \X \to \Y \in \C^1$ is said to be 
$(L, \l)$ convex if and only if $f$ is $L$-Lipschitz smooth and $\l$-strongly convex, here $L \geq \l \geq 0$.
\end{definition}

\begin{proposition}[Properties of Lipschitz] The below are few properties of Lipschitz functions,
\begin{enumerate}
\item If function $f: \X \to \Y \in \C^{1}$ then $\sup_{x \in \X}\frac{\nmm{\IP{\nabla f(x)}{x}}_{\Y}}{\nmm{x}_{\X}} = \nmL{f}$.
\item If convex function $f: \X \to \Y \in \C^{1}$ is $L$-Lipschitz smooth then,
\begin{align}
f(x_0) + \IP{\nabla f(x_0)}{x-x_0}_{\Y} \leq f(x) \leq f(x_0) + \IP{\nabla f(x_0)}{x-x_0}_{\Y} + \frac{L}{2}\nmm{x-x_0}_{\X}^2.
\end{align}
\item If convex function $f: \X \to \Y \in \C^{1}$ is $(L, \l)$ convex then,
\be
\begin{split}
f(x_0) + \IP{\nabla f(x_0)}{x-x_0}_{\Y} &+ \frac{\l}{2}\nmm{x-x_0}_{\X}^2 \leq f(x) \\
&\quad  \leq f(x_0) + \IP{\nabla f(x_0)}{x-x_0}_{\Y} + \frac{L}{2}\nmm{x-x_0}_{\X}^2.
\end{split}
\ee
\end{enumerate}
\end{proposition}

\subsection{Concentration of Measure}

\begin{definition}[Greater than or approximately equal to] The inequality $f \gtrsim g$ means that
$\exists C > 0$ such that $f \geq Cg$.
\end{definition}

\begin{definition}[Sub-Gaussianity]
A random variable, $X$ is said to be sub-Gaussian with proxy variance, $\sigma^2$ if the following is satisfied,
\begin{align}
\mathbb{E}_{X}\left[ e^{t}[X-\mathbb{E}[X]] \right] \leq \exp\left({-\frac{t^2\sigma^2}{2}}\right); \forall t \geq 0.
\end{align}
We denote, $X \sim SG(\sigma^2)$.
\end{definition}

\begin{definition}[Sub-exponential]
A random variable $X$ is said to be subexponential with the proxy parameter $\l$ if the following is satisfied
\begin{align}
\mathbb{E}_{X}\left[e^{t[X-\mathbb{E}[X]]}\right] \leq \exp\left({-\frac{t\l}{2}}\right); \forall t \geq 0.
\end{align}
We denote $X \sim SE(\l)$.
\end{definition}

\begin{proposition}[Properties of Sub-Gaussianity and Sub-exponential]\label{prop:sg_se}
Let $X, Y$ be two random variables that need not be independent.
\begin{enumerate}
\item $X \in \R^{n} \sim SG\left(\frac{\sigma_X^2}{n}I_{n \times n}\right)$ if and only if $\nmm{X}^2 \sim SE(\sigma_X^2)$.
\item If $X \in \R^{n} \sim SG\left(\frac{\sigma_X^2}{n}I_{n \times n}\right)$, then for any Lipschitz function $\phi: \X \to \mathbb{R}$,
$\phi(X) \sim SG(\nmL{\phi}^2\sigma_X^2/n)$.
\item If $X \in \R^{n} \sim SG\left(\frac{\sigma_X^2}{n}I_{n \times n}\right)$, and $Y \in \R^{n} \sim SG\left(\frac{\sigma_Y^2}{n}I_{n \times n}\right)$, then $\IP{X}{Y} \sim SE(\sigma_X\sigma_Y)$. 
\end{enumerate}
\end{proposition}

\begin{lemma}[Uniform concentration of function]\label{lemma:conc_unf}
Consider an $n_X$-dimensional vector $X$, and
a parameterized function, $g_{\theta}: \X \to \R$, where $\theta \in \F_{\theta}$.
Let $\C$ be some convex set obeying $P(\cap_{i=1}^{N}X_i \in \C) \geq 1 - \delta_{\C}$. 
Assume that for any fixed $\theta_1, \theta_1 \in \F_{\theta}$ and any $Z \in \C$
we have
\be
\left|g_{\theta_1}(Z) - g_{\theta_2}(Z)\right| \leq Kd(\theta_1, \theta_2).
\ee
In addition, suppose that for any fixed $\theta \in \F_{\theta}$, we have
\be
\left|\mathbb{E}\left[g_{\theta}(\P_{\C}(X)) - g_{\theta}(X)\right]\right| \leq B,
\ee
where $\P_{\C}(\cdot)$ denotes the Euclidean projection onto the set $\C$. Finally,
suppose that for any fixed $\theta$ and $\epsilon \in [-t, t]$ it holds that \begin{align}
P\left(\left|\int_{\omega}(g_{\theta} \circ \P_{\C})d\mu_N(\omega) -
\int_{\omega}(g_{\theta} \circ \P_{\C})d\mu(\omega)\right| \geq \epsilon\right)
\leq \delta(\varepsilon),
\end{align}

Then for any $\epsilon \in [-t, t]$,
\be
P\left(
\sup_{\theta \in \F_{\theta}}\left|\int_{\omega}g_{\theta}d\mu_N(\omega) - \int_{\omega}g_{\theta}d\mu(\omega)\right|
\geq \epsilon + B
\right) \leq \N(\F_{\theta}, d(., .), \epsilon/(2K))\delta(\epsilon/4)+\delta_{\C}.
\ee
\end{lemma}

\begin{proof}
The proof technique is similar to that of \citep[Lemma 6]{li-wei-arxiv23} but
includes more general parameter sets $\F_{\theta}$. Let us define
\be
h_{\theta}(X) := g_{\theta}(\P_{\C}(X)),
\ee
from the assumptions in the lemma, we have that,
\begin{align}
P\left(\left|\int_{\omega}h_{\theta}d\mu_N(\omega) -
\int_{\omega}h_{\theta}d\mu(\omega)\right| \geq \epsilon\right)
\leq \delta(\varepsilon),
\end{align}

Next, we must establish uniform concentration overall $\theta \in \E_{\theta}$.
Let us construct a $\nu$-net for $\F_{\theta}$. 
For any $\theta' \in \N_{\nu}(\F_{\theta}, d(., .))$, $\theta \in \F_{\theta}$ from
the triangular inequality, and
as 
\be
|h_{\theta} - h_{\theta'}| = |h_{\theta}(X) - h_{\theta'}(X)| = |g_{\theta}(\P_{\C}(X)) - g_{\theta'}(\P_{\C}(X))| \leq 
Kd(\theta, \theta').
\ee
Then for any, $X$ we have that that,
\begin{align}
h_{\theta'} - Kd(\theta, \theta')
\leq h_{\theta} \leq h_{\theta'} + Kd(\theta, \theta').
\end{align}

Integrating with respect to the measure $\mu_N$, we obtain
\begin{align}
\int_{\omega}h_{\theta'}d\mu_N(\omega) - Kd(\theta, \theta')
\leq \int_{\omega}h_{\theta}d\mu_N(\omega) \leq \int_{\omega}h_{\theta'}d\mu_N(\omega) + L_Xd(\theta, \theta').
\end{align}
Similarly for the measure $\mu$ we obtain
\begin{align}
\implies
\int_{\omega}h_{\theta'}d\mu(\omega) - Kd(\theta, \theta')
\leq \int_{\omega}h_{\theta}d\mu(\omega) \leq \int_{\omega}h_{\theta'}d\mu(\omega) + Kd(\theta, \theta').
\end{align}
Now, subtracting the above equations, we obtain
\bd
\int_{\omega}h_{\theta'}d\mu_N(\omega) -
\int_{\omega}h_{\theta'}d\mu(\omega) - 2Kd(\theta, \theta') 
\ed
\bd
\leq 
\int_{\omega}h_{\theta}d\mu_N(\omega) -
\int_{\omega}h_{\theta}d\mu(\omega)
\leq 
\ed
\be
\int_{\omega}h_{\theta'}d\mu_N(\omega) -
\int_{\omega}h_{\theta'}d\mu(\omega) + 2Kd(\theta, \theta').
\ee
Now, take the absolute value on both sides. Later on, applying triangular inequality, we obtain
\begin{align}
\left|\int_{\omega}h_{\theta}d\mu_N(\omega) -
\int_{\omega}h_{\theta}d\mu(\omega)\right|
\leq 
\left|\int_{\omega}h_{\theta'}d\mu_N(\omega) -
\int_{\omega}h_{\theta'}d\mu(\omega)\right| + 2Kd(\theta, \theta').
\end{align}

Now choose, $\theta^*$ as $arg\sup_{\theta \in \F_{\theta}}\left|\int_{\omega}h_{\theta'}d\mu_N(\omega) -
\int_{\omega}h_{\theta'}d\mu(\omega)\right|$, then we have that,

\begin{align}
\sup_{\theta \in \F_{\theta}}\left|\int_{\omega}h_{\theta'}d\mu_N(\omega) -
\int_{\omega}h_{\theta'}d\mu(\omega)\right|
\leq 
\left|\int_{\omega}h_{\theta'}d\mu_N(\omega) -
\int_{\omega}h_{\theta'}d\mu(\omega)\right| + 2Kd(\theta^*, \theta').
\end{align}

Now choose any $\theta'$ that lies at-most $\nu$ from $\theta^*$ on the metric, $d(., .)$, i.e, $d(\theta', \theta^*) \leq \nu$, we have

\begin{align}
\sup_{\theta \in \F_{\theta}}\left|\int_{\omega}h_{\theta}d\mu_N(\omega) -
\int_{\omega}h_{\theta}d\mu(\omega)\right|
\leq 
\left|\int_{\omega}h_{\theta'}d\mu_N(\omega) -
\int_{\omega}h_{\theta'}d\mu(\omega)\right| +2K\nu.
\end{align}

By definition, we can bound the 
right hand term by the supremum,

\begin{align}
\sup_{\theta \in \F_{\theta}}\left|\int_{\omega}h_{\theta}d\mu_N(\omega) -
\int_{\omega}h_{\theta}d\mu(\omega)\right|
\leq 
2K\nu + \sup_{\theta' \in \N_{\nu}(\F_{\theta}, d(., .))}\left|\int_{\omega}h_{\theta'}d\mu_N(\omega) -
\int_{\omega}h_{\theta'}d\mu(\omega)\right|.
\end{align}
We apply the probability measure on both side, obtaining,
\bd
P\left(\sup_{\theta \in \F_{\theta}}\left|\int_{\omega}h_{\theta}d\mu_N(\omega) -
\int_{\omega}h_{\theta}d\mu(\omega)\right| \geq \epsilon\right)
\ed
\be
\leq 
P\left(\sup_{\theta' \in \N_{\nu}(\F_{\theta}, d(., .))}\left|\int_{\omega}h_{\theta'}d\mu_N(\omega) -
\int_{\omega}h_{\theta'}d\mu(\omega)\right| \geq \epsilon - 2K\nu\right),
\ee
the inequality is satisfied by the monotonicity of the probability measure. Now
we apply the union-argument for the $\nu$-net cover then we have
\bd
P\left(\sup_{\theta \in \F_{\theta}}\left|\int_{\omega}h_{\theta}d\mu_N(\omega) -
\int_{\omega}h_{\theta}d\mu(\omega)\right| \geq \epsilon\right)
\leq 
\ed
\be
P\left(\bigcup_{\theta' \in  \N_{\nu}(\F_{\theta}, d(., .))}\left|\int_{\omega}h_{\theta}d\mu_N(\omega) -
\int_{\omega}h_{\theta}d\mu(\omega)\right| \geq \epsilon - 2K\nu\right),
\ee
Now we upper bound the right side union term with summation, and then we have
\bd
P\left(\sup_{\theta \in \F_{\theta}}\left|\int_{\omega}h_{\theta}d\mu_N(\omega) -
\int_{\omega}h_{\theta}d\mu(\omega)\right| \geq \epsilon\right)
\leq 
\ed
\be
\sum_{\theta' \in  \N_{\nu}(\F_{\theta}, d(., .))}P\left(\left|\int_{\omega}h_{\theta}d\mu_N(\omega) -
\int_{\omega}h_{\theta}d\mu(\omega)\right| \geq \epsilon - 2K\nu\right),
\ee
Now we replace the summation with the $\nu$-covering number, $\N(\F_{\theta}, d(., .), \nu)$ obtaining
\begin{align}
P\left(\sup_{\theta \in \F_{\theta}}\left|\int_{\omega}h_{\theta}d\mu_N(\omega) -
\int_{\omega}h_{\theta}d\mu(\omega)\right| \geq \epsilon\right)
\leq \N(\F_{\theta}, d(., .), \nu)\delta(\epsilon
-2K\nu).
\end{align}
Now set $\nu = \epsilon/(2K)$ then we have
\begin{align}\label{eq:sup_proj_g}
\implies 
P\left(\sup_{\theta \in \F_{\theta}}\left|\int_{\omega}h_{\theta}d\mu_N(\omega) -
\int_{\omega}h_{\theta}d\mu(\omega)\right| \geq \epsilon\right)
\leq \N(\F_{\theta}, d(., .), \epsilon/(2K))\delta(\epsilon/2).
\end{align}

Now, we have established the uniform concentration for $h_{\theta}$. Next, we move onto 
relating $h_{\theta}$ with the desired function $h_{\theta}$.

Recall that 
\be
\left|\mathbb{E}\left[h_{\theta}(X) - g_{\theta}(X)\right]\right| \leq B.
\ee
As $P(\cap_{i=1}^{N}X_i \in \C) \geq 1 - \delta_{\C}$, we can safely claim 
$\int_{\omega}g_{\theta}d\mu_N(\omega) = \int_{\omega}h_{\theta}d\mu_N(\omega)$
with probability at least $1-\delta_{\C}$. We have 

\bd
\left|\int_{\omega}g_{\theta}d\mu_N(\omega) - \int_{\omega}g_{\theta}d\mu(\omega)\right|
= \left|\int_{\omega}h_{\theta}d\mu_N(\omega) - \int_{\omega}g_{\theta}d\mu(\omega)\right|
\ed
\be
\leq \left|\int_{\omega}h_{\theta}d\mu_N(\omega) - \int_{\omega}h_{\theta}d\mu(\omega)\right|
+ \left|\int_{\omega}h_{\theta}d\mu(\omega) - \int_{\omega}h_{\theta}d\mu(\omega)\right|,
\ee
\be\label{eq:proj_prob}
\implies \left|\int_{\omega}g_{\theta}d\mu_N(\omega) - \int_{\omega}g_{\theta}d\mu(\omega)\right|
\leq \left|\int_{\omega}h_{\theta}d\mu_N(\omega) - \int_{\omega}h_{\theta}d\mu(\omega)\right| + B,
\ee
with probability at least $1-\delta_{\C}$. Now we check the Lipschitzness of function $h_{\theta}$
in $\theta$ we have
\be
|h_{\theta_1}(X) - h_{\theta_2}(X)|
= |g_{\theta_1}(\P_{\C}(X)) - g_{\theta_2}(\P_{\C}(X))| \leq Kd(\theta_1, \theta_2).
\ee

Similarly, in expectation measure, we have that
\bd
|\mathbb{E}[h_{\theta_1}(X)] - \mathbb{E}[h_{\theta_2}(X)]|
= |\mathbb{E}[g_{\theta_1}(\P_{\C}(X))] - \mathbb{E}[g_{\theta_2}(\P_{\C}(X))]| 
\leq \mathbb{E}\left[|g_{\theta_1}(\P_{\C}(X)) - g_{\theta_2}(\P_{\C}(X))|\right]
\ed
\be
\leq Kd(\theta_1, \theta_2).
\ee

Consequently for any $\theta' \in \{\theta': d(\theta, \theta') \leq \epsilon/(2K)\}$, we have that
\be
\left|\int_{\omega}g_{\theta}d\mu_N(\omega) - \int_{\omega}g_{\theta}d\mu(\omega)\right| \leq
 \left|\int_{\omega}h_{\theta'}d\mu_N(\omega) - \int_{\omega}h_{\theta'}d\mu(\omega)\right| + \nu + B.
\ee
Now we choose 
$\theta = \theta^* = \sup_{\theta \in \F_{\theta}}\left|\int_{\omega}g_{\theta}d\mu_N(\omega) - \int_{\omega}g_{\theta}d\mu(\omega)\right|$ then,
\bd
\sup_{\theta \in \F_{\theta}}\left|\int_{\omega}g_{\theta}d\mu_N(\omega) - \int_{\omega}g_{\theta}d\mu(\omega)\right| = 
\left|\int_{\omega}g_{\theta^*}d\mu_N(\omega) - \int_{\omega}g_{\theta^*}d\mu(\omega)\right|
\ed
\be
\leq
 \left|\int_{\omega}h_{\theta'}d\mu_N(\omega) - \int_{\omega}h_{\theta'}d\mu(\omega)\right| + \epsilon + B.
\ee
We can take a supremum over $\theta'$ in the upper bound of the right side term; we have
\be
\sup_{\theta \in \F_{\theta}}\left|\int_{\omega}g_{\theta}d\mu_N(\omega) - \int_{\omega}g_{\theta}d\mu(\omega)\right|
\leq B + \epsilon
+  \sup_{\theta' \in \F_{\theta'}}\left|\int_{\omega}h_{\theta'}d\mu_N(\omega) - \int_{\omega}h_{\theta'}d\mu(\omega)\right|.
\ee
Then we use the inequality \eqref{eq:sup_proj_g} and \eqref{eq:proj_prob} we have that,
\be
\sup_{\theta \in \F_{\theta}}\left|\int_{\omega}g_{\theta}d\mu_N(\omega) - \int_{\omega}g_{\theta}d\mu(\omega)\right|
\leq 2\epsilon + B,
\ee
with probability at least $1 - \left[\N(\F_{\theta}, d(., .), \epsilon/(2K))\delta(\epsilon/2) + \delta_{\C}\right]$. Now rescaling we obtain that
\be
P\left(
\sup_{\theta \in \F_{\theta}}\left|\int_{\omega}g_{\theta}d\mu_N(\omega) - \int_{\omega}g_{\theta}d\mu(\omega)\right|
\geq \epsilon + B
\right) \leq \N(\F_{\theta}, d(., .), \epsilon/(2K))\delta(\epsilon/4)+\delta_{\C}.
\ee
\end{proof}

% \bibliography{refs}

\end{document}